\documentclass[twoside]{article}

\usepackage[accepted]{aistats2020}

\usepackage{amsmath,amsfonts,bm}

\def\Eqref#1{Equation~\eqref{#1}}
\def\1{\bm{1}}

\def\vone{{\bm{1}}}

\def\vg{{\bm{g}}}

\def\vu{{\bm{u}}}
\def\vv{{\bm{v}}}
\def\vw{{\bm{w}}}
\def\vx{{\bm{x}}}

\def\mA{{\bm{A}}}
\def\mB{{\bm{B}}}

\def\mD{{\bm{D}}}

\def\mG{{\bm{G}}}

\def\mI{{\bm{I}}}

\def\mP{{\bm{P}}}

\def\mU{{\bm{U}}}

\def\mW{{\bm{W}}}

\DeclareMathAlphabet{\mathsfit}{\encodingdefault}{\sfdefault}{m}{sl}
\SetMathAlphabet{\mathsfit}{bold}{\encodingdefault}{\sfdefault}{bx}{n}

\def\gG{{\mathcal{G}}}

\def\sS{{\mathbb{S}}}

\def\sV{{\mathbb{V}}}
\def\sW{{\mathbb{W}}}
\def\sX{{\mathbb{X}}}

\def\emA{{A}}

\def\emG{{G}}

\newcommand{\R}{\mathbb{R}}

\newcommand{\Var}{\mathrm{Var}}

\DeclareMathOperator{\sign}{sign}
\DeclareMathOperator{\Tr}{Tr}

\usepackage{float}

\usepackage{hyperref}
\hypersetup{
	colorlinks = true,
	citecolor = blue,
}
\usepackage{url}
\usepackage[pdftex]{graphicx}
\usepackage[textsize=scriptsize, textwidth = 1.5cm]{todonotes}
\newcommand{\V}{\mathcal V}
\newcommand{\Ed}{\mathcal E}

\newcommand{\y}{\mathbf y}

\newcommand{\bigO}{\mathcal O}

\newcommand{\Rsp}{R_{\textrm{sp}}}

\usepackage{xifthen}
\newcommand{\Esp}[1]{\ifthenelse{\equal{#1}{}}{E_{\textrm{sp}}}{E_{\textrm{sp},#1}}}

\usepackage{mathtools}
\DeclarePairedDelimiter\parens{[}{]}
\DeclareMathOperator{\EXop}{\mathbb{E}}% expected value
\newcommand\EX[1]{\EXop\parens*{#1}}

\newcommand\EXs[2]{\EXop_{#1}\parens*{#2}}

\DeclarePairedDelimiter\norm{\lVert}{\rVert}  % ||norm||

\usepackage{amsmath}
\usepackage{bbm}

\usepackage[labelformat=simple]{subcaption}

\usepackage{multirow}

\DeclareMathOperator{\DistOp}{\mathrm{dist}}
\newcommand\dist[2]{\DistOp{\!\left(#1,#2\right)}}

\usepackage{mleftright}
\usepackage{amssymb}

\usepackage{amsthm}
\theoremstyle{plain}% default
\newtheorem{thm}{Theorem}[section]
\newtheorem{lem}[thm]{Lemma}
\newtheorem{prop}[thm]{Proposition}
\newtheorem{cor}[thm]{Corollary}

\theoremstyle{definition}

\theoremstyle{remark}

\usepackage[round]{natbib}

\DeclareMathOperator*{\minim}{minimize}

\def \ver{internal}
\def \verinternal{internal}

\def \vertechrep{technicalrep}

\newcommand{\onlytechrep}[1]{% %parts that appear only in the technical report
	\ifx\ver\verinternal%
		{}
	\else%
		\ifx\ver\vertechrep%
			{{#1}}%
		\else {}%
		\fi%
	\fi%
}

\newcommand{\onlysub}[1]{% %parts that appear only in the technical report
	\ifx\ver\vertechrep%
		{}
	\else%
		{#1}%
	\fi%
}

\newcommand{\highlight}[1]{%	%highlight some parts in the internal version
	\ifx\ver\verinternal%
		{\textcolor{blue}{#1}}%
	\else {#1}%
	\fi%
}

\newcommand{\ic}[1]{%	%internal comments
	\ifx\ver\verinternal%
		{\textcolor{red}{#1}}%
	\else {}%
	\fi%
}

\newcommand{\citeapp}[1]{%
	\ifx\ver\vertechrep%
		{(Appendix~\ref{#1})}%
	\else%
		\ifx\ver\verinternal%
				{(Appendix~\ref{#1})}%
		\else{{\cite[App. \ref{#1}]{sm}}}%
		\fi
	\fi%
}

\newcommand{\inciteapp}[1]{%
	\ifx\ver\vertechrep%
		{Appendix~\ref{#1}}%
	\else%
		\ifx\ver\verinternal%
				{Appendix~\ref{#1}}%
		\else{{\cite[App. \ref{#1}]{sm}}}%
		\fi
	\fi%
}

\begin{document}

% If your paper is accepted and the title of your paper is very long,
% the style will print as headings an error message. Use the following
% command to supply a shorter title of your paper so that it can be
% used as headings.
%
%\runningtitle{I use this title instead because the last one was very long}

% If your paper is accepted and the number of authors is large, the
% style will print as headings an error message. Use the following
% command to supply a shorter version of the authors names so that
% they can be used as headings (for example, use only the surnames)
%
%\runningauthor{Surname 1, Surname 2, Surname 3, ...., Surname n}

\twocolumn[
\aistatstitle{Decentralized gradient methods: does topology matter?}

%\title{Decentralized gradient methods: does topology matter?}
%\maketitle

\aistatsauthor{ Giovanni Neglia  \And Chuan Xu \And  Don Towsley \And Gianmarco Calbi }

\aistatsaddress{ Inria, Univ.~C\^ote d'Azur \And Inria, Univ.~C\^ote d'Azur \And  UMass Amherst  \And Inria, Univ.~C\^ote d'Azur \AND 
France\And France\And USA \And France \AND
giovanni.neglia@inria.fr \And chuan.xu@inria.fr \And towsley@cs.umass.edu \And gianmarco.calbi@inria.fr
}

]

\begin{abstract}
Consensus-based distributed optimization methods have recently been advocated as alternatives to parameter server and ring all-reduce paradigms for large scale %data-parallel 
training of machine learning models. In this case, each worker maintains a local estimate of the optimal parameter vector and iteratively updates it by averaging the estimates obtained from its neighbors, and applying a correction on the basis of its local dataset. While theoretical results suggest that worker communication topology should have strong impact on the number of epochs needed to converge, previous experiments have shown the opposite conclusion. This paper sheds lights on this apparent contradiction and show how sparse topologies can lead to faster convergence even in the absence of communication delays. 
\end{abstract}

\section{INTRODUCTION}
%!TEX root = main.tex

In 2014, Google's Sybil machine learning (ML) platform was processing hundreds of terabytes through thousands of cores to train models with hundreds of billions of parameters~\citep{canini14}. At this scale, no single machine can solve these problems in a timely manner, and, as time goes on, the need for efficient distributed  solutions  becomes even more urgent. 
%For example, conservative forecasts  estimate the size of human genomes databases to double every 12 months until 2025~\cite{stephens15}. This growth rate well exceed what the outdated Moore's law would predict for transistor density. More and more computing nodes will then be required in the future. 
For example, experiments in~\citep{young17} rely on more than $10^4$ computing nodes to iteratively improve the (hyper)parameters of a deep neural network.

The example in~\citep{young17} is typical of a large class of iterative  ML distributed algorithms. Such algorithms begin with a guess of an optimal vector of parameters and proceed through multiple iterations over the input data to improve the solution. The process evolves in a data-parallel manner: input data is divided among worker threads.
%, each of which iterates over its data subset and adjust the solution based on its local view of the latest parameter values.  Solution adjustments are then exchanged among workers. 
%This operation appears to fit general-purpose computing frameworks like Apache Spark, which was designed to extend MapReduce and other data parallel abstractions to the iterative setting.
Currently, two communication paradigms are commonly used to coordinate the different workers~\citep{distr_tf18}: %asynchronous 
parameter server and 
%synchronous 
ring all-reduce. Both paradigms are natively supported by TensorFlow~\citep{abadi16}.

In the first case, a stateful parameter server 
(PS)~\citep{smola10}  %\todo{I think this reference is not so related, suggest to change to \cite{li14nips}}
maintains the current version of the model parameters. Workers  use locally available versions of the model to compute ``delta'' updates of the parameters (e.g.,~through a gradient descent step). Updates are then aggregated by the parameter server and combined with its current state to produce a new estimate of the optimal parameter vector.
%The parameter server could wait for all workers before updating the vector (synchronous operation), but 
%\emph{stragglers}, i.e.~slow tasks, would significantly reduce computation speed in a multi-machine setting~\cite{ananthanarayanan13,karakus17,li18}. Transient slowdowns are common in computing systems (especially in shared ones) and have many causes, such as resource contention, background OS activities,  garbage collection, and (for ML tasks) stopping criteria calculations. For this reason,  PS often operates asynchronously, updating the parameter vector as soon as  it receives the result of a single worker.
%an input  from $k$ of the $M$ workers (where $k$ can be as small as $1$). 
%While this approach increases system throughput (parameter updates per time unit), some workers will operate on stale versions of the parameter slowing and, in some cases, even preventing convergence to the optimal model~\cite{gonzalez15,kadav16}. %,lian18icml}. 
%Moreover, the PS can easily become a networking or computation bottleneck.

As an alternative, 
%When workers are homogeneous computing units (e.g.,~GPUs or TPUs) located in the same machine, synchronous operation becomes a viable solution. Moreover, 
it is possible to remove the PS, by letting each worker aggregate the inputs of all other workers through the ring all-reduce algorithm~\citep{ringallreduce17}. With $M$ workers, each aggregation phase requires $2(M-1)$ communication steps with $\bigO(1)$ data transmitted per worker. %Ring all-reduce has been deployed in many efficient low level implementations, e.g., in NVIDIA's library NCCL. 
There are many efficient low level implementations of ring all-reduce, e.g., in NVIDIA's library NCCL. 

We observe that  both the PS and the ring all-reduce paradigms 1)~maintain a unique candidate parameter vector at any given time and 2) rely \emph{logically} on an all-to-all communication scheme.\footnote{Each node needs to receive the aggregate of all other nodes' updates to move to the next iteration. Aggregation is performed by the PS or along the ring through multiple rounds.} 
Recently~\cite{lian17nips,lian18icml} have promoted an alternative approach in the ML research community, where each worker 1)~keeps updating a local version of the parameters and 2)~broadcasts its updates only to a subset of nodes (its neighbors). This family of algorithms became originally popular in the control community, starting from the seminal work of \cite{tsitsiklis86}
%Tsitsiklis \emph{et al} 
on distributed gradient methods. They are often referred to as \emph{consensus-based distributed optimization methods}.
Experimental results in~\citep{lian17nips,lian18icml,luo19} show that 
\begin{enumerate}
	\item in terms of number of epochs, the convergence speed is almost the same when the communication topology is a ring or a clique, 
	%\todo{and the data is uniformly random split?}
	\emph{contradicting} theoretical findings that predict convergence to be faster on a clique;
	\item in terms of wall-clock time, convergence is faster for sparser topologies, an effect attributed in~\citep{lian17nips} to smaller communication load. %worker communication costs. 
\end{enumerate}
In particular, \cite{luo19} summarize their findings as follows ``\emph{in theory, the bigger the spectral gap,} [i.e., the more connected the topology] \emph{the fewer iterations it takes to converge. However, our experiments do not show a significant difference in the convergence rate w.r.t. iterations, even when spectral gaps are very dissimilar}.''

In this paper we contribute to a better understanding of the potential advantages of consensus-based gradient methods. In particular,
\begin{enumerate}
%\item we explore the whole range of connectivity, not limiting ourselves to the extreme cases of rings and cliques,
\item we present a refined convergence analysis that helps to explain the apparent contradiction among theoretical results and empirical observations,
\item we show that sparse topologies can speed-up  wall-clock time convergence even when communication costs are negligible, because they intrinsically mitigate the straggler problem.
%\item our experiments indicate that, under a realistic distribution of computation times, sparse topologies like rings and 3-degree expander graphs may be the best practical choices.
\end{enumerate}

The paper is organized as follows. Section~\ref{s:background} provides  required background. Our theoretical analysis of the effect of  communication topology is in Sect.~\ref{s:quality}.   Experiment results in Sect.~\ref{s:experiments} confirm our findings. %Sect.~\ref{s:related} discusses our conclusions and relates them to existing literature. 
Section~\ref{s:concl} concludes the paper.

\section{NOTATION AND BACKGROUND}
\label{s:background}
%!TEX root = icml.tex
The goal of supervised learning is to learn a function that maps an input to an output  using $S$ examples from a training dataset $\sS = \{(\vx^{(l)},y^{(l)}), l =1, \dots S\}$. Each example $(\vx^{(l)},y^{(l)})$ is a pair consisting of an input object $\vx^{(l)}$ and an associated target value $y^{(l)}$.
In order to find the best statistical model, ML techniques often find the set of parameters $\vw \in \R^n$ that solves the following optimization problem:
\begin{align}
\label{e:problem}
\minim_\vw  \;\; \sum_{l=1}^S f(\vw,\vx^{(l)},y^{(l)})
%& \textrm{subject to} & &  \vw \in \sW \subset \R^n \nonumber
\end{align}
where function $f(\vw,\vx^{(l)},y^{(l)})$ %: \to \R$ 
represents the error the model commits on the $l$-th element of the dataset $\sS$ when parameter vector $\vw$ is used. The objective function may also include a regularization term that enforces some ``simplicity'' (e.g., sparseness) on $\vw$; such a term is easily taken into account in our analysis. 
%Both $f$ and $C$ are often convex functions of $\vw$. 
%For example, this is the case of standard linear regression, ridge and  lasso regression, as well as Support Vector Machine methods  or logistic regression for classification. Sparse inverse covariance estimation  and  matrix completion problems also lead to convex instances of problem~\eqref{e:problem}~\cite{bubeck15}.

%Different iterative techniques have been developed to solve problem~(\ref{e:problem}) (see~\cite{bubeck15} for a nice introduction). 
%\todo{Due to growth of available data and statistical model's complexity?}
Due to increases in available data and statistical model complexity, 
%\ic{I mean statistical models are becoming more complex} 
distributed solutions are often required to determine the parameter vector in a reasonable time.
The dataset in this case is divided among $M$ workers ($\sS = \cup_{j=1}^M\sS_j$), possibly with some overlap. For simplicity, we consider that all local datasets $\sS_j$ have the same size. Problem~(\ref{e:problem}) can be restated in an equivalent form as minimization of the sum of functions local to each node:
\begin{align}
\label{e:problem2}
\minim_\vw \;\;  F(\vw)=\sum_{j=1}^M F_j(\vw),%\\
%& \minim_\vw & & F(\vw)=\frac{1}{M}\sum_{i=1}^M F_i(\vw),%\\
%& \textrm{subject to} & & \vw \in \sW\nonumber
\end{align}
where $F_j(\vw) = \frac{1}{|\sS_j|}\sum_{(\vx^{(l)},y^{(l)}) \in \sS_j} f(\vw, \vx^{(l)}, y^{(l)})$. 

%In this paper, we consider as archetype the Distributed Subgradient Method (DSM)  for non-differentiable functions proposed in~\cite{nedic09}, but most of our considerations hold for many schemes including distributed dual averaging in~\cite{duchi12} and their stochastic variants considered in~\cite{lian17nips,lian18icml}.  In DSM there is no shared parameter vector, but each worker (node in what follows) keeps updating a local version of the parameters and broadcasts its updates to a subset of nodes (its neighbors).\footnote{
%	Message compression techniques like those described in~\cite{li14osdi} have been proved to be very effective and can be applied also in this context.
%}
The distributed system can be represented by a directed \emph{dataflow graph}  $\gG = (\V,\Ed)$, where $\V = \{1, 2, \dots M\}$ is the set of nodes (the workers) and 
% $\gG = (\V,\Ed)$ be a directed \emph{dependency graph}.
 % with $\V$ as set of nodes and $\Ed$ as set of edges. 
 an edge $(i,j) \in \Ed$ indicates that, at each iteration, node~$j$ waits for updates from node $i$ for the previous iteration. We assume the graph is strongly connected.
 Let $N_j =  \{i | (i,j) \in \Ed\}$ denote the in-neighborhood of node~$j$, i.e.,~the set of predecessors of node~$j$ in $\gG$. Each node $j$ maintains a local estimate of the parameter vector $\vw_j(k)$ and broadcasts it to its successors. The local estimate is updated as follows:
\begin{equation}
\label{e:dsm}
%\vw_i(k+1) = \Pi_\sW\!\!\left(\sum_{j\in N_i \cup \{i\}} \emA_{i,j} \vw_j(k) - \eta(k)\vg_i(\vw_i(k))\right),
\vw_j(k+1) = \sum_{i\in N_j \cup \{j\}} \vw_i(k) \emA_{i,j}  - \eta(k) \vg_j(\vw_j(k)).
\end{equation}
%$\Pi_\sW$ denotes the projection operator on the feasible set $\sW$. Apart from the projection, 
The node computes a weighted average (consensus/gossip component) of the estimates of its neighbors and itself, and then corrects it taking into account a stochastic subgradient\footnote{
	Given a function $f(\vw)$, a subgradient of $f()$ in $\vw$ is a vector $\vg$, such that $f(\y)-f(\vw)\ge \vg^{\intercal}(\y - \vw)$. In general a function can have many subgradients in a point $\vw$. %, this set is called the differential of the function and is denoted as $\partial f(\vw)$. 
	When the function $f$ is differentiable in $\vw$, the only subgradient is the gradient. %, i.e. $\partial f(\vw) = \{\nabla f(\vw)\}$. 
	With some abuse of notation, we indicate \emph{a} subgradient in $\vw$ as $\vg(\vw)$, even if $\vg$ is not a function.
} 
$\vg_j (\vw_j(k))$ of  its local function, i.e.,
\[\vg_j(\vw_j(k)) = \frac{1}{B}\sum_{(\vx^{(l)},y^{(l)}) \in \xi_j(k)} \partial f(\vw_j(k), \vx^{(l)}, y^{(l)}),\]
where $\partial f(\vw, \vx^{(l)}, y^{(l)})$ denotes a subgradient of $f$ with respect to $\vw$, and $\xi_j(k)$ is a random minibatch of size $B$ drawn from $\sS_j$.
Parameter $\eta(k) > 0$ is the (potentially time-varying) learning rate. $\mA=(\emA_{i,j})$ is an $M \times M$ matrix of non-negative weights. We call $\mA$ the \emph{consensus matrix}.\footnote{
	We are describing a synchronous DSM. The consistency model could be weaker, allowing node $i$ to use older estimates from its neighbors~\citep{li14nips}. 
}

The operation of a synchronous PS or ring all-reduce is captured by~(\ref{e:dsm}) when the underlying graph $\mathcal G$ is a clique, $\mA ={\bf 1} {\bf 1}^\intercal/M$, where $\bf 1$ is the $M \times 1$  vector consisting of all ones, and $\vw_i(0)=\vw_j(0), \forall i,j\in \V$. %TensorFlow and GraphLab allow for different graph topologies and, in general, different  consensus matrices~$\mA$.

Under some standard technical conditions, \cite[Thm.~8]{nedic18} and~\cite[Thm.~2]{duchi12} conclude, respectively for Distributed Subgradient Method (DSM)
% in the one-dimensional case ($n=1$) 
and for the Dual Averaging Distributed method, that the number of iterations $K_\epsilon$ needed to approximate the minimum objective function value by the desired error $\epsilon$ is
\begin{equation}
\label{e:kepsilon}
K_\epsilon \in \bigO\!\left( \frac{1}{\epsilon^2 \gamma(\mA)}\right),
\end{equation}
where $0\le \gamma(\mA)\le 1$ is the spectral gap of the matrix $\mA$, i.e.,~the difference between the moduli of the two largest eigenvalues of $\mA$. The spectral gap quantifies 
how information flows in the network. In particular, the spectral gap is maximal for a clique with weights $A_{i,j}=1/M$. Motivated by these convergence results,
 existing theoretically-oriented literature has concluded that a more connected network topology leads to faster convergence~\citep{nedic18,duchi12}. But some recent experimental results report that consensus-based gradient methods achieve similar performance after the same number of iterations/epochs on topologies as different as rings and cliques. For example~\cite[Fig.~3]{lian17nips} shows almost overlapping training losses for different ResNet architectures trained on CIFAR-10 with up to one hundred workers. \cite[Fig.~20]{luo19}, 
 %\todo{one of the new reference is added here}
 \cite[Fig.~11 in supplementary material]{DBLP:conf/icml/KoloskovaSJ19}, and our experimental results in~Sect.~\ref{s:experiments} confirm these findings.

 \cite{lian17nips} provide a partial explanation for this insensitivity in their Corollary~2, showing that the convergence rate is topology-independent 1)~after a large number of iterations ($\bigO(M^5/\gamma(\mA)^2)$), 
 %\todo{The other reference \cite{DBLP:conf/icml/AssranLBR19} could be added here, but it has the same property as Lian's paper do that offers a loose bound on the insensitivity. }
 2)~for a vanishing learning rate, and 3)~when the functions $F_j$ are differentiable with Lipschitzian gradients. Under the additional hypothesis of strong convexity,  \cite{olshevsky19nonasymptotic} prove that topology insensitivity should manifest after $\bigO(M/\gamma(\mA)^2)$ iterations.\footnote{We provide numerical estimates for the number of iterations predicted by \citep{lian17nips} and \citep{olshevsky19nonasymptotic} in \inciteapp{a:insensitivity}.} \cite{DBLP:conf/icml/AssranLBR19} provide similar results in terms of the graph diameter and maximum degree.
 %A tighter bound in terms of the number of workers ($\bigO(M/\gamma(A)^2)$) is in  under the additional hypothesis of strong convexity.
 These results do not explain why insensitivity is often observed in practice (as shown in~\citep{lian17nips,lian18icml,luo19,DBLP:conf/icml/KoloskovaSJ19}) 1)~since the beginning of the training phase,
 %since early stages of the training process, 
 2)~with constant learning rates, and 3)~for non-differentiable machine learning models (e.g., neural networks).  In the following section, we present a refined convergence analysis that explains when and why the effect of topology on the number of iterations needed to converge is weaker than what previously predicted.

\section{ANALYSIS} 
\label{s:quality}
A less connected topology requires more iterations to achieve a given precision as indicated by (\ref{e:kepsilon}). Our detailed analysis below shows that, when consensus-based optimization methods are used for ML training, the increase in the number of iterations is much less pronounced than previous studies predict. This is due to two different effects.
%that are illustrated also through the toy example in Sect.~\ref{s:toy_example}.
First, consensus is affected only by variability in initial estimates and subgradients across nodes, and not by their absolute values. Second,  certain configurations of initial estimates and subgradients are more difficult to achieve a consensus over, and would make the training highly dependent on the topology, but they are unlikely to be obtained by randomly partitioning the dataset.% among the nodes.

%\label{s:bounds}

Let $n$ be the number of parameters of the model, and $\mW(k)$ and $\mG(k)$ be $n\times M$ matrices, whose columns are, respectively,  node estimates $\vw_1(k), \dots, \vw_M(k)$ and subgradients $\vg_1(\vw_1(k)), \dots, \vg_M(\vw_M(k))$ at the completion of iteration $k$. \Eqref{e:dsm} can be rewritten in the  form $\mW(k+1)= \mW(k) \mA - \eta(k) \mG(k)$, from which we obtain iteratively 
\begin{equation}
\label{e:from_start}
\mW(k+1) = \mW(0) \mA^{k+1}- \sum_{h=0}^k \eta(h) \mG(h) \mA^{k-h}.
\end{equation}

%\begin{align}
%& ||  \mW(k+1)  - \bar\mW(k+1)||_{\max}  \triangleq \max_{i,j} |\emW_{i,j} - 
%\bar\evw_j|\\
%	&	\le \left\Vert\mW(0) \left( \frac{\vone \vone^\intercal}{M} - \mA^{k+1} \right)\right\Vert_{\max} + \gamma \left\Vert \sum_{h=0}^k \mG(h) \left( \frac{\vone \vone^\intercal}{M} - \mA^{k-h}  \right)\right\Vert_{\max}\\
%	& 	= \left\Vert (\mW(0) - \bar\mW(0)) \left( \frac{\vone \vone^\intercal}{M} - \mA^{k+1} \right)\right\Vert_{\max} + \gamma \left\Vert \sum_{h=0}^k (\mG(h) - \bar\mG(h) ) \left( \frac{\vone \vone^\intercal}{M} - \mA^{k-h}  \right)\right\Vert_{\max}\\
%	&	 =  \max_{i} \left \lVert (\mW_{i,:}(0) - \bar\mW_{i,:}(0)) \left( \frac{\vone \vone^\intercal}{M} - \mA^{k+1} \right) \right\rVert_\infty \nonumber\\
%	& \phantom{=} + \gamma \max_{i} \left \lVert \sum_{h=0}^k (\mG_{i,:}(0) - \bar\mG_{i,:}(0)) \left( \frac{\vone \vone^\intercal}{M} - \mA^{k+1} \right) \right\rVert_\infty\\
%	& \le \max_i \left \lVert \mW_{i,:}(0) - \bar\mW_{i,:}(0) \right \rVert_2 \sigma_2(A)^{k+1} +  \gamma  \sum_{h=0}^k  \max_{i} \left  \lVert \mG_{i,:}(0) - \bar\mG_{i,:}(0) \right \rVert_{2} \sigma_2(A)^{k-h}
%\end{align}

We make the following assumptions:\footnote{
Experiments in Sect.~\ref{s:experiments} show that our conclusions hold also when these assumptions are not satisfied. 
}
%\begin{table}[!h]
%    \centering
%    \begin{tabular}{ll}
%      {\bf A1} all functions $F_i$ are convex   & {\bf A2} the set of (global) minimizers $\sW^*$ is non-empty,\\
%      & \\
%      {\bf A3} graph $\gG$ is strongly connected, & {\bf A4} matrix $\mA$ is normal and doubly stochastic\\
%      \multicolumn{2}{c}{{\bf A5} the Frobenius norm of the subgradient matrices $\mG(k)$ is bounded in expectation over the vector $\bm{\xi}$ of random minibatches drawn at the nodes, i.e., there exists~$E$, such that $\EXs{\bm{\xi}}{\norm*{\mG(k)}_F^2} \le E$.}
%    \end{tabular}
%\end{table}
\begin{description}
	\item[A1] all functions $F_i$ are convex, 
	\item[A2] the set of (global) minimizers $\sW^*$ is non-empty,
	\item[A3] graph $\gG$ is strongly connected, 
	\item[A4] matrix $\mA$ is normal (i.e.,~$\mA^\intercal \mA = \mA \mA^\intercal$) and doubly stochastic,
	\item[A5] the squared Frobenius norm of subgradient matrix $\mG(k)$ is bounded in expectation over the vector $\bm{\xi}$ of minibatches randomly drawn at nodes, i.e., there exists~$E$, such that $\EXs{\bm{\xi}}{\norm*{\mG(k)}_F^2} \le E$.
	\end{description}
Assumptions A1-A4 are standard ones in the related literature, see for example~\citep{nedic09,duchi12,nedic18}. Assumption A5 imposes a bound on the (expected) \emph{energy} of the subgradients, because  $\norm*{\mG(k)}_F^2=\sum_{j} \norm*{\vg_j(\vw_j(k))}_2^2$. In the literature it is often replaced by the stronger requirement that the norm-2 of the subgradients $\vg_j(\vw_j(k))$ 
%(the columns of $\mG(k)$) 
is bounded.
%For example, this is the case of standard linear regression, ridge and  lasso regression, as well as Support Vector Machine methods  or logistic regression for classification. Sparse inverse covariance estimation  and  matrix completion problems also lead to convex instances of problem~\eqref{e:problem}~\cite{bubeck15}.
%Assumption A5 can be formally written as $\EXs{\bm{\xi}}{\norm*{\mG(k)}_F^2} \le E$, where the expectation is over the vector $\bm{\xi}$ of random minibatches drawn at the nodes. 
%The constant $E$ is a bound of the (expected) ``energy'' of the subgradients, because  $\norm*{\mG(k)}_F^2=\sum_{j} \norm*{\vg_j(\vw_j(k))}_2^2$. 
Let $\Delta \mG(k)$ denote the matrix $\mG(k) - \mG(k) \vone \vone^\intercal/{M}$, whose column $j$ is the difference between  subgradient $\vg_j(\vw_j(k))$ and the average of subgradients $ \sum_{j=1}^M \vg_j(\vw_j(k))/M$. %It holds $\norm*{\Delta \mG(k)}_F^2 = \norm*{\mG(k)}_F^2 - \norm*{\mG(k) \vone \vone^\intercal/{M}}_F^2 \le  \norm*{\mG(k)}_F^2$. 
%Then, A5 implies that $\norm*{\Delta \mG(k)}_F^2$ is bounded in expectation too, i.e.,~it exists $\Esp{}\le E$ such that $\EXs{\bm{\xi}}{\norm*{\Delta \mG(k)}_F^2} \le \Esp{}$. Then also the energy of the variability of the subgradients is bounded.
$\norm*{\Delta\mG(k)}_F^2$ captures the variability in the subgradients.
Assumption A5 also implies that there exist two constants $\Esp{}\le E$ and $H \le \sqrt{E}$ such that
\[\EXs{\bm{\xi}}{\norm*{\Delta \mG(k)}_F^2} \le \Esp{}, \;\;\; \norm*{\EXs{\bm{\xi}}{\mG(k)}}_F \le H.\]
%where $\Delta \mG(k)$ denotes the matrix $\mG(k) - \mG(k) \vone \vone^\intercal/{M}$, whose column $j$ is the difference between  subgradient $\vg_j(\vw_j(k))$ and the average subgradient $ \sum_{j=1}^M \vg_j(\vw_j(k))/M$.
Similarly, let $R$ denote the energy of the initial parameter vectors (or an upper bound for it),  i.e.,~$R\triangleq\norm*{\mW(0)}_F^2$. We also denote by $\Rsp$   the energy for the difference matrix $\Delta \mW(0)\triangleq \mW(0) - \mW(0) \vone \vone^ \intercal/M$, i.e.,~ $\Rsp \triangleq \norm*{\Delta \mW(0)}_F^2$. $\Rsp$ captures the variability in initial estimates. It holds $\Rsp \le R$. 

Because of Assumption~A4, the consensus matrix has a spectral decomposition with orthogonal projectors $\mA = \sum_{q=1}^Q \lambda_q \mP_q,$
where $\lambda_1, \dots, \lambda_Q$ are the $Q\le M$ distinct eigenvalues of $\mA$,  $\mP_q$ is the orthogonal projector onto the nullspace of $\mA - \lambda_q \mI$ along the range of $\mA - \lambda_q \mI$. We assume that the eigenvalues are ordered so that $|\lambda_1|\ge |\lambda_2| \ge \dots \ge |\lambda_Q|$. Assumptions~A3 and A4 imply that $\lambda_1=1$, and $|\lambda_2|<1$
%, and $\mP_1 = \vone \vone^\intercal /M$ 
\citeapp{a:linear_algebra}. 
Finally, we define
\begin{equation}
\label{e:alpha1}
\alpha \triangleq \begin{cases}
     \sqrt{\sum_{q=2}^Q e_q \left| \frac{\lambda_q}{\lambda_2} \right|^{2}}, & \textrm{if }\lambda_2 \neq 0,\\
    1, & \textrm{otherwise.}
    \end{cases}
\end{equation}
where $e_q$ is an upper-bound for the normalized fraction of energy $\EXs{\bm \xi}{\norm*{\Delta \mG(k)}_F^2}$ in the subspace defined by projector $\mP_q$ \citeapp{a:convergence}. The quantity $\alpha$ can be interpreted as an effective bound for the fraction of the energy $\Esp{}$ that falls in the subspace relative to the second largest eigenvalue $\lambda_2$.

We are now ready to introduce our main convergence result. We state it for the average model over nodes and time, i.e.,~for $\hat{\overline \vw}(K-1) \triangleq \frac{1}{K} \sum_{k=0}^{K-1} \frac{1}{M} \sum_{i=1}^M  \vw_i(k)$. 
We have also derived a similar bound for the local time-average model at each node, i.e.,~for $\hat \vw_i(K-1)\triangleq \frac{1}{K} \sum_{k=0}^{K-1}   \vw_i(k)$ \citeapp{a:conv_local_estimate}.\footnote{For subgradient methods, convergence results are usually for the time-average model.}
\begin{prop}
\label{p:new_bound}
Under assumptions A1-A5 and that~a constant learning rate $\eta(k)=\eta$ is used, an upper bound for the objective value at the end of the \mbox{($K-1$)th} iteration is given by:
\begin{align}
\hspace{-0.9cm}& \EX{F\!\left(\hat  {\overline \vw}(K-1) \right)}- F^*  \le   \frac{M}{2 \eta {K}} {\dist{\hat {\overline \vw}(0)} {\sW^*}}^2 + \frac{\eta E}{2} \nonumber\\
    & \phantom{123} + 2 H \sqrt{\Rsp}\frac{ \sqrt{M} }{K} \frac{1 - |\lambda_2|^{K}}{1- |\lambda_2|} \nonumber \\
    & \phantom{123} + 2 \eta H \sqrt{\Esp{}} \left( (1-\alpha) \frac{K-1}{K} \right.\label{e:new_bound}\\
    & \phantom{123+ 2 \eta H \sqrt{\Esp{}}\left(\right.}\left. + \frac{ \alpha}{1- |\lambda_2|} \left( 1- \frac{1}{K} \frac{1 - |\lambda_2|^{K}}{1- |\lambda_2|} \right)\right).
 \nonumber
\end{align}
\end{prop}
%\begin{prop}
%\label{p:new_bound}
%Under assumptions A1-A5 and that~a constant learning rate $\eta(k)=\eta_0/\sqrt{K}$ is used, an upper bound for the objective value at the end of the \mbox{($K-1$)th} iteration is given, for each $i$, by:
%\begin{align}
%\hspace{-0.9cm}& \EX{F\!\left(\hat    \vw_i(K-1) \right)}- F^*  \le   \frac{M}{2 \eta_0 \sqrt{K}} {\dist{\overline \vw(0)} {\sW^*}}^2 + \frac{\eta_0 E}{2 M \sqrt{K}} + H \frac{3 M \sqrt{\Rsp}}{K} \frac{1 - |\lambda_2|^{K}}{1- |\lambda_2|} \nonumber\\
%    & \phantom{\le}+ 3\sqrt{M}H \sqrt{\Esp{}}\frac{\eta_0}{\sqrt{K}}\left(     
%       (1-\alpha(1)) \frac{K-1}{K}
%    \phantom{\le}+ \frac{ \alpha(1)}{1- |\lambda_2|} \left( 1- \frac{1}{K} \frac{1 - |\lambda_2|^{K}}{1- |\lambda_2|} \right)\right)
%\label{e:new_bound}	 
%\end{align}
%\begin{align}
%\hspace{-0.9cm}& F\!\left(\frac{1}{K} \sum_{k=0}^{K-1}   \vw_i(k) \right)- F^*  \le \frac{M}{2 \eta_0 \sqrt K} {\dist{\bar \vw(0)} {\sW^*}}^2 \nonumber\\ 
%\hspace{-0.9cm}		 			 & + \frac{\eta_0 L^2}{2 \sqrt K} + 2 L M n^{1/2} \Bigg(  \frac{\Rsp}{K} \frac{1 - |\lambda_2|^{K}}{1- |\lambda_2|}    \phantom{\Big)}\nonumber\\
%		\phantom{Bigg(}\hspace{-0.9cm} & \! +\frac{\eta_0 \Lsp}{\sqrt K} \left(\! (1\!-\!\alpha(1)) \mathbf 1_{K>1}+ \!\frac{   \alpha(1)}{1- |\lambda_2|} \left(\! 1\!- \! \frac{1 - |\lambda_2|^{K}}{K (1- |\lambda_2|)}\! \right)\!\right)\! \Bigg),
%\end{align}\label{e:new_bound}
%for each node $i$.
Here, $\dist{\vx} {\sW^*}$ denotes the Euclidean distance between vector $\vx$ and set of global minimizers $\sW^*$. 
%$\overline \vw(0)$ denotes the average of the initial parameter vectors.
The proof is in \inciteapp{a:proof_new_bound}.
 The first two terms on the right hand side of (\ref{e:new_bound}) also appear  when studying the convergence of centralized subgradient methods. The last two terms appear because of the distributed consensus component of the algorithm and depend on $|\lambda_2|<1$. We observe that $1-|\lambda_2|$ is the spectral gap of $\mA$. It measures how well connected the graph is. In particular, the larger the spectral gap (the smaller $\lambda_2$), the better the connectivity and the smaller the bound in \eqref{e:new_bound}.

From  Proposition~\ref{p:new_bound}, we can derive a looser bound analogous to the bound for DSM in~\citep{nedic09}. In fact, observing that $\Rsp \le R$, $\Esp{} \le E$, $H\le \sqrt{E}$, and $0\le \alpha\le1$, we can prove~{\citeapp{a:proof_classic_bound}}:
%Because $\Rsp < R$, $\Lsp <L$ and $0\le \alpha(1)\le1$, % and $\frac{1 - |\lambda_2|^{K}}{1- |\lambda_2|} \ge 1$ for $K>1$, 
%the following bound follows from Proposition~\ref{p:new_bound}~{\citeapp{a:proof_classic_bound}}. 
 \begin{cor}
\label{p:classic_bound}
Under assumptions A1-A5 and that constant learning rate $\eta(k)=\eta$ is used, an upper bound for the objective value at the end of the \mbox{($K-1$)th} iteration is given by:
\begin{align}
&\EX{F\!\left(\hat {\overline \vw}(K-1) \right)}  - F^*  \le   \frac{M}{2 \eta {K}} {\dist{\hat{\overline \vw}(0)} {\sW^*}}^2 + \frac{\eta E}{2 }\nonumber\\
    &\phantom{12345} + 2 \sqrt{E} \sqrt{R} \frac{\sqrt{M}}{K} \frac{1 - |\lambda_2|^{K}}{1- |\lambda_2|} \nonumber\\
    &\phantom{12345} +   2 \eta  E  \frac{1}{1- |\lambda_2|} \left( 1-  \frac{1}{K}\frac{1 - |\lambda_2|^{K}}{1- |\lambda_2|} \right).
\label{e:classic_bound}	 
\end{align}
%\begin{align}
%\EX{F\!\left(\hat    \vw_i(K-1) \right)}- F^*  &\le   \frac{M}{2 \eta_0 \sqrt{K}} {\dist{\overline \vw(0)} {\sW^*}}^2 + \frac{\eta_0 E}{2 M \sqrt{K}} + \sqrt{E} \frac{3 M \sqrt{R}}{K} \frac{1 - |\lambda_2|^{K}}{1- |\lambda_2|} \nonumber\\
%    & \phantom{\le}+ 3\sqrt{M} E\frac{\eta_0}{\sqrt{K}}\frac{  1 }{1- |\lambda_2|} \left( 1-  \frac{1}{K}\frac{1 - |\lambda_2|^{K}}{1- |\lambda_2|} \right).
%\label{e:classic_bound}	 
%\end{align}
In particular, if workers compute full-batch subgradients and the 2-norm of subgradients of functions $F_i$ is bounded by a constant $L$, we obtain: 
\begin{align}
&F\!\left(\hat {\overline \vw}(K-1) \right)- F^*  \le   \frac{M}{2 \eta {K}} {\dist{\hat{\overline \vw}(0)} {\sW^*}}^2 + \frac{\eta M L^ 2}{2 }\nonumber\\
    &\phantom{12345} + 2  L \sqrt{R} \frac{M}{K} \frac{1 - |\lambda_2|^{K}}{1- |\lambda_2|} \nonumber\\
    &\phantom{12345} +   2 \eta  L^2  \frac{M}{1- |\lambda_2|} \left( 1-  \frac{1}{K}\frac{1 - |\lambda_2|^{K}}{1- |\lambda_2|} \right).
\label{e:classic_bound2}	 
\end{align}
%\begin{align}
%\EX{F\!\left(\hat    \vw_i(K-1) \right)}- F^*  &\le   \frac{M}{2 \eta_0 \sqrt{K}} {\dist{\overline \vw(0)} {\sW^*}}^2 + \frac{\eta_0 E}{2 M \sqrt{K}} + \sqrt{E} \frac{3 M \sqrt{R}}{K} \frac{1 - |\lambda_2|^{K}}{1- |\lambda_2|} \nonumber\\
%    & \phantom{\le}+ 3\sqrt{M} E\frac{\eta_0}{\sqrt{K}}\frac{  1 }{1- |\lambda_2|} \left( 1-  \frac{1}{K}\frac{1 - |\lambda_2|^{K}}{1- |\lambda_2|} \right).
%\label{e:classic_bound}	 
%\end{align}
%\begin{align}
%\hspace{-0.9cm}& F\!\left(\frac{1}{K} \sum_{k=0}^{K-1}   \vw_i(k) \right)- F^*  \le \frac{M}{2 \eta_0 \sqrt K} {\dist{\bar \vw(0)} {\sW^*}}^2 \nonumber\\ 
%\hspace{-0.9cm}		 			 & + \frac{\eta_0 L^2}{2 \sqrt K} + 2 L M n^{1/2} \Bigg(  \frac{R}{K} \frac{1 - |\lambda_2|^{K}}{1- |\lambda_2|}  \phantom{\Big)}\nonumber\\
%		\phantom{Bigg(}\hspace{-0.9cm} &+ \frac{\eta_0 L}{\sqrt K} \frac{  1 }{1- |\lambda_2|} \left( 1-  \frac{1 - |\lambda_2|^{K}}{K(1- |\lambda_2|)} \right) \Bigg) .\label{e:classic_bound}
%\end{align}
%for each node $i$.
\end{cor}
%Inequality~(\ref{e:classic_bound}) shows that the error depends on the 2-norms of the initial parameter estimates ($R$) as well as on the subgradient ($L$). This kind of relation corresponds to classic bounds for DSM and can, for example, be derived as in~\cite{nedic09}, assuming the consensus matrix is constant and normal. 

When $K$ is large enough, 
%By looking at~\eqref{e:classic_bound2} and observing that $\eta_0$ is usually proportional to $1/L$, one could conclude that 
the fourth term  in~\eqref{e:classic_bound} and~\eqref{e:classic_bound2} is dominant, so that the error is essentially proportional to $1/(1-|\lambda_2|)$. Note  that the constant multiplying $1/(1-|\lambda_2|)$  in~\eqref{e:classic_bound} is larger than the corresponding one in~\eqref{e:new_bound} by a factor
\begin{equation}
\label{e:beta}
\beta \triangleq \frac{1}{\alpha} \times \frac{E}{\sqrt{\Esp{}} H}
\end{equation}
 The value $\beta$ roughly indicates how much looser bound~\eqref{e:classic_bound} is in comparison to bound~\eqref{e:new_bound}.

Existing theoretical works like~\citep{nedic09,duchi12,nedic18} derived bounds similar to~\eqref{e:classic_bound2} and concluded then that one should select the learning rate proportional to $\sqrt{1-|\lambda_2|}$ to reduce the effect of topology. In particular, one obtains (\ref{e:kepsilon}) when $\eta= \eta_0 \sqrt{(1-|\lambda_2|)/K}$.
%\todoi{Chuan: In the third term of (\ref{e:classic_bound}), we replace $\sqrt{E*R}$ by $H\times \sqrt{\Rsp}$ right? }
Our bound~(\ref{e:new_bound}) improves bound (\ref{e:classic_bound}) by replacing $R$ in the third terms of (\ref{e:classic_bound}) by the smaller value $\Rsp$, and $\sqrt{E}$ in the third and fourth terms  by the smaller values $H$ and $\sqrt{\Esp{}}$, and introducing the new coefficient $\alpha$.  We  qualitatively describe the  effect of these constants.

\paragraph{$\mathbf\Rsp$} Bound~\eqref{e:new_bound} shows that the norm of the initial estimates ($R$) does not really matter, but rather variability among workers does. In particular, for ML computation we can make $\vw_i(0)=\vw_j(0)$ for each $i$ and $j$, and then  $\Rsp =0$, so that the third term in the RHS of~(\ref{e:new_bound}) vanishes. 

\paragraph{$\mathbf{E_{\textrm{sp}}, H}$} %, $\mathbf{H}$}
%We observe that the constant multiplying $1/(1-|\lambda_2|)$ in the fourth term in~\eqref{e:new_bound} is smaller than the corresponding one in~\eqref{e:classic_bound} by a factor $\beta \triangleq \sqrt{E/\Esp{}}\times \sqrt{E}/H \times 1/\alpha$.
For $\Esp{}$,  considerations similar to those applying to $\Rsp$ hold. What matters is the variability of the subgradients. 
Assume that the dataset is replicated at each node and each node computes the subgradient over the full batch ($B=S$). In this case all subgradients would be equal, and $\norm*{\Delta \mG(k)}=0$,  $\Esp{}=0$, and the fourth term would also vanish. This corresponds to the fact that, when initial parameter vectors, as well as local functions, are the same, the parameter vectors are equal at any iteration $k$ and the system evolves exactly as it would under a centralized subgradient method. %POSSIBLE REMOVAL
In general, local subgradients can be expected to be close (and $E \gg \Esp{}$), if 1)~local datasets are representative of the entire dataset (the dataset has been randomly split and $|\sS_j|\gg M$), and 2)~large batch sizes are used. On the other hand, when batch sizes are very small, one expects stochastic subgradients to be very noisy, and as a consequence the energy of the matrix $\mG$ to be much larger than the energy of $\EXs{\bf \xi}{\mG}$,  so that $\sqrt{E} \gg H$. In both cases, $E/(\sqrt{\Esp{}} H)$ is large (in the first case because $\sqrt E \gg \sqrt{\Esp{}}$, and in the second because $\sqrt{E} \gg H$). We quantify these effects below.

\paragraph{$\bm{\alpha}$} From~\eqref{e:from_start} we see that the effect of the subgradients is modulated by $\mA^{k-h}$, that equals $\sum_{q=1}^Q \lambda_q^{k-h} \mP_q$. The energy of the subgradients is spread across the different  
subspaces defined by the eigenvectors of $\mA$. The classic bound~\eqref{e:classic_bound} implicitly assumes that all energy falls in the subspace corresponding to $\lambda_2$ (this occurs if the row of the matrices $\mG(k)$ are aligned with the second eigenvector). In reality, on average each subspace will only get $1/Q$-th of the total energy ($e_q \approx 1/Q$), and the energy in other subspaces will be dissipated faster than what happens for the subspace corresponding to~$\lambda_2$. $\alpha\le 1$ quantifies this effect.
%and we can approximate it as follows
%\begin{equation}
    %\hat \alpha(1) \triangleq  \sqrt{\sum_{q=2}^Q \frac{1}{Q} \left| \frac{\lambda_q}{\lambda_2} \right|^{2}}. \label{e:alpha1_ext}
%\end{equation}

A toy example in \inciteapp{a:toy_example} illustrates qualitatively these effects. Here we provide estimates for the expected values of $E$, $\Esp{}$, and $H$ over all possible ways to distribute the dataset $\sS$ randomly across the nodes. We reason as follows. For a given parameter vector $\vw$, consider the set of subgradients at all dataset points, i.e.,~$\cup_{(\vx^{(l)}, y^{(l)})\in \sS}\{\partial f(\vw, \vx^{(l)}, y^{(l)})\}$. The average subgradient over all  datapoints is $\partial F(\vw)$. 
%Let $\bm{\sigma^2}(\vw)$ denote a vector, whose $i$-th component is the variance of the $i$-th component of the subgradients over the whole dataset. In what follows, we will need to refer to the norm-2 of the two vectors $\partial F(\vw)$ and $\bm{\sigma^2}(\vw)$, for which we will omit the dependency on $\vw$ writing simply $\norm*{\partial F}_2^2$ and $\norm*{\bm{\sigma^2}}_2^2$.
Let $\sigma^2(\vw)$ denote the trace of the covariance matrix of all subgradients. 
%PUT BACK IN ARXIV
$\sigma^2(\vw)$ then equals the sum of the variances of all the components of the subgradients.
%\todo{a little bit confusing for me this sentence. Maybe is better we just use the definition below that: We denote by $\sS_C$ the expanded dataset that every data point in $\sS$ is replicated $C$ times.
We denote by $\sS_C$ the expanded dataset where each datapoint is replicated $C$ times with $1\le C\le M$.
%We now distribute the dataset across $M$ nodes replicating each point $C$ times (where $C\in \{1, 2, \dots M\}$). We denote by $\sS_C$ the expanded dataset. 
The dataset  $\sS_C$ is split across the $M$ nodes.
Each node selects a random minibatch from its local dataset and we denote by $\mG$ the corresponding subgradient matrix. 
\begin{prop}
\label{p:estimates}
Consider a uniform random permutation $\pi$ of~~$\sS_C$ with the constraint that  $C$ copies of the same point are placed at $C$ different nodes. The following holds
\begin{align}
 &\EXs{\pi}{\EXs{\bm\xi}{\norm*{\mG}_F^2}} = M \left( \norm*{\partial F}_2^2+ \frac{S-B}{B(S-1)}\sigma^2\right), \nonumber\\
 &\EXs{\pi}{\EXs{\bm\xi}{\norm*{\Delta\mG}_F^2}} = \sigma^2  \frac{MC(S-B) -CS +MB}{CB(S-1)}, \nonumber \\
&\EXs{\pi}{\norm*{\EXs{\bm\xi}{\mG}}_F} \label{e:estimates}\\
& \phantom{\EXs{\pi}{}}\in \left[ \sqrt{M} \norm*{\partial F}_2, \sqrt{M} \sqrt{\norm*{\partial F}_2^2 + \frac{M-C}{C(S-1)} \sigma^2} \;\right]. \nonumber 
 \end{align}
% \begin{align}
% &\EXs{\pi}{\EXs{\bm\xi}{\norm*{\mG}_F^2}} = M \left( \norm*{\partial F}_2^2+ \frac{S-B}{B(S-1)}\sigma^2\right), \label{e:estimates}\\
% &\EXs{\pi}{\EXs{\bm\xi}{\norm*{\Delta\mG}_F^2}} = M \sigma^2 \left( \frac{S-B}{B(S-1)}- \frac{CS-MB}{MCB(S-1)}\right), \nonumber \\
%&\EXs{\pi}{\norm*{\EXs{\bm\xi}{\mG}}_F}  \in \left[ \sqrt{M} \norm*{\partial F}_2, \sqrt{M} \sqrt{\norm*{\partial F}_2^2 + \frac{M-C}{C(S-1)} \sigma^2}\right]. \nonumber 
% \end{align}
\end{prop}
We can use~\eqref{e:estimates} to study how $E$, $\Esp{}$, and $H$ vary with dataset size, batch size, and number of replicas, using the following approximations:
\begin{align}
& \widehat E  = \EXs{\pi}{\EXs{\bm\xi}{\norm*{\mG}_F^2}},\nonumber  
 \widehat{E}_{\textrm{sp}}  =  \EXs{\pi}{\EXs{\bm\xi}{\norm*{\Delta\mG}_F^2}},  \\
& \widehat H  = \sqrt{M} \sqrt{\norm*{\partial F}_2^2 + \frac{M-C}{C(S-1)} \sigma^2}.\label{e:perm_approx}
\end{align}

\begin{figure*}[ht]
    \centering
    \begin{subfigure}[b]{0.47\textwidth}
        \centering
        \includegraphics[width=0.9\textwidth]{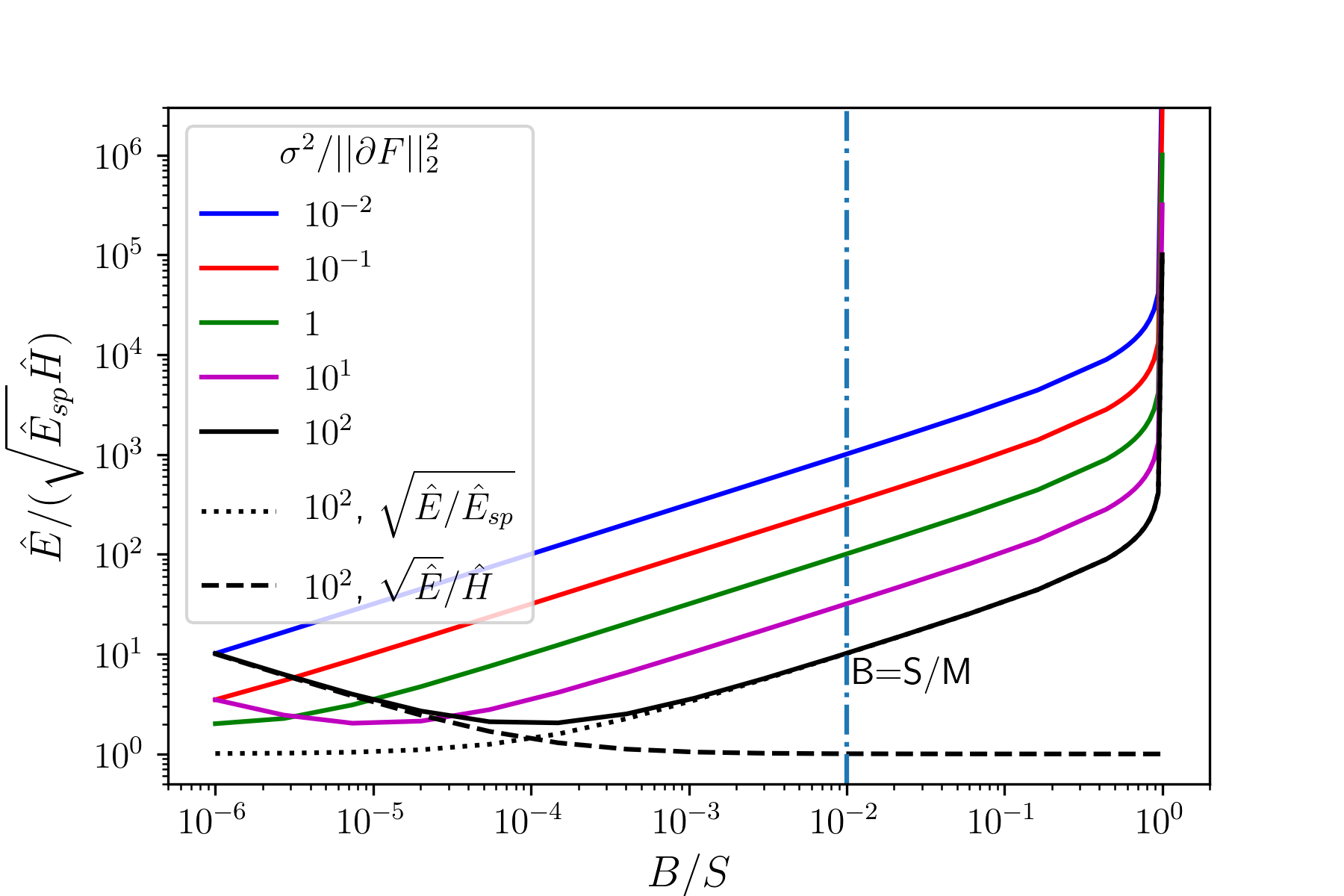}
        \caption[]{\small $C = M$}
        %\caption[]{\small Estimate of $E/(\sqrt{\Esp{}} H)$ versus the relative batch size $B/S$ for $M=100$, $S=10^6$, $C=M$, and different level of heterogeneity ($\sigma^2/\norm*{\partial F}_2^2$) of the dataset. Curves for $C=1$  are very similar, but for the fact that the batch size can scale only up to $S/M$.%The plot curves do not depend much on the replication factor $C$, but for the fact that the batch size can scale up to $C S/M$.
        %}
        \label{f:ratio}
    \end{subfigure}
    \hfill
    \begin{subfigure}[b]{0.47\textwidth}
        \centering
        \includegraphics[width=0.9\textwidth]{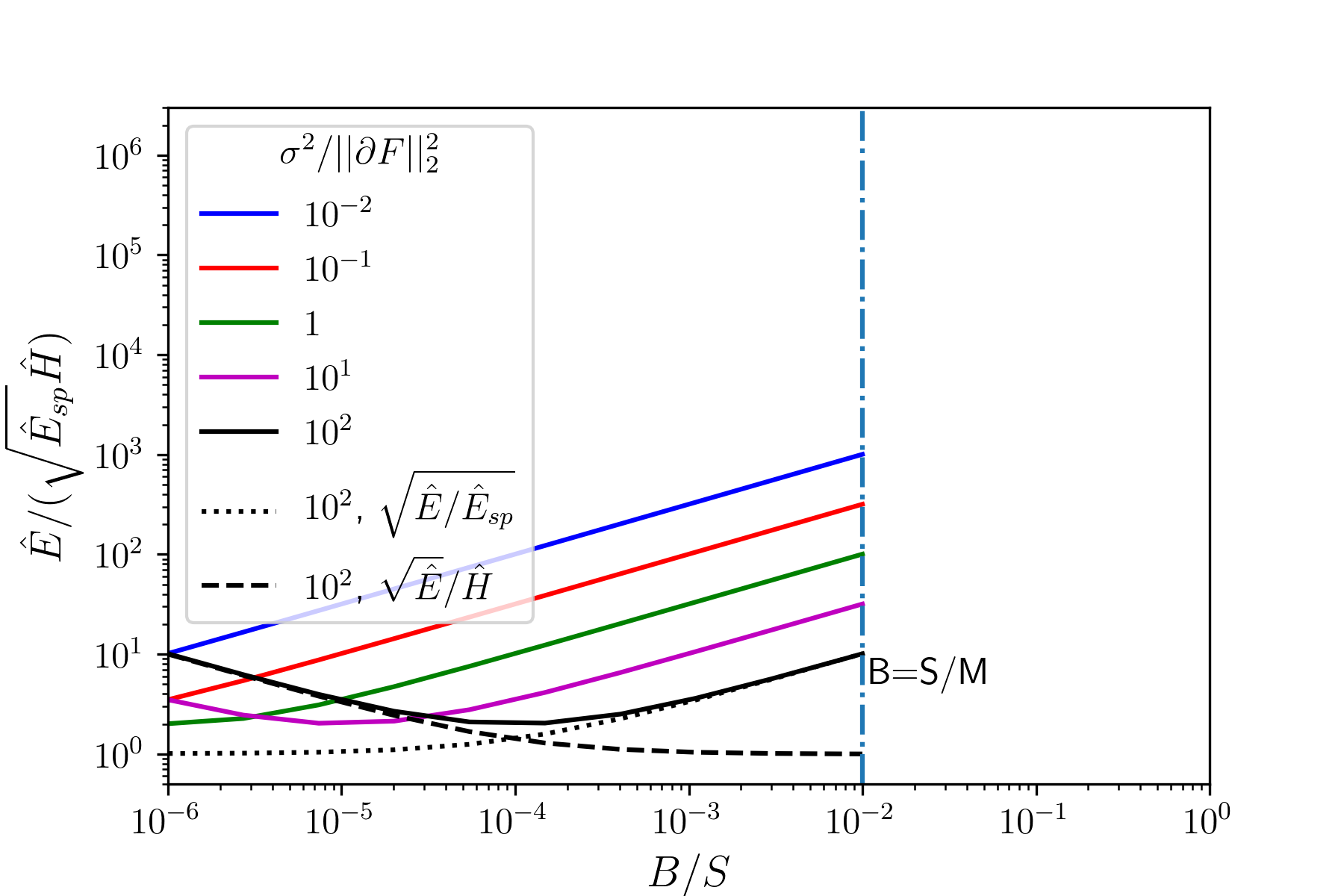}
        \caption[]{\small $C = 1$}
        %\caption[]{\small Estimate of $E/(\sqrt{\Esp{}} H)$ versus the relative batch size $B/S$ for $M=100$, $S=10^6$, $C=1$, and different level of heterogeneity ($\sigma^2/\norm*{\partial F}_2^2$) of the dataset. }
        \label{f:ratio_c1}
    \end{subfigure}
    \caption[]{\small Estimate of $E/(\sqrt{\Esp{}} H)$ versus the relative batch size $B/S$ for $M=100$, $S=10^6$ and different level of heterogeneity.}
    \label{f:ratio_total}
\end{figure*}

 Figure~\ref{f:ratio_total} illustrates the ratio $\widehat E/(\sqrt{\widehat{E}_{\textrm{sp}}} \widehat H) (=\beta \alpha)$ for a particular setting. It also highlights the two regimes discussed above:  $\beta \approx 1/\alpha \times \sqrt{E/\Esp{}}$ for large batch sizes and $\beta \approx 1/\alpha \times \sqrt{E}/H$ for small ones.
 %, where for $H$ we have considered the upper-bound in~\eqref{e:estimates}. 
 As $\beta$ indicates how much looser bound~\eqref{e:classic_bound} is in comparison to bound~\eqref{e:new_bound}, and $\beta > E/(\sqrt{\Esp{}} H)$, the figure shows that~\eqref{e:classic_bound} may indeed overestimate the effect of the  topology by many orders of magnitudes.
The comparison of Fig.~\ref{f:ratio_c1} and Fig.~\ref{f:ratio} shows that $\frac{E}{\sqrt{\Esp{}} H}$ (and then $\beta$)  does not depend much on the replication factor $C$, but for the fact that the batch size can scale up to $C S/M$.

%\subsection{A toy example}
%\label{s:toy_example}
%\input{toy_example.tex}

%\section{Topology and throughput} 
%\label{s:speed}
%\input{speed.tex}

\section{EXPERIMENTS}
\label{s:experiments}

\begin{table*}[t]
\caption{Empirical estimation of $E$, $E_{\textrm{sp}}$, $H$, $\alpha$ on different ML problems and comparison of their joint effect ($\beta$) with the value $\widehat \beta$ predicted through~\eqref{e:perm_approx}. Number of iterations by which training losses for the ring and the clique differ by $4\%$, $10\%$, as predicted by the old bound~\eqref{e:classic_bound}, $k'_o$, by the new one~\eqref{e:new_bound}, $k'_n$, and as measured in the experiment, $k'$. When a value exceeds the total number of iterations we ran (respectively 1200 for CT, 1190 for MNIST, and 1040 for CIFAR-10), we simply indicate it as $\infty$.}
\label{t:parameters}
\setlength\tabcolsep{2.5pt} 
\centering
\begin{tabular}{cccccccccccccccccc}
\hline
\multirow{2}{*}{Dataset} & \multirow{2}{*}{Model}& \multirow{2}{*}{M}& \multirow{2}{*}{B} & \multirow{2}{*}{$\eta$} & \multirow{2}{*}{$\sqrt{E/E_{sp}}$}& \multirow{2}{*}{$\sqrt{E}/H$} & \multirow{2}{*}{$\frac{1}{\alpha}$} & \multirow{2}{*}{$\beta$} & \multirow{2}{*}{$\widehat \beta$}& \multicolumn{3}{c}{$@ 4\%$}  & \multicolumn{3}{c}{$@ 10\%$}\\
&&&&&&&&&&$k'_o$& $k'_n$& $k'$ & $k'_o$& $k'_n$ & $k'$\\
\hline 
\multirow{4}{*}{ \begin{tabular}{c}CT\\ (S=52000)\end{tabular}}& 
\multirow{4}{*}{ \begin{tabular}{c}Linear regr.\\ n=384\end{tabular}}& 
\multirow{2}{*}{16}        &128 & \multirow{4}{*}{0.0003}  & 7.92  &1.01 & 1.53 & 12.23 & 12.31& 1 & $\infty$ & $\infty$ & 1& $\infty$ & $\infty$ \\
            & &            & 3250 & & 38.45  &    1.00&1.64 & 62.86 & 60.97&1 & $\infty$ & $\infty$ & 1& $\infty$ & $\infty$ \\ 
&&\multirow{2}{*}{100}     & 128 & & 7.75  &1.01 &1.54 &12.05 & 11.56&1 & 10 & $\infty$ & 1& $\infty$ & $\infty$\\
& &                        & 520 & & 15.58 &1.00 &1.51 &23.60 & 22.96&1 &17 & $\infty$ & 1& $\infty$ & $\infty$\\
\hline
\multirow{3}{*}{ \begin{tabular}{c}MNIST\\ (S=60000)\end{tabular}}& 
\multirow{4}{*}{ \begin{tabular}{c}2-conv layers\\ n=431080\end{tabular}}
& \multirow{2}{*}{16} &128 & \multirow{3}{*}{0.1} & 1.45 &1.42 &1.49 & 3.07 &2.92& 1 &16 & $\infty$ & 1& 72 & $\infty$\\
     & &                   &500 & & 2.15 & 1.14 &1.53 & 3.75 &3.71 & 1 & 22& 40&1 & 260 &  $\infty$ \\
%    & &                   &1000 & 2.64 & 1.08 &1.48 & 4.22 & \\
& &                     64    & 128 & & 1.41 & 1.42 & 1.51  &3.02 &3.03& 1 &10 &$\infty$ & 1& 24& $\infty$\\
\cline{1-1}\cline{3-16}
split by digit&&           10   &500 & 0.01 &1.01 & 1.00 &1.42 & 1.42 & 3.62& 1 & 3 & 60 & 1& 7& 100\\
\hline
\multirow{2}{*}{ \begin{tabular}{c}CIFAR-10\\ (S= 50000)\end{tabular}}& 
\multirow{2}{*}{ \begin{tabular}{c}ResNet18\\ n=11173962 \end{tabular}}& 
\multirow{2}{*}{16}        &128 & \multirow{2}{*}{0.05}  &1.07&3.35 &1.49 & 5.34&5.62 &1 & 10&30 & 1& 20&$\infty$\\
            & &          & 500 & &1.18&1.91 & 1.50&  3.40& 3.52& 1  & 21& $\infty$ & 1& $250$& $\infty$\\
%             & &          &   & & & &  &  &   & & & & & \\ 
\hline
\end{tabular}
\end{table*}

    \begin{figure*}[t]
%\vspace{.3in}
        \centering
        \begin{subfigure}[b]{0.325\textwidth}
            \centering
            \includegraphics[width=\textwidth]{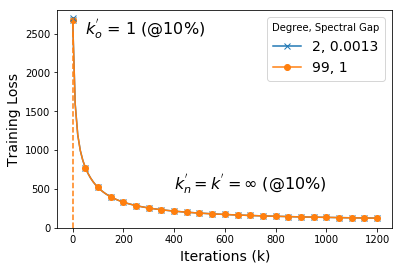}
            \caption[]%
            {{\small CT: M=100, B = 128}}
            \label{f:regr_err_vs_iter}
        \end{subfigure}
        \hfill
        \begin{subfigure}[b]{0.325\textwidth}  
            \centering 
            \includegraphics[width=\textwidth]{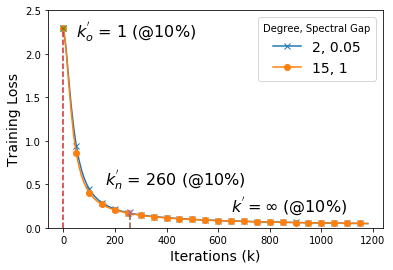}
            \caption[]%
            {{\small MNIST: M=16, B = 500}}    
            \label{f:regr_iter_vs_time}
        \end{subfigure}
        %\vskip\baselineskip
        \hfill
        \begin{subfigure}[b]{0.325\textwidth}   
            \centering 
            \includegraphics[width=\textwidth]{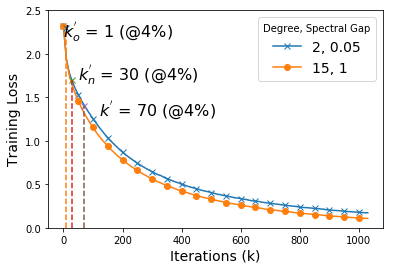}
            \caption[]%
            {{\small CIFAR-10: M=16, B = 500}}    
            \label{f:regr_error_vs_time}
        \end{subfigure}
         \caption[ The average and standard deviation of critical parameters ]
        {\small Effect of network connectivity  (degree $d$) on the iterations to convergence.} 
        \label{f:typical_behaviour}
    \end{figure*}

\begin{figure}[ht]
    \centering
    \includegraphics[scale=0.4]{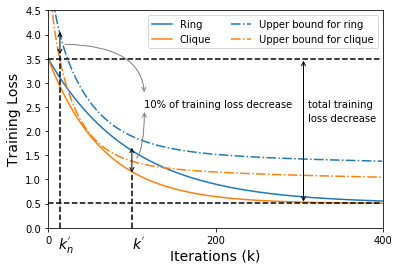}
    \caption{How to determine the number of iteration at which training loss for the clique and for the ring differs significantly.}
    \label{f:qualitative}
\end{figure}

With our experiments we want to 1) evaluate the effect of topology on the number of epochs to converge, and in particular quantify $E$, $\Esp{}$, $H$, and $\alpha$ in practical ML problems, 2) evaluate the effect of topology on the convergence \emph{time}.
We considered three different optimization problems:
\begin{enumerate}
%\item Minimization of hinge loss function for classification on the dataset SUSY from~\citep{ucimlr} ($S=5 \times 10^5$, $n=18$~\citep{baldi14}). In this case both convexity  and bounded energy subgradients  assumptions (A1 and A5) hold.
\item Minimization of  mean squared error (MSE) for linear regression on the dataset ``Relative location of CT slices on axial axis'' from~\citep{ucimlr,graf11}. Convexity holds, but gradients are potentially unbounded.
%\item Minimization of hinge loss function for classification on the dataset SUSY from~\cite{ucimlr} ($S=5 \times 10^5$ samples each with $n=18$ features consisting of kinematic properties measured by the particle detectors in an accelerator~\cite{baldi14}). In this case both convexity  and bounded energy subgradients  assumptions (A1 and A5) hold.
%\item Minimization of  mean squared error (MSE) for linear regression on the dataset ``Relative location of CT slices on axial axis'' from~\cite{ucimlr} ($S=53500$ samples each with $n=385$ features extracted from computer tomography images~\cite{graf11}). Convexity holds, but gradients are potentially unbounded.
\item Minimization of cross-entropy loss through a neural network with two convolutional layers on MNIST dataset~\citep{Lecun}. Neither convexity, nor subgradient boundness hold.
\item Minimization of cross-entropy loss through ResNet18 neural network~\citep{he2016deep} on CIFAR-10 dataset~\citep{Krizhevsky09learningmultiple}. Neither convexity, nor subgradient boundness hold. Moreover, we employ local subgradients with classical momentum~\citep{sutskever2013importance} (with coefficient 0.9). % $g_i(\vw_k)$ in \eqref{e:dsm} by adding the momentum part.  
\end{enumerate}
We have developed an ad-hoc Python simulator that allows us to test clusters with a large number of nodes, as well as a distributed application using PyTorch MPI backend to run experiments on a real GPU cluster platform.\footnote{
The platform is composed of various types of GPUs, e.g., GeForce GTX 1080 Ti, GeForce GTX Titan X and Nvidia Tesla V100.     
%	%Both the Python simulator and the PyTorch implementation are available as supplementary material.
} 
In general, datasets have been randomly split across the different workers without any replication ($C=1$). For MNIST we have also considered a scenario with $M=10$ workers, where each worker has been assigned all images for a specific digit. 
%\todo{each worker is assigned a unique class of digits}
%Data points replication ($C>1$) further reduces the sensitivity to the topology.
A constant learning rate has been set using the configuration rule from~\citep{smith17} described in \inciteapp{a:experiments}. 
%We increase geometrically the learning rate and determine two ``knees'': the  learning rate value for which the loss starts decreasing significantly and the value for which it starts increasing again.  We set the learning rate to the geometric average of these two values. 
Interestingly, for a given ML problem, when the dataset is split randomly, this procedure has led to choose the learning rate independently of the average node degree. The values selected are indicated in  Table~\ref{t:parameters}. %Consequently, differences observed across different topologies cannot be decreased by selecting larger learning rates. 
%This excludes that the difference we observe among different topologies could be systematically compensated by selecting larger learning rates. 
Each node starts from the same model parameters ($\Rsp{}=0$) that have been initialized through PyTorch default functions. 
%standard parameter initialization for neural network layers~\cite{he2015delving}.
We report here a subset of all results, the others can be found in \inciteapp{a:experiments}. 

Table~\ref{t:parameters} shows values of $\sqrt{E/\Esp{}}$, $\sqrt{E}/H$, $1/\alpha$, and their product $\beta$ for different problems and different settings.\footnote{
    Some additional experiments in \inciteapp{a:experiments}  show that $\Rsp$ and $R$ have a smaller effect on the bounds, as the third term in~\eqref{e:new_bound} and in~\eqref{e:classic_bound} converges to $0$ when $K$ diverges.
} $E$, $\Esp{}$, and $H$ have been evaluated through empirical averages using the random minibatches drawn at the first iteration. $\alpha$ is computed for an undirected ring topology.
Remember that the value $\beta$ (defined in~\eqref{e:beta}) indicates how much tighter the new bound~\eqref{e:new_bound} is in comparison to the classic one~\eqref{e:classic_bound}. We also use~\eqref{e:perm_approx} 
%and Proposition~\ref{p:estimates} 
to provide an estimate of $\beta$ as follows $\widehat \beta = 1/\alpha \times \widehat E/(\sqrt{\widehat E_{\textrm{sp}}} \widehat H)$.
The approximation is very accurate when the dataset is split randomly across the nodes. On the other hand, for MNIST, when all images for a given digit are assigned to the same node, local datasets are very different and  approximations \eqref{e:perm_approx} are too crude (but our bound~\eqref{e:new_bound} still holds).
 %We have also evaluated the same quantities at different iterations, but the results are almost the same.
Interestingly, $\beta$ is dominated by different effects for the three ML problems.
%the main contribution to $\beta$ comes from different effects (even when number of workers and batch sizes are equal as for~$M=16$ and $B=128$): 
The similarity of local datasets prevails for CT (large $\sqrt{E/\Esp{}}$), while the noise of stochastic subgradients prevails for CIFAR (large $\sqrt{E}/H$). For MNIST the three effects, including energy spreading over different eigenspaces  ($1/\alpha$), contribute almost equally. This can be explained considering that, even if local datasets have similar sizes, they are statistically more different the more complex the model  to train, i.e.,~the \mbox{larger $n$}.
%a more complex model can overfit the local datasets).

%\todo{new Fig 3}
From~\eqref{e:new_bound} and~\eqref{e:classic_bound}, we can also compute at which iteration the two bounds predict that the effect of the topology becomes significant, by identifying when the training loss difference between the clique and the ring accounts for a given percentage of the loss decrease over the entire training period. Figure~\ref{f:qualitative} qualitatively illustrates the procedure.\footnote{
    In order to be able to compare the upper bounds \eqref{e:classic_bound} and \eqref{e:new_bound} with the actual loss curves, we rescale them by a factor determined so that the upper bound curve and the experimental one are tangent for the clique topology (Fig.~\ref{f:qualitative}). Moreover, once determined at which iteration rings and cliques should differ, we update the upper-bounds with new estimates for $E$, $\Esp{}$, $H$, $R$, and $\Rsp{}$ computed at this iteration, and check if they now predict a larger number of iterations.  
} These predictions are indicated in the last columns of Table~\ref{t:parameters} and are compared with the values observed in the experiments ($k'$).
%\footnote{
%    We ignore the third term because all the initial parameter vectors are equal to zero so that $\Rsp=0$.
%}
%The values of $\beta$ observed in our experiments range from $3$ to $60$. Nevertheless, the effect of 

We note that forecasts are very different, despite the fact that, in some settings, our bound is only 3 times tighter than the classic one. Bound~\eqref{e:classic_bound} predicts that the training loss curves should differ by more than $10\%$ since the first iteration ($k'_o=1$). The new bound~\eqref{e:new_bound} correctly identifies that the topology's effect becomes evident later, sometimes beyond the total number of iterations performed in the experiment (in this case we indicate $k'_n=\infty$).
%The two predictions are roughly in a ratio equal to $\beta$, that can range from $3$ to more than $60$ (Table~\ref{t:parameters}). 

Figure~\ref{f:typical_behaviour} shows the training loss evolution $F(\hat{\overline\vw}(k))$ for specific settings (one for each ML problem) and two very different topologies (undirected ring and clique), 
when the dataset is split randomly across the nodes. The behaviour is qualitatively similar to what observed in previous works~\citep{lian17nips, lian18icml, luo19}; despite the remarkable difference in the level of connectivity (quantified also by the spectral gap), the curves are very close, sometimes indistinguishable. 
%such predictions ($K_1$ from~\eqref{e:classic_bound} and $K_2$ from~\eqref{e:new_bound}) together with the cross-entropy evolution for different levels of connectivity of the communication topology. For CIFAR-10, our bound predicts that the effect of the topology will be significant only after XXX iterations that is beyond the maximum of the x axis. For MNIST, we observe that the topology has still no evident effect even after $K_2$, suggesting that the dependency on the topology is even weaker than what our bound suggests. 

\begin{figure}
    \centering
    \includegraphics[scale=0.4]{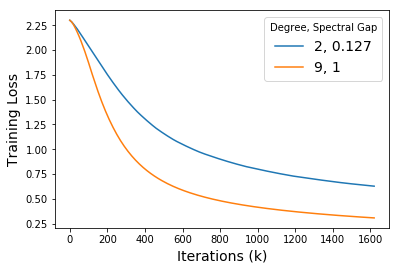}
    \caption{MNIST split by digit, M=10, B=500.}
    \label{f:minist_by_digit}
\end{figure}
Figure~\ref{f:minist_by_digit} shows the same plot for the case when MNIST 
%dataset has not been randomly split across the nodes (generating similar local  datasets), but all 
images for the same digit have been assigned to the same node. In this case the local datasets are very different and $\sqrt{E/\Esp{}}\approx 1$; the topology has a remarkable effect! This plot warns against extending the empirical finding in~\citep{lian17nips,lian18icml,luo19} to settings where local datasets can be highly different as it can be for example in the case of federated learning~\citep{konecny15}.
%\todo{If we could somehow shrink the previous paragraph to put this part to the previous page, we could succeed to have 8 pages.}

%PART WHERE WE DISCUSS THE EMPIRICAL EVALUATION OF E Esp AND H, WE COMPARE THEM WITH THE THEORETICAL ESTIMATES AND WE SHOW THEIR EFFECT OVER THE TIME AFTER WHICH THE TOPOLOGY SHOULD KICK IN.

\begin{figure*}[h]
%\vspace{.3in}
        \centering
        \begin{subfigure}[b]{0.325\textwidth}  
            \centering 
            \includegraphics[width=\textwidth]{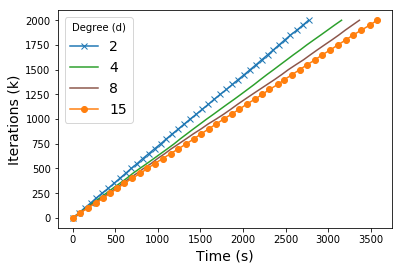}
            \caption[]%
            {{\small Throughput}}    
            \label{f:nn_iter_vs_time}
        \end{subfigure}
        \hfill
        %\vskip\baselineskip
        \begin{subfigure}[b]{0.325\textwidth}
            \centering
            \includegraphics[width=\textwidth]{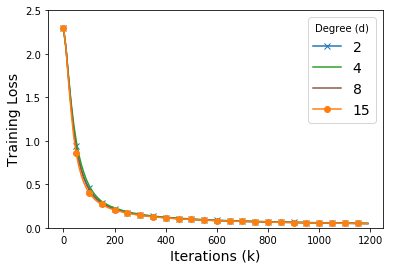}
            \caption[]%
            {{\small Error vs iterations}}    
            \label{f:nn_err_vs_iter}
        \end{subfigure}
        \hfill
        \begin{subfigure}[b]{0.325\textwidth}   
            \centering 
            \includegraphics[width=\textwidth]{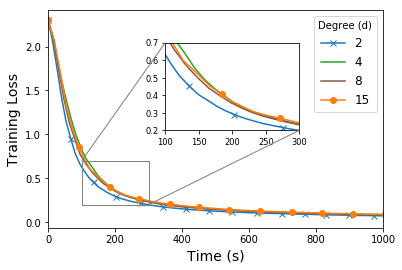}
            \caption[]%
            {{\small Error vs time}}    
            \label{f:nn_error_vs_time}
        \end{subfigure}
         \caption[ The average and standard deviation of critical parameters ]
        {\small Effect of network connectivity  (degree $d$) on the convergence for dataset MNIST with computation times from a Spark cluster. M = 16, B = 500.} 
        \label{f:nn}
    \end{figure*}
%\todo{to remove 49 from Fig 2a}
The experiments above confirm that the communication topology has little influence on the number of \emph{epochs} needed to converge (when local datasets are statistically similar). Our analysis reconciles (at least in part) theory and experiments by pushing farther the training epoch at which the effect of the topology should be evident.

The conclusion about the role of the topology is radically different if one considers the \emph{time} to converge. For example, \cite{karakus17} and \cite{luo19} observe experimentally that sparse topologies can effectively reduce the convergence wall-clock time. A possible explanation is that each iteration is faster because less time is spent in the communication phase: the less connected the graph, the smaller  the communication load at each node. \cite{lian17nips,lian18icml} advance this explanation to justify  why DSM on ring-like topologies can converge faster than the \mbox{centralized PS}. 
%This time is equal to the product of the number of iterations to converge ($K_\epsilon$) and the average time required to carry out one iteration ($\theta$).

Here, we show that sparse topologies can speed-up wall-clock time convergence even when communication costs are negligible, because they intrinsically mitigate the effect of stragglers, i.e.,~tasks whose completion time can be occasionally much longer than its typical value. Transient slowdowns are common in computing systems (especially in shared ones) and have many causes, such as resource contention, background OS activities, garbage collection, and (for ML tasks) stopping criteria calculations. Stragglers can significantly reduce computation speed in a multi-machine setting~\citep{ananthanarayanan13,karakus17,li18}. For consensus-based method, one can hope that, when the topology is sparse, a temporary straggler only slows down a limited number of nodes (its out-neighbors in $\mathcal G$), so that the system can still maintain a high throughput.

\cite{neglia19infocom} have proposed approximate formulas to evaluate the throughput of distributed ML systems for some specific random distribution of the computation time (uniform, exponential, and Pareto). Here, we take a more practical approach. Our PyTorch-based distributed application allow us to simulate systems with arbitrary distributions of the computation times and communication delays. We have carried out experiments with zero communication delays (an ideal network) and two different empirical distributions for the computation time. One was obtained by running stochastic gradient descent on a production Spark cluster with sixteen servers using Zoe Analytics~\citep{pace17},
 each with two 8-core Intel E5-2630 CPUs running at 2.40GHz.
 %(total of 32 cores with hyper-threading enabled).
 %, 128GB of memory, 1Gbps Ethernet network
%fabric and ten 1TB hard drives. 
The other was extracted from ASCI-Q super-computer traces~{\cite[Fig.~4]{petrini03}}. 
Figure~\ref{f:nn}~shows the effect of topology connectivity on the convergence time for a MNIST experiment with Spark computation distribution. We consider in this case undirected $d$-regular random graphs.
The number of iterations completed per node grows faster the less connected the topology (Fig.~\ref{f:nn}~(a)). As the training loss is almost independent of the topology (Fig.~\ref{f:nn}~(b)), the ring achieves the shortest  convergence time (Fig.~\ref{f:nn}~(c)), even if there is no communication delay. Qualitatively similar results for other ML problems and time distributions are in \inciteapp{a:experiments}.

\section{CONCLUSIONS}
\label{s:concl}
%!TEX root = icml.tex

We have explained, both through analysis and experiments, when and why the communication topology does not affect the number of epochs consensus-based optimization methods need to converge, an effect recently observed in many papers, but not thoroughly investigated. We have also shown that, as a consequence of this invariance, a less connected topology achieves  a shorter convergence time, not necessarily because it incurs a smaller communication load, but because it mitigates the stragglers' problem. 
The distributed operation of consensus-based approaches appears particularly suited for federated and multi-agent learning. Our study points out that further research is required for these applications, because the benefits observed until now are dependent on the statistical similarity of the local datasets, an assumption that is not satisfied in federated learning.

%can also  balanced communication pattern  that lead to a shorter iteration time also from topology We have also shown that whenwhenever the Whenever this independence  Moreover, we have shown that the effect of topologies can lead to faster convergence of consensus-based gradient methods, even when communication delays can be neglected. This gain is achieved through spatial model inconsistency, which mitigates the effect of stragglers. A number of interesting research questions are still open: how to combine spatial and chronological model inconsistency? how to dynamically adapt the topology during the optimization to react to observed straggling workers? how do consensus-based gradient methods generalize in comparison to more common paradigms like PS or ring all-reduce, where a unique parameter vector is maintained in the system?

%\todo{Since the acknowledgements could be excluded, I moved this part here.}
\section{ACKNOWLEDGEMENTS}
We warmly thank Alain Jean-Marie, Pietro Michiardi, and Bruno Ribeiro for their feedback on the manuscript and their help preparing the rebuttal letter.
We are also grateful to the OPAL infrastructure from Université Côte d'Azur and Inria Sophia Antipolis - Méditerranée ``NEF" computation platform for providing resources and support.
%We are also grateful to Inria IT services for providing support to use the 
%``Nef" computation cluster.
Finally, We would like to thank the anonymous reviewers for
their insightful comments and suggestions, which helped
us improve this work.

This work was supported in part by ARL under Cooperative Agreement W911NF-17-2-0196.

\bibliographystyle{plainnat}
\bibliography{asynchronous}

\begin{thebibliography}{35}
\providecommand{\natexlab}[1]{#1}
\providecommand{\url}[1]{\texttt{#1}}
\expandafter\ifx\csname urlstyle\endcsname\relax
  \providecommand{\doi}[1]{doi: #1}\else
  \providecommand{\doi}{doi: \begingroup \urlstyle{rm}\Url}\fi

\bibitem[uci()]{ucimlr}
{UCI Machine Learning Repository}.
\newblock \url{http://archive.ics.uci.edu/ml/datasets/}.

\bibitem[Abadi et~al.(2016)Abadi, Barham, Chen, Chen, Davis, Dean, Devin,
  Ghemawat, Irving, Isard, Kudlur, Levenberg, Monga, Moore, Murray, Steiner,
  Tucker, Vasudevan, Warden, Wicke, Yu, and Zheng]{abadi16}
Mart\'{\i}n Abadi, Paul Barham, Jianmin Chen, Zhifeng Chen, Andy Davis, Jeffrey
  Dean, Matthieu Devin, Sanjay Ghemawat, Geoffrey Irving, Michael Isard,
  Manjunath Kudlur, Josh Levenberg, Rajat Monga, Sherry Moore, Derek~G. Murray,
  Benoit Steiner, Paul Tucker, Vijay Vasudevan, Pete Warden, Martin Wicke, Yuan
  Yu, and Xiaoqiang Zheng.
\newblock Tensorflow: A system for large-scale machine learning.
\newblock In \emph{Proceedings of the 12th USENIX Conference on Operating
  Systems Design and Implementation}, OSDI'16, pages 265--283, 2016.
\newblock ISBN 978-1-931971-33-1.

\bibitem[Ananthanarayanan et~al.(2013)Ananthanarayanan, Ghodsi, Shenker, and
  Stoica]{ananthanarayanan13}
Ganesh Ananthanarayanan, Ali Ghodsi, Scott Shenker, and Ion Stoica.
\newblock Effective straggler mitigation: Attack of the clones.
\newblock In \emph{Proc. of the 10th USENIX Conf. NSDI}, 2013.

\bibitem[Assran et~al.(2019)Assran, Loizou, Ballas, and
  Rabbat]{DBLP:conf/icml/AssranLBR19}
Mahmoud Assran, Nicolas Loizou, Nicolas Ballas, and Michael Rabbat.
\newblock Stochastic gradient push for distributed deep learning.
\newblock In \emph{Proceedings of the 36th International Conference on Machine
  Learning, {ICML} 2019}, volume~97 of \emph{Proceedings of Machine Learning
  Research}, pages 344--353. {PMLR}, 2019.

\bibitem[Baldi et~al.(2014)Baldi, Sadowski, and Whiteson]{baldi14}
Pierre Baldi, Peter Sadowski, and Daniel Whiteson.
\newblock Searching for exotic particles in high-energy physics with deep
  learning.
\newblock \emph{Nature communications}, 5:\penalty0 4308, 2014.

\bibitem[Canini et~al.(2014)Canini, Chandra, Ie, McFadden, Goldman, Gunter,
  Harmsen, LeFevre, Lepikhin, Lloret~Llinares, Mukherjee, Pereira, Redstone,
  Shaked, and Singer]{canini14}
Kevin Canini, Tushar Chandra, Eugene Ie, Jim McFadden, Ken Goldman, Mike
  Gunter, Jeremiah Harmsen, Kristen LeFevre, Dmitry Lepikhin, Tomas
  Lloret~Llinares, Indraneel Mukherjee, Fernando Pereira, Josh Redstone, Tal
  Shaked, and Yoram Singer.
\newblock Sibyl: A system for large scale supervised machine learning, 2014.
\newblock Technical talk.

\bibitem[Duchi et~al.(2012)Duchi, Agarwal, and Wainwright]{duchi12}
John~C. Duchi, Alekh Agarwal, and Martin~J. Wainwright.
\newblock Dual averaging for distributed optimization: Convergence analysis and
  network scaling.
\newblock \emph{IEEE Trans. on Automatic Control}, 57\penalty0 (3):\penalty0
  592--606, 2012.

\bibitem[Gibiansky(2017)]{ringallreduce17}
Andrew Gibiansky.
\newblock Bringing hpc techniques to deep learning.
\newblock online, \url{http://research.baidu.
  com/bringing-hpc-techniques-deep-learning}, 2017.

\bibitem[{Google I/O}(2018)]{distr_tf18}
{Google I/O}.
\newblock Distributed tensorflow training.
\newblock online, \url{https://www.youtube.com/watch?v=bRMGoPqsn20}, 2018.

\bibitem[Graf et~al.(2011)Graf, Kriegel, Schubert, P{\"o}lsterl, and
  Cavallaro]{graf11}
Franz Graf, Hans-Peter Kriegel, Matthias Schubert, Sebastian P{\"o}lsterl, and
  Alexander Cavallaro.
\newblock {2D Image Registration in CT Images Using Radial Image Descriptors}.
\newblock In \emph{Proc. of MICCAI}, pages 607--614, 2011.

\bibitem[He et~al.(2016)He, Zhang, Ren, and Sun]{he2016deep}
Kaiming He, Xiangyu Zhang, Shaoqing Ren, and Jian Sun.
\newblock Deep residual learning for image recognition.
\newblock In \emph{Proceedings of the IEEE conference on computer vision and
  pattern recognition}, pages 770--778, 2016.

\bibitem[Karakus et~al.(2017)Karakus, Sun, Diggavi, and Yin]{karakus17}
Can Karakus, Yifan Sun, Suhas Diggavi, and Wotao Yin.
\newblock Straggler mitigation in distributed optimization through data
  encoding.
\newblock In \emph{Proc. of NIPS}, pages 5434--5442. 2017.

\bibitem[Koloskova et~al.(2019)Koloskova, Stich, and
  Jaggi]{DBLP:conf/icml/KoloskovaSJ19}
Anastasia Koloskova, Sebastian~U. Stich, and Martin Jaggi.
\newblock Decentralized stochastic optimization and gossip algorithms with
  compressed communication.
\newblock In \emph{Proceedings of the 36th International Conference on Machine
  Learning, {ICML}}, volume~97 of \emph{Proceedings of Machine Learning
  Research}, pages 3478--3487. {PMLR}, 2019.

\bibitem[Konecn\'y et~al.(2015)Konecn\'y, McMahan, and Ramage]{konecny15}
Jakub Konecn\'y, Brendan McMahan, and Daniel Ramage.
\newblock {Federated Optimization: Distributed Optimization Beyond the
  Datacenter}.
\newblock In \emph{Neural Information Processing Systems (workshop)}, 2015.

\bibitem[Krizhevsky(2009)]{Krizhevsky09learningmultiple}
Alex Krizhevsky.
\newblock Learning multiple layers of features from tiny images.
\newblock Technical report, 2009.

\bibitem[{Lecun} et~al.(1998){Lecun}, {Bottou}, {Bengio}, and {Haffner}]{Lecun}
Y.~{Lecun}, L.~{Bottou}, Y.~{Bengio}, and P.~{Haffner}.
\newblock Gradient-based learning applied to document recognition.
\newblock \emph{Proceedings of the IEEE}, 86\penalty0 (11):\penalty0
  2278--2324, Nov 1998.

\bibitem[Li et~al.(2014)Li, Andersen, Smola, and Yu]{li14nips}
Mu~Li, David~G Andersen, Alexander~J Smola, and Kai Yu.
\newblock Communication efficient distributed machine learning with the
  parameter server.
\newblock In Z.~Ghahramani, M.~Welling, C.~Cortes, N.~D. Lawrence, and K.~Q.
  Weinberger, editors, \emph{Advances in Neural Information Processing Systems
  27}, pages 19--27. Curran Associates, Inc., 2014.

\bibitem[Li et~al.(2018)Li, Kalan, Avestimehr, and Soltanolkotabi]{li18}
Songze Li, Seyed Mohammadreza~Mousavi Kalan, A.~Salman Avestimehr, and Mahdi
  Soltanolkotabi.
\newblock {Near-Optimal Straggler Mitigation for Distributed Gradient Methods}.
\newblock In \emph{Proc. of the 7th Intl. Workshop ParLearning}, May 2018.

\bibitem[Lian et~al.(2017)Lian, Zhang, Zhang, Hsieh, Zhang, and
  Liu]{lian17nips}
Xiangru Lian, Ce~Zhang, Huan Zhang, Cho-Jui Hsieh, Wei Zhang, and Ji~Liu.
\newblock {Can Decentralized Algorithms Outperform Centralized Algorithms? A
  Case Study for Decentralized Parallel Stochastic Gradient Descent}.
\newblock In \emph{NIPS}, 2017.

\bibitem[Lian et~al.(2018)Lian, Zhang, Zhang, and Liu]{lian18icml}
Xiangru Lian, Wei Zhang, Ce~Zhang, and Ji~Liu.
\newblock Asynchronous decentralized parallel stochastic gradient descent.
\newblock In \emph{ICML}, 2018.

\bibitem[Luo et~al.(2019)Luo, Lin, Zhuo, and Qian]{luo19}
Qinyi Luo, Jinkun Lin, Youwei Zhuo, and Xuehai Qian.
\newblock Hop: Heterogeneity-aware decentralized training.
\newblock In \emph{Proceedings of the Twenty-Fourth International Conference on
  Architectural Support for Programming Languages and Operating Systems},
  ASPLOS '19, pages 893--907, 2019.

\bibitem[McKay(1981)]{mckay81}
Brendan~D. McKay.
\newblock The expected eigenvalue distribution of a large regular graph.
\newblock \emph{Linear Algebra and its Applications}, 40:\penalty0 203 -- 216,
  1981.

\bibitem[McKay and Wormald(1990)]{mckay90}
Brendan~D. McKay and Nicholas~C. Wormald.
\newblock Uniform generation of random regular graphs of moderate degree.
\newblock \emph{J. Algorithms}, 11\penalty0 (1):\penalty0 52--67, February
  1990.
\newblock ISSN 0196-6774.

\bibitem[Meyer(2000)]{meyer00}
Carl~D. Meyer, editor.
\newblock \emph{Matrix Analysis and Applied Linear Algebra}.
\newblock SIAM, Philadelphia, PA, USA, 2000.
\newblock ISBN 0-89871-454-0.

\bibitem[Nedi\'c and Ozdaglar(2009)]{nedic09}
Angelia Nedi\'c and Asuman~E. Ozdaglar.
\newblock Distributed subgradient methods for multi-agent optimization.
\newblock \emph{{IEEE} Trans. Automat. Contr.}, 54\penalty0 (1):\penalty0
  48--61, 2009.

\bibitem[Nedi\'c et~al.(2018)Nedi\'c, Olshevsky, and Rabbat]{nedic18}
Angelia Nedi\'c, Alex Olshevsky, and Michael~G. Rabbat.
\newblock Network topology and communication-computation tradeoffs in
  decentralized optimization.
\newblock \emph{Proc. of the IEEE}, 106\penalty0 (5):\penalty0 953--976, May
  2018.

\bibitem[Neglia et~al.(2019)Neglia, Calbi, Towsley, and
  Vardoyan]{neglia19infocom}
Giovanni Neglia, Gianmarco Calbi, Don Towsley, and Gayane Vardoyan.
\newblock {The Role of Network Topology for Distributed Machine Learning}.
\newblock In \emph{{IEEE International Conference on Computer Communications
  (INFOCOM)}}, 2019.

\bibitem[Pace et~al.(2017)Pace, Venzano, Carra, and Michiardi]{pace17}
Francesco Pace, Daniele Venzano, Damiano Carra, and Pietro Michiardi.
\newblock Flexible scheduling of distributed analytic applications.
\newblock In \emph{Proceedings of the 17th IEEE/ACM International Symposium on
  Cluster, Cloud and Grid Computing}, CCGrid '17, pages 100--109, Piscataway,
  NJ, USA, 2017. IEEE Press.
\newblock ISBN 978-1-5090-6610-0.

\bibitem[Petrini et~al.(2003)Petrini, Kerbyson, and Pakin]{petrini03}
Fabrizio Petrini, Darren~J. Kerbyson, and Scott Pakin.
\newblock The case of the missing supercomputer performance: Achieving optimal
  performance on the 8,192 processors of asci q.
\newblock In \emph{Proceedings of the 2003 ACM/IEEE Conference on
  Supercomputing}, SC '03, pages 55--, 2003.

\bibitem[Pu et~al.(2019)Pu, Olshevsky, and
  Paschalidis]{olshevsky19nonasymptotic}
Shi Pu, Alex Olshevsky, and Ioannis~Ch. Paschalidis.
\newblock A non-asymptotic analysis of network independence for distributed
  stochastic gradient descent, 2019.
\newblock arXiv preprint arXiv:1906.02702v9.

\bibitem[Smith(2017)]{smith17}
Leslie~N Smith.
\newblock Cyclical learning rates for training neural networks.
\newblock In \emph{Applications of Computer Vision (WACV), 2017 IEEE Winter
  Conference on}, pages 464--472. IEEE, 2017.

\bibitem[Smola and Narayanamurthy(2010)]{smola10}
Alexander Smola and Shravan Narayanamurthy.
\newblock An architecture for parallel topic models.
\newblock \emph{Proc. VLDB Endow.}, 3\penalty0 (1-2):\penalty0 703--710,
  September 2010.
\newblock ISSN 2150-8097.

\bibitem[Sutskever et~al.(2013)Sutskever, Martens, Dahl, and
  Hinton]{sutskever2013importance}
Ilya Sutskever, James Martens, George Dahl, and Geoffrey Hinton.
\newblock On the importance of initialization and momentum in deep learning.
\newblock In \emph{International conference on machine learning}, pages
  1139--1147, 2013.

\bibitem[Tsitsiklis et~al.(1986)Tsitsiklis, Bertsekas, and
  Athans]{tsitsiklis86}
John Tsitsiklis, Dimitri Bertsekas, and Michael Athans.
\newblock Distributed asynchronous deterministic and stochastic gradient
  optimization algorithms.
\newblock \emph{IEEE Transactions on Automatic Control}, 31\penalty0
  (9):\penalty0 803--812, September 1986.

\bibitem[Young et~al.(2017)Young, Rose, Johnston, Heller, Karnowski, Potok,
  Patton, Perdue, and Miller]{young17}
Steven~R. Young, Derek~C. Rose, Travis Johnston, William~T. Heller, Thomas~P.
  Karnowski, Thomas~E. Potok, Robert~M. Patton, Gabriel Perdue, and Jonathan
  Miller.
\newblock {Evolving Deep Networks Using HPC}.
\newblock In \emph{Proceedings of the Machine Learning on HPC Environments},
  MLHPC'17, 2017.

\end{thebibliography}
\newpage
\appendix
\onecolumn

\section{Notation}
\label{a:notation}
We use an overline to denote an average over all the nodes and a ``hat'' to denote the time-average. For example
\begin{align}
\overline \vw(k) = \frac{1}{M}\sum_{i=1}^M \vw_i(k), & &
\hat \vw_i(k) = \frac{1}{k+1}\sum_{h=0}^k \vw_i(h).
\end{align}
For a matrix, e.g.,~$\mW(k) = (\vw_1(k), \dots \vw_M(k))$, $\overline \mW(k)$ denotes the matrix whose column $i$ is $\overline \vw(k)$, i.e.,
\[\overline \mW(k) = \mW(k) \mP(\vone),\]
where $\mP(\vone)=\frac{\vone \vone^{\intercal}}{M}$ is the orthogonal projector on the subspace generated by $\vone$.
$\Delta \mW(k)$ is used to denote the difference $\mW(k)- \overline \mW(k)$.

Given a matrix $\mA$, $\mA_{i,:}$  and $\mA_{:,j}$ denote the $i$-th row and the $j$-th column, respectively. 

We use different standard matrix norms, whose definitions are reported here for completeness. Let $\mA$ be a $I \times J$ matrix:
%\begin{align}
%& \lVert \mA \rVert_{\max} = \max_{1 \le i \le I, 1\le j \le J} |A_{i,j} |, & & 
%\lVert \mA \rVert_{1} = \max_{1 \le j \le J} \sum_{i=1}^I |A_{i,j}|,\\
%& \lVert \mA \rVert_{\infty} = \max_{1 \le i \le I} \sum_{j=1}^J |A_{i,j}|,& & 
%\lVert \mA \rVert_2 = \sigma_{\max}(\mA),\\
%&\lVert \mA \rVert_F = \sqrt{ \sum_{1 \le i \le I, 1\le j \le J}  |A_{i,j} |^2}  = \sqrt{\sum_{i=1}^{\min(I,J)} \sigma_i^2(\mA)}&  & 
%\end{align}
\begin{align}
    & \lVert \mA \rVert_2 = \sigma_{\max}(\mA),\\
     & \lVert \mA \rVert_F = \sqrt{ \sum_{1 \le i \le I, 1\le j \le J}  |A_{i,j} |^2}  = \sqrt{\sum_{i=1}^{\min(I,J)} \sigma_i^2(\mA)}, 
\end{align}
where $\{\sigma_{i}(\mA)\}$ are the singular values of the matrix $\mA$ and  $\sigma_{\max}(\mA)$ is the largest one.

%The following inequalities will be used in what follows.
%\begin{equation}
%\label{e:frob_two}
%\norm*{\mA}_F \le \sqrt{\textrm{rank}(\mA)} \norm*{\mA}_2 \le \sqrt{\min(I,J)} \norm*{\mA}_2.
%\end{equation}
%It follows immediately considering the expressions of $\norm{\mA}_2$ and $\norm{\mA}_F$ in terms of the singular values of $\mA$.
%\begin{equation}
%\label{e:frob_max_col}
%\norm*{\mA}_F \le \sqrt{J} \max_{1\le j \le J} \norm*{\mA_{:,j}}_2.
%\end{equation}
%If the $\mA$ is a random matrix, it is possible to prove a similar inequality for the expectation of the Frobenius norm:
%\begin{equation}
%\label{e:frob_max_col_rand}
%\EXs{\bm\xi}{\norm*{\mA}_F} \le \sqrt{J} \max_{1\le j \le J} %\sqrt{\EXs{\bm\xi}{\norm*{\mA_{:,j}}_2^2}}.
%\end{equation}
%We provide the proof for~\eqref{e:frob_max_col_rand}.
%\begin{align*}
%\EXs{\bm\xi}{\norm*{\mA}_F}& = \EXs{\bm\xi}{\sqrt{\sum_{j=1}^J %\norm*{\mA_{:,j}}_2^2 }} \le\sqrt{\EXs{\bm\xi}{\sum_{j=1}^J %\norm*{\mA_{:,j}}_2^2 }} =\sqrt{\sum_{j=1}^J \EXs{\bm\xi}{ %\norm*{\mA_{:,j}}_2^2 }} \\
%& \le  \sqrt{J \max_{1\le j \le J} %\EXs{\bm\xi}{\norm*{\mA_{:,j}}_2^2} }= \sqrt{J}\max_{1\le j \le J} %\sqrt{\EXs{\bm\xi}{\norm*{\mA_{:,j}}_2^2}}, 
%\end{align*}
%where the first inequality follows from the concavity of the square root.

We will also consider the Frobenius inner product between matrices defined as follows
\begin{equation}
    \langle \mA, \mB \rangle_F \triangleq \sum_{i, j} A_{i,j} B_{i,j} = \Tr\left(\mA^\intercal \mB\right) \label{e:frobenius_prod}
\end{equation}

All the results in Appendix~\ref{a:convergence} assume that the matrix $\mA$ is irreducible, primitive, doubly stochastic,  non-negative, and normal.

\section{Linear algebra reminders}
\label{a:linear_algebra}

\subsection{Irreducible primitive doubly stochastic non-negative matrices}
\label{s:ipds_matrices}
We remind some results from Perron-Frobenius theory~{\cite[Ch.~8]{meyer00}}.
As our communication graph $\gG = (\V,\Ed)$ is strongly connected, and $A_{i,j}>0$ whenever  $(i,j)$ is an edge of $\gG$, the $M \times M$ consensus matrix $\mA$ is irreducible. Moreover, the consensus matrix has non-null diagonal elements and then it is also primitive. The spectral radius $\rho(\mA) \triangleq \max_i |\lambda_i|$ is then itself a simple eigenvalue.
Because $\mA$ is also stochastic, its eigenvalue $\lambda_1=1$  coincides with the spectral radius ($1\le \max |\lambda_i| \le \Vert \mA \Vert_1=1$). 

%NOT SURE WHY WE PUT THE FOLLOWING PARAGRAPH AND I THINK WE MAY NOT HAVE LISTED ALL THE HYPOTHESES (E.G. A IRREDUCIBLE AND PRIMITIVE AND PERHAPS EVEN STOCHASTIC)
%We also observe that, for any non-negative matrix $\mA$, $\Vert \mA \Vert_2=1$ if and only if the matrix is doubly stochastic. In fact, if $\mA$ is doubly stochastic, so is $\mA^\intercal \mA$ and hence $(\Vert A\Vert_2^2 \triangleq) \rho(\mA^\intercal \mA)=1$. For the opposite direction, assume that $\vone$ is not a left eigenvector, then the vector $\vone^\intercal \mA$ is not aligned with $\vone$ and it follows from Cauchy-Schwarz inequality:
%\[ M = \vone^\intercal \mA \vone < \Vert \mA^\intercal \vone \Vert_2  \times \Vert \vone \Vert_2 = \Vert \mA^\intercal \vone \Vert_2 \sqrt{M}.\]
%Hence $ \Vert \mA^\intercal \vone \Vert_2 > \sqrt M = \Vert \vone \Vert_2$, from which it follows $ \Vert \mA^\intercal \Vert_2 >1 $, contradicting the hypothesis $\Vert \mA \Vert_2 =1$. 

Let $\mP(\vone)\triangleq \frac{\vone \vone^\intercal}{M}$ be the orthogonal projector on the subspace generated by the unit vector $\vone$.

The non-zero singular values of $\mA$ are the positive square roots of the non-zero eigenvalues of $\mA^\intercal  \mA$. We observe that $(\mA - \mP(\vone))^\intercal (\mA - \mP(\vone)) = \mA^\intercal  \mA - \mP(\vone)$. The spectrum of $\mA^\intercal  \mA - \mP(\vone)$ is equal to the spectrum of $\mA^\intercal  \mA$ but for one eigenvalue $1$ that is replaced by an eigenvalue $0$. It follows that $\sigma_1(\mA - \mP(\vone)) = \sigma_2(\mA)$. 
%Moreover, it holds $\sigma_2(\mA) \le 1$:
%\[\sigma_2(\mA) = \underset{\textrm{dim} \mathcal V =n-1}{\min} \;\; \underset{\vw \in \mathcal V}{\max}\frac{\Vert \mA \vw \Vert}{\Vert \vw \Vert} 
%\le \underset{\vone^\intercal \vw =\mathbf 0}{\max}\frac{\Vert \mA \vw \Vert}{\Vert \vw \Vert} 
%= \underset{\vone^\intercal \vw=0}{\max}\frac{\Vert (\mA -\mP_1) \vw \Vert}{\Vert \vw \Vert} 
%\le \underset{\vw}{\max}\frac{\Vert (\mA -\mP_1) \vw \Vert}{\Vert \vw \Vert} = \sigma_1(\mA -\mP_1),\]
%where the first equality falls from Courant-Fisher theorem~{\cite[p.~555]{meyer00}} and the minimum is over all subspaces $\mathcal V$ with dimension $n-1$. The inequality comes from cons

\subsection{Normal matrices}
An $M \times M$ matrix $\mA$ is \emph{normal} if $\mA^\intercal \mA = \mA \mA^\intercal $. A  matrix $\mA$ is normal if and only if it is unitarily diagonalizable {\cite[p.~547]{meyer00}}, i.e.,~it exists a complete orthonormal set of eigenvectors $\vu_1, \vu_2, \dots \vu_M$ such that $\mU^\intercal \mA \mU = \mD$, where $\mD$ is the diagonal matrix containing the eigenvalues and $\mU$ has the eigenvectors as columns.

Normal matrices have a spectral decomposition with orthogonal projectors~{\cite[p.~517]{meyer00}}, i.e.,~
\[\mA = \sum_{q=1}^Q \lambda_q \mP_q,\]
where $\lambda_1, \dots \lambda_Q$ are the $Q\le M$ eigenvalues of $\mA$,  $\mP_i$ is the orthogonal projector onto the nullspace of $\mA - \lambda_i \mI$ along the range of $\mA - \lambda_i \mI$, $\mP_i \mP_j=\mathbf 0$ for $i\neq j$, and $\sum_{i=1}^M \mP_i = \mI$. Because the projectors $\mP_i$ are orthogonal and non null it holds $\mP_i^\intercal = \mP_i $ and $||\mP_i||_2=1$~{\cite[p.~433]{meyer00}}. Moreover, for any vector $\vx$ and $h\ge0$, it holds:
\begin{equation}
\left\Vert \mA^h \vx \right\Vert_{2}^2 = \sum_{q=1}^Q |\lambda_q |^{2 h}  \left\Vert \mP_q \vx \right\Vert_{2}^2.
\label{e:spectral_decomp}
\end{equation}

Symmetric matrices as well as circulant matrices are normal. In fact,  a circulant matrix is always diagonalizable by the Fourier matrix and then it is normal.

The non-zero singular values of a normal matrix $\mA$ are the modules of its eigenvalues, i.e.,~$\sigma_q(\mA)=| \lambda_q(\mA)| $. If the matrix $\mA$ is also imprimitive, irreducible, doubly stochastic and non-negative, it holds 
\begin{equation}
	\label{e:sigmas}
	\sigma_1(\mA - \mP(\vone)) = \sigma_2(\mA) = | \lambda_2 | < 1.
\end{equation}
Moreover, observe that in this case $\mP_1 =\mP(\vone)$.

\newpage 
\section{Insensitivity quantification in previous work}
\label{a:insensitivity}
\cite{lian17nips} and \cite{olshevsky19nonasymptotic} have studied the convergence of decentralized stochastic gradient method predicting topology independence after a certain number of iterations. Here, we evaluate quantitatively their predictions on some ML problems and 
%.  to estimate measure their predictions according to their respective bounds (i.e., we obtain the iterations started from where the topology plays a negligible role in their bounds) and 
compare them with our experimental observations. The results show that these predictions are very loose and then they do not fully explain why insensitivity is often observed since the beginning of the training phase.

These two papers have slightly different assumptions than ours.
\begin{description}
	\item[A$'$1] the consensus matrix $\mA$ is symmetric and doubly stochastic,  
	\item[A$'$2] every $F_j(\vw)$ has L-Lipschitz continuous gradient, i.e.,
	 \[
	 \lVert \nabla F_j(\vw) -  \nabla F_j(\vv) \rVert_2 \leq L \lVert \vw-\vv \rVert_2,\,\,\, \forall \vw,\vv \in \R^n, \forall j\in \{1,2,...,M\} ,
	 \]
	\item[A$'$3] the expected variance of stochastic gradient is uniformly bounded, i.e.,
\[ \mathbb{E}_{\xi}[ \lVert \vg_j(\vw)-\nabla F_j(\vw)\rVert_2^2 ]\leq \sigma^2,
\,\,\, \forall \vw \in \R^n, \forall j\in \{1,2,...,M\}, \]
%where $\tilde{\vg}_j(\vw) = \frac{1}{|\sS_j|}\sum_{(\vx^{(l)},y^{(l)}) \in \sS_j} \partial f(\vw, \vx^{(l)}, y^{(l)})$, 

  	\item[A$'$4] every $F_j(\vw)$ is $\mu$-strongly convex, i.e.,
\[
(\nabla F_j(\vw) - \nabla F_j(\vv))^T(\vw-\vv) \geq \mu \lVert \vw-\vv \rVert^2_2, \forall \vw,\vv \in \R^n, \forall j\in \{1,2,...,M\}.  
\]
\end{description}
We observe that A$'$1 implies A4, A$'$2 implies A5, and A$'$4 implies A1.

Next, we will present the bounds obtained in \citep{lian17nips,olshevsky19nonasymptotic} and how we evaluate them quantitatively.

\begin{cor}[Corollary~2 in \cite{lian17nips}]
Under assumptions A2-A3, A$'$1-A$'$3, and that a constant learning rate $\eta = \frac{1}{2L+\sigma \sqrt{K/M}} $ is used, the convergence rate is independent of the topology if the total number of iterations $K$ is sufficiently large, in particular if
\begin{equation}
K\geq K_l \triangleq \frac{4L^4M^5}{\sigma^2(f(0)+L)^2(1-|\lambda_2|)^2}.
\label{eq:lowerboundLian}
\end{equation}
\label{cor:Lian}
\end{cor}

%Let $K_{l}$ be the right term of (\ref{eq:lowerboundLian}), i.e., the lower bound on the iteration from which the independence of topology should appear. 
We estimate the constants $L, \sigma^2$, and $K_l$ as follows:

\begin{description}
\item[Estimate of $L$:] $\widehat{L} = \max_j \max_{(\vw,\vv) \in \Xi} \frac{\lVert \nabla F_j(\vw) -  \nabla F_j(\vv) \rVert_2}{ \lVert \vw-\vv \rVert\_2 }$, where $\Xi$ is a random set of pairs of parameter vectors. 

\item[Estimate of $\sigma^2$:]  $\widehat{\sigma^2} = \frac{\max_k \lVert \Delta G(k) \rVert_F^2}{M} = \max_k \frac{\Esp{}(k)}{M}$, as the dataset is randomly split.

\item[Estimate of $K_l$:]  $\widehat{K_{l}} =  \frac{4\widehat{L}^4M^5}{\widehat{\sigma^2}(f(0)+\widehat{L})^2(1-|\lambda_2|)^2}$.
\end{description}
We observe that $\widehat{L}$ in general underestimates the Lipschitz constant $L$, and  $\widehat{\sigma^2}$ is likely to overestimate $\sigma^2$. Both effects lead to  underestimate $K_l$, that makes our conclusions stronger.

We evaluate these quantities for three machine learning problems (SUSY,\footnote{Minimization of hinge loss function (with L2 regularization) for classification on the dataset SUSY from~\citep{ucimlr} ($S=5 \times 10^5$, $n=18$~\citep{baldi14}). In this case the function $F(.)$ is strongly convex and subgradients have bounded energy (A5 and A$'$4 hold).} CT and MNIST). The tests are done on 16 workers with batch size $B=128$ and $|\Xi|=100$. 
As we want to measure from which iteration there is no difference between undirected ring and clique topologies, we consider  in (\ref{eq:lowerboundLian}) the spectral gap ($1-|\lambda_2|$) of the ring. The results are given in Table~\ref{tab:predictions_Lian}. Corollary~\ref{cor:Lian} predicts that topology insensitivity should be observed starting from $10^6$ iterations for MNIST and even after for the other problems. On the contrary, our experiments (see Fig.~\ref{f:regr_err_vs_iter} and~\ref{f:regr_iter_vs_time}) show no significant effect of topology. 

\begin{table}[]
    \centering
    \begin{tabular}{|cccc|}
    \hline
    
    Dataset& $\widehat{L}$& $\widehat{\sigma^2}$ &  $\widehat{K_{l}}$ \\
    \hline
    SUSY&5.03 & 2.82&1.0e10   \\
    CT&  37.56&1953.27& 9.2e11 \\
    MNIST& 86.05 &12.83& 2.3e6 \\
    \hline 
    \end{tabular}
    \caption{Number of iterations after which topology-insensitivity should manifest based on \citep{lian17nips} for three ML problems.}
    \label{tab:predictions_Lian}
\end{table}

\begin{thm}[Theorem (4.2) in \citep{olshevsky19nonasymptotic}] Under assumptions A1-A3, A$'$2-A$'$4, suppose $\theta > 2$, learning rate $\eta(k) = \frac{\theta}{\mu(k+K_0)}$.  For all $k\geq K_1-K_0$ it holds
\begin{eqnarray}
\mathbb{E}[\lVert\bar{\vw}(k) -\vw^*\rVert^2]&\leq & \frac{\theta^2\sigma^2}{(1.5\theta-1)M\mu^2\tilde{k}}+\frac{3\theta^2(1.5\theta-1)\sigma^2}{(1.5\theta-2)M\mu^2 \tilde{k}^2 } + \frac{6\theta L^2 V}{(1.5\theta-2)M\mu^2 \tilde{k}^2}, \label{eq:convergence}
\end{eqnarray}
where $\vw^*$ is the minimizer of $F(\cdot)$, $K_0 = \left\lceil \frac{2\theta L^2}{\mu^2} \right\rceil$, \mbox{$K_1 = \left\lceil \frac{24L^2\theta}{(1-|\lambda_2|^2)\mu^2}\right\rceil$}, $\tilde{k} = k+K_0$,\\ \mbox{$\nabla F(\vw^* \vone^\intercal )=[\nabla F_1(\vw^*), \nabla F_2(\vw^*),\dots \nabla F_M(\vw^*)]$}, $ X = \max\{\lVert\mW(0)-\vw^* \vone^\intercal\rVert^2_F, \frac{9\sum_{i=1}^M \lVert \nabla F_i(\vw^*)\rVert^2_2 }{\mu^2} +\frac{M\sigma^2}{L^2} \}$,\\ 
$ V=\max\left\{ K_1^2 X,\frac{8\theta^2|\lambda_2|^2}{\mu^2(1-|\lambda_2|^2)} \left( \frac{4\lVert \nabla F(\vw^* \vone^\intercal)\rVert^2_F }{1-|\lambda_2|^2}+M\sigma^2 +\frac{4ML^2  Y}{(1-|\lambda_2|^2)K_1} \right)\right\}$,  
 and $  Y = \frac{K_1 X}{M}+ \frac{3}{(4\theta-3)}\left(\frac{\sigma^2\theta^2}{M\mu^2}+\frac{\sigma^2|\lambda_2|^2\theta^2}{2\mu^2}\right) + \frac{12 \lVert \nabla F(\vw^* \vone^\intercal)\rVert^2_F|\lambda_2|^2\theta^2}{(4\theta-3)M\mu^2(1-|\lambda_2|^2)} $.
\label{thm:alex}
\end{thm}

%From Theorem~\ref{thm:alex}, we derive the following looser bound:
%\begin{eqnarray}
%\mathbb{E}[\lVert\bar{\vw}(k) -\vw^*\rVert^2]&< &  \frac{\theta^2\sigma^2}{(1.5\theta-1)M\mu^2\tilde{k}}+ \frac{6\theta^2\sigma^2}{M\mu^2\tilde{k}^2}+\frac{6\theta L^2 V}{(1.5\theta-2)M\mu^2\tilde{k}^2}, \label{eq:convergence}
%\end{eqnarray}

Notice that in (\ref{eq:convergence}), only the third term is related to the topology as $ V$ depends on the spectral gap of the consensus matrix. 
The bound in Corollary~\ref{cor:Lian} is obtained imposing that the term depending on the spectral gap is smaller than the other terms \citep[Thm.~1]{lian17nips}.  
Here we apply the same idea, requiring the third term of (\ref{eq:convergence}) to be bounded by the first and the second term of (\ref{eq:convergence}), i.e.,
\[\frac{6\theta L^2 V}{(1.5\theta-2)M\mu^2\tilde{k}^2} \leq \frac{\theta^2\sigma^2}{(1.5\theta-1)M\mu^2\tilde{k}}+ \frac{3\theta^2(1.5\theta-1)\sigma^2}{(1.5\theta-2)M\mu^2\tilde{k}^2}, \] 
and then 
\[\frac{6 L^2 V}{(1.5\theta-2) \tilde{k}} \leq \frac{\theta\sigma^2}{(1.5\theta-1)}+ \frac{3\theta(1.5\theta-1)\sigma^2}{(1.5\theta-2)\tilde{k}}. \] 
Then, we have
\[
\tilde{k} \geq \left(\frac{6L^2 V}{\theta\sigma^2}-4.5\theta+3 \right)\frac{1.5\theta-1}{1.5\theta-2}.
\]

As $ V\geq \left(\frac{24L^2\theta}{(1-|\lambda_2|^2)\mu^2}\right)^2\frac{M\sigma^2}{L^2}$ and $\theta>2$, assuming that $\frac{3456 ML^4}{\mu^4 (1-|\lambda_2|^2)^2}-\frac{2L^2}{\mu^2}- 4.5>0$, we conclude that the convergence rate is independent of the topology if the total number of iterations $k$ is sufficiently large, and more precisely if
%\begin{eqnarray}
%    k &\geq&\left[\left(\frac{3456 M L^4}{\mu^4(1-|\lambda_2|^2)^2}-4.5\right)\theta+3 \right]\frac{1.5\theta-1}{1.5\theta-2} - K_0 \nonumber \\
%    &\geq& \frac{6912 ML^4}{\mu^4(1-|\lambda_2|^2)^2} -6 - K_0.
%    \label{eq:alex}
%\end{eqnarray}

\begin{eqnarray}
    k &\geq&\left[\left(\frac{3456 M L^4}{\mu^4(1-|\lambda_2|^2)^2}-4.5\right)\theta+3 \right]\frac{1.5\theta-1}{1.5\theta-2} - K_0 \nonumber \\
    &>& \left[\left(\frac{3456 M L^4}{\mu^4(1-|\lambda_2|^2)^2}-4.5\right)\theta+3 \right] - \frac{2\theta L^2}{\mu^2} -1 \nonumber \\
    &=&\left(\frac{3456 M L^4}{\mu^4(1-|\lambda_2|^2)^2}-\frac{2L^2}{\mu^2}-4.5\right)\theta + 2 \nonumber  \\
    &>& \frac{6912 ML^4}{\mu^4(1-|\lambda_2|^2)^2} - \frac{4L^2}{\mu^2} - 7.\nonumber 
\end{eqnarray}

We denote the last expression in the sequence of inequalities as $K'_l$, i.e.,
\begin{equation}
    K'_l \triangleq \frac{6912 ML^4}{\mu^4(1-|\lambda_2|^2)^2} - \frac{4L^2}{\mu^2} - 7.\label{eq:alex}
\end{equation}

%Let $K'_l$ be the right term of (\ref{eq:alex}) which gives the lower bound for the appearance of independence of topology. 
Here we evaluate $K'_l$ for SUSY dataset, as its loss function is strongly convex fulfilling the assumptions of Thm.~\ref{thm:alex}. Remind that the loss function of SUSY is the hinge loss with L2 regularization $\frac{\mu}{2}\lVert \vw \rVert^2$, which is $\mu$-strongly convex. 

The quantity $L$ in (\ref{eq:alex}) is estimated as above. 
We derive estimates for  $K'_l$, $K_0$, $K_1$ and the rule to set the learning rate $\eta(k)$, by simply replacing $\widehat L$ in their analytical expressions. 

\begin{table}[]
    \centering
        \begin{tabular}{|c|ccccc|}
        \hline
        $\mu$& $\widehat{L}$  & $\widehat{K_0}$ & $\widehat{K_1}$ & $\eta(0)$& $\widehat{K'_l}$\\
        \hline 
        0.01 & 5.03& 2.5e6& 3.1e8& 2e-4& 7.4e17\\
        1 & 5.03& 254& 31140& 0.02& 7.4e9\\
        \hline
        \end{tabular}
    \caption{Number of iterations after which topology-insensitivity should manifest based on \citep{olshevsky19nonasymptotic} for SUSY dataset.}
    \label{tab:predictions_alex}
\end{table}

\begin{figure}
    \centering
    \includegraphics[scale=0.5]{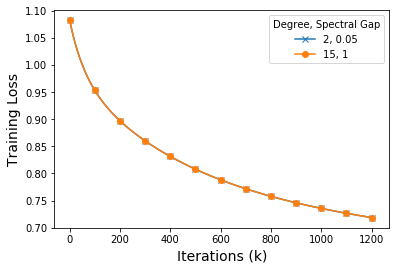}
    \caption{SUSY, M = 16, B = 128.}
    \label{fig:susy}
\end{figure}

In our experiments, we consider $M = 16$, $B = 128$. As the value of $K'_l$ is highly sensitive to $\mu$, we consider both $\mu = 0.01$ and $\mu=1$. 
The results are shown in Table~\ref{tab:predictions_alex}.
$\widehat{K_0}$, $\widehat{K_1}$ and $\eta(0)$ are measured for the case where $\theta = 5$.  

The values predicted for $K'_l$ are larger than 1e9 iterations. But again we observe no effect of topology in our experiments. For example, Fig.~\ref{fig:susy} shows the training loss over the number of iterations for $\mu=1$.
%for $$However, upon our observations in the experiment, when setting the learning rate schedule to $\frac{\theta}{\mu(k+K)}$ where $\mu = 1$, $\theta = 5$ and $K=\lceil \frac{2\theta L^2}{\mu^2} \rceil$, the independence of topology appears already in the early stage of training as shown in Fig.~\ref{fig:susy}. 

%, and the other parameters $$. Substituting the measured $\widehat{L}$ and $\mu$ into the (\ref{eq:alex}) where $\theta$ is set to 5 and $1-|\lambda_2|$ is the spectral gap of the ring topology, we obtain the estimates on the intermediate parameters $\widehat{K}$, $\widehat{K_1}$ from Thm.~\ref{thm:alex}, the initial learning rate $\eta(0)$ and $K'_l$. These results are shown in Table~\ref{tab:predictions_alex}.

\newpage
\section{Convergence results}
\label{a:convergence}

All the following results assume that the matrix $\mA$ is irreducible, primitive, doubly stochastic,  non-negative, and normal.

\begin{lem}
The following inequality holds. 
\begin{align}
&\EXs{\bm\xi}{||  \Delta\mW(k+1) ||_{F}} \le \sqrt{M}  \norm*{\Delta \mW(0)}_F |\lambda_2|^{k+1}  + \sum_{h=0}^k \eta(h) \sqrt{\sum_{l=2}^Q |\lambda_l|^{2(k-h)} \EXs{\bm\xi}{\norm*{\Delta\mG(h)  \mP_l}_F^2}} \label{e:state_ineq_proj}.
\end{align}

\end{lem}

\begin{proof}
	\begin{align}
    \lVert\Delta\mW(k+1) \rVert_{F} & = \left\Vert\mW(0) \left( \mA^{k+1} - \mP(\vone) \right) -   \sum_{h=0}^k  \eta(h) \mG(h) \left( \mA^{k-h} - \mP(\vone)  \right)\right\Vert_{F}\label{e:difference}\\
	&   = \left\Vert \Delta \mW(0) \left( \mA^{k+1} - \mP(\vone) \right) -   \sum_{h=0}^k  \eta(h) \Delta \mG(h) \left( \mA^{k-h} - \mP(\vone)  \right)\right\Vert_{F}\label{e:remove_avg}\\
	& 	\le \left\Vert \Delta\mW(0) \left( \mA^{k+1} - \mP(\vone) \right)\right\Vert_{F} +  \sum_{h=0}^k \eta(h)\left\Vert \Delta\mG(h)  \left( \mA^{k-h} - \mP(\vone)  \right)\right\Vert_{F}\\
	& 	=  \left\Vert \Delta \mW(0) \left( \mA - \mP(\vone) \right)^{k+1}\right\Vert_{F} +  \sum_{h=0}^k \eta(h)\left\Vert \Delta\mG(h)  \left( \mA - \mP(\vone)  \right)^{k-h}\right\Vert_{F}\label{e:state_ineq_common}\\
	& 	=  \left\Vert \Delta\mW(0) (\Delta \mA)^{k+1}\right\Vert_{F} +  \sum_{h=0}^k \eta(h)\left\Vert \Delta\mG(h) (\Delta \mA)^{k-h}\right\Vert_{F}.
	\end{align}
The first equality~\eqref{e:difference} follows from \eqref{e:from_start}, $\overline \mW(k+1) = \mW(k+1) \mP(\vone)$, and  $\mA \mP(\vone) = \mP(\vone)$, because $\mA$ is row stochastic. \Eqref{e:remove_avg} follows from the fact that, for any matrix $\mB$, $\overline \mB = \mB \mP(\vone)$ and then $\overline \mB (\mA^h - \mP(\vone)) = \mB \mP(\vone) \mA^h - \mB \mP(\vone)^2 = \mB \mP(\vone) - \mB \mP(\vone) = \mathbf 0$, because $\mA$
 is column stochastic and $\mP(\vone)$ is a projector.
For \eqref{e:state_ineq_common} expand $(\mA - \mP(\vone))^h$ taking into account again that  $\mA$ is row stochastic.

Let us now bound separately the two terms on the right hand side of inequality~\eqref{e:state_ineq_common}. For the first one it holds
\begin{align}
\left\Vert \Delta \mW(0) (\Delta \mA)^{k+1}\right\Vert_{F} & \le \norm*{\Delta \mW(0)}_F \norm*{(\Delta \mA)^{k+1}}_F \nonumber\\
	& = \norm*{\Delta \mW(0)}_F \sqrt{\sum_{l=2}^M |\lambda_l|^{2 (k+1)}} \nonumber \\
	&\le \sqrt{M}  \norm*{\Delta \mW(0)}_F |\lambda_2|^{k+1} \label{e:bound_first_term},
\end{align}
where the first inequality follows from the sub-multiplicative property of Frobenius norm.
For the second term on the right hand side of~\eqref{e:state_ineq_common}, we carry out a more careful analysis.
\begin{align}
\norm*{\Delta\mG(h) (\Delta \mA)^{k-h}}_{F}^2 & = \sum_{i=1}^n  \norm*{\Delta\mG_{i,:}(h) (\Delta \mA)^{k-h}}_2^2 \nonumber \\
& = \sum_{i=1}^n  \norm*{\Delta\mG_{i,:}(h) \sum_{l=2}^Q |\lambda_l|^{k-h} \mP_l}_2^2\nonumber\\
			& =  \sum_{i=1}^n  \sum_{l=2}^Q |\lambda_l|^{2(k-h)} \norm*{\Delta\mG_{i,:}(h)  \mP_l}_2^2\nonumber\\ 
			& =   \sum_{l=2}^Q |\lambda_l|^{2(k-h)} \sum_{i=1}^n \norm*{\Delta\mG_{i,:}(h)  \mP_l}_2^2 \nonumber \\
			& =   \sum_{l=2}^Q |\lambda_l|^{2(k-h)} \norm*{\Delta\mG(h)  \mP_l}_F^2.\label{e:bound_second_term}
\end{align}
From \eqref{e:state_ineq_common}, \eqref{e:bound_first_term}, \eqref{e:bound_second_term}, and Jensen's inequality, it follows
\begin{align*}
\EXs{\bm\xi}{||  \Delta\mW(k+1) ||_{F}} \le \sqrt{M}  \norm*{\Delta \mW(0)}_F |\lambda_2|^{k+1}  +\sum_{h=0}^k \eta(h)\sqrt{\sum_{l=2}^Q |\lambda_l|^{2(k-h)} \EXs{\bm\xi}{\norm*{\Delta\mG(h)  \mP_l}_F^2}}.
\end{align*}
\end{proof}

Let $R$ be a bound on the energy of the initial parameter vector across the different nodes, i.e.,~$\left \lVert \mW(0) \right \rVert_F^2 \le R$, and $\Rsp$ be the corresponding bound for the spread of the parameter vectors around their averages, i.e.,~$\left \lVert \mW(0) - \overline \mW(0) \right \rVert_F^2 \le \Rsp$.
Similarly, we define $E$ and $\Esp{}$ as bounds for the subgradient matrix $\Delta G$ for any time $h$:
$\sup_{h \ge0} \EXs{\bm\xi}{\left \lVert \mG(h) \right \rVert_F^2} \le E$, $\sup_{h \ge0} \EXs{\bm\xi}{\left \lVert \Delta \mG(h) \right \rVert_F^2} \le \Esp{}$. 
We observe that $\Rsp \le R$ and $ \Esp{} \le E$.

Moreover, for a normal matrix $\mA$, we define for $l=1, \dots M$: 
\[  \Esp{l} \triangleq \sup_{h \ge0}   \EXs{\bm\xi}{\norm*{\Delta \mG_{i,:}(h)  \sum_{l'=2}^l \mP_{l'}}_F^2}, \]
with the usual convention that $\sum_i^j \cdot = 0$ if $j<1$ and then $\Esp{1}=0$.
We observe that $\Esp{l}$ represents the maximum expected energy $\Delta \mG(h)$ 
in the projection subspace defined by the first $l$ projectors. In particular it holds
\begin{align}
	\Esp{M} 	&	= \sup_{h \ge0}  \left \lVert \Delta\mG(h) \sum_{l'=2}^M \mP_{l'} \right \rVert_F^2 
			= \sup_{h \ge0}   \left \lVert \Delta\mG(h) (\mI - \mP_1) \right \rVert_F^2  \\
		& =  \sup_{h \ge0}   \left \lVert \Delta\mG(h) \right \rVert_F^2 \le \Esp{}.
\end{align}
Let us now consider the normalized fraction of energy in each subspace, defined as follows:
\begin{align}
	e_l & \triangleq \frac{\Esp{l} - \Esp{l-1}}{\Esp{M}},
\end{align}
so that $\sum_l e_l = 1$.
Finally, let 
\begin{align}
	\alpha(h) &\triangleq  \sqrt{\sum_{l=2}^M e_l \left| \frac{\lambda_l}{\lambda_2} \right|^{2 h}},
\end{align}
and we denote $\alpha(1)$ simply as $\alpha$.
We observe that $| \lambda_l/ \lambda_2| \le 1$, then $\alpha(h)$ is decreasing in $h$. Moreover, $e_2 \le 1$ as
\[\sqrt{e_2}= \sqrt{ e_2 \left| \frac{\lambda_2}{\lambda_2} \right|^{2 h}} \le  \sqrt{\sum_{l=2}^M e_l \left| \frac{\lambda_l}{\lambda_2} \right|^{2 h}} \le \sqrt{\sum_{l=2}^M e_l }=1. \]
 $\alpha(h)$ can be considered a bound for the effective energy contribution of the vector $\Delta\mG(h) $ in the projection subspace defined by $\mP_2$.

%Moreover, let $\Lpr$ denote a bound for the projection of the spread of the subgradients, i.e.,~$\sup_{h \ge0} \max_l \max_i \left \lVert \left( \mG_{i,:}(h) - \overline \mG_{i,:}(h) \right) \mP_l \right \rVert_2 \le \Lpr$.

\begin{cor} 
\label{c:max_dist}
Considering the definition of $\Rsp$, $\Esp{}$, and $\alpha(l)$, the following inequality holds for a constant learning rate $\eta$: 
\begin{align}
& ||  \Delta\mW(k)||_{F}  \le  \sqrt{M} \sqrt{\Rsp} |\lambda_2|^{k} +  \eta  \sqrt{\Esp{}} \left( (1- \alpha) \mathbbm 1_{k \ge 1} + \alpha \frac{1-| \lambda_2 | ^{k}}{1-| \lambda_2 |  }\right).\label{e:state_ineq_proj2}
\end{align}
\end{cor}
\begin{proof}
	The first term on the right hand side bounds $\sqrt{M} \norm*{\mW(0)}_F$ by definition of $\Rsp$.
	% and the fact that the learning rate $\eta(h)$ is non-decreasing. 
	%The bound in (\ref{e:state_ineq_spread2}) is tighter because $\Rsp \le R$ and $\Lsp \le L$.

Observe that
\begin{align}
	 \sqrt{ \sum_{l=2}^M  | \lambda_l| ^{2(k-h)} \EXs{\bm\xi}{\left  \lVert \Delta\mG(h) \mP_l \right \rVert_{F}^2 }}
	& \le  |\lambda_2|^{k-h} \sqrt{\Esp{}} \sqrt{ \sum_{l=2}^M \frac{\EXs{\bm\xi}{\left  \lVert \Delta\mG(h) \mP_l \right \rVert_{F}^2}}{\Esp{}}  \left| \frac{\lambda_l}
{\lambda_2} \right|^{2 (k-h)}}\\
	& \le |\lambda_2|^{k-h} \sqrt{\Esp{}} \sqrt{\sum_{l=2}^M e_l   \left| \frac{\lambda_l}
{\lambda_2} \right|^{2 (k-h)}} \\
&\le |\lambda_2|^{k-h} \sqrt{\Esp{}} \alpha(k-h).
\end{align}
Using this bound, we obtain
\begin{align}
	\sum_{h=0}^{k}   \sqrt{ \sum_{l=2}^M  | \lambda_l| ^{2(k-h)} \EXs{\bm\xi}{\left  \lVert \Delta\mG(h) \mP_l \right \rVert_{2}^2 }}  
	&\le \sqrt{\Esp{}} \sum_{h=0}^{k} \alpha(h)    |\lambda_2|^{h}\\
	& = \sqrt{\Esp{}} \left(1+ \sum_{h=1}^{k} \alpha(h)    |\lambda_2|^{h}\right)\\
   & \le \sqrt{\Esp{}} \left(1+ \alpha(1) \sum_{h=1}^{k}     |\lambda_2|^{h}\right) \\
   & = \sqrt{\Esp{}} \left(1+ \alpha \sum_{h=1}^{k}     |\lambda_2|^{h}\right) \\
	& = \sqrt{\Esp{}} \left( (1-\alpha) + \alpha \sum_{h=0}^{k}     |\lambda_2|^{h}\right)\\
	& = \sqrt{\Esp{}} \left( (1-\alpha) + \alpha \frac{1 - |\lambda_2|^{k+1}}{1- |\lambda_2|}\right).
\end{align}
From this last bound and \eqref{e:state_ineq_proj}, an inequality similar to \eqref{e:state_ineq_proj2}, but without the indicator function, follows immediately. The indicator function can be introduced because, for $k=0$, it is simply $||  \Delta\mW(k)||_{F}  \le  \sqrt{\Rsp} $.

%	\begin{align*}
%	\sqrt{ \sum_{l=2}^s  | \lambda_l| ^{2(k-h)} 

\end{proof}

\begin{lem}
\label{l:distance}
Let $\sW^*$  the (non-empty)  optimal solution set. It holds:
\begin{align}
&\EXs{\bm\xi}{{\dist {\overline \vw(k+1)} {\sW^*}}^2} \le  \EXs{\bm\xi}{{\dist {\overline \vw(k)} {\sW^*}}^2}  + \frac{\eta^2 E}{M}   + \frac{4 \eta H}{M}  \norm*{\Delta\mW(k)}_F  - \frac{2 \eta}{M} \left( \EXs{\bm\xi}{F(\overline \vw(k))} - F^*\right).
\end{align}
where  $\dist \vx \sX$ denotes the distance between a vector $\vx$ and the set $\sX$.
\end{lem}
\begin{proof}
The proof follows closely the proof in~{\cite[Lemma~5]{nedic09}}, replacing the Euclidean norm with the Frobenius norm used in this paper.

We denote the subgradients of $F_j$ in $\vw_j(k)$ and $\overline \vw(k)$ simply as   $\vg_j(k)$ and $\tilde \vg_j(k)$, respectively, i.e.,~$\vg_j(k)=\vg_j(\vw_j(k))$ and $\tilde \vg_j(k) = \vg_j(\overline \vw(k))$. Let $\mG(k)$ and $\tilde \mG(k)$ be respectively the matrices whose columns are $\vg_j(k)$ and $\tilde \vg_j(k)$.
Let $\vx$ be a generic vector in $\mathbb R^n$.
\begin{align}
   \norm*{\overline \vw(k+1) - \vx}_2^2  & = \norm*{\overline \vw(k) - \vx - \frac{\eta}{M} \sum_{j=1}^M \vg_j(\vw_j(k))}_2^2\nonumber\\
            & = \norm*{\overline \vw(k) - \vx}_2^2 + \frac{\eta^2}{M^2} \norm*{ \sum_{j=1}^M \vg_j(k)}_2^2 - \frac{2 \eta}{M} \sum_{j=1}^M \vg_j(k)^\intercal (\overline \vw(k) - \vx)\nonumber\\
            & \leq \norm*{\overline \vw(k) - \vx}_2^2 + \frac{\eta^2}{M} \norm*{ \mG(k)}_F^2 - \frac{2 \eta}{M} \sum_{j=1}^M \vg_j(k)^\intercal (\overline \vw(k) - \vx).\label{e:quadratic}
\end{align}
Let us bound the scalar product $\vg_j(k)^\intercal (\overline \vw(k)-\vx)$:
\begin{align}
    \vg_j(k)^\intercal  (\overline \vw(k) - \vx) & = \vg_j(k)^\intercal (\overline \vw(k) - \vw_j(k)) + \vg_j(k)^\intercal (\vw_j(k)- \vx) \nonumber\\
    & \ge \vg_j(k)^\intercal (\overline \vw(k) - \vw_j(k)) + F_j(\vw_j(k)) - F_j(\vx) \label{e:subgrad1}\\
    & = \vg_j(k)^\intercal (\overline \vw(k) - \vw_j(k)) + F_j(\vw_j(k)) - F_j(\overline \vw(k))     + F_j(\overline \vw(k)) - F_j(\vx)\nonumber\\
    & \ge \vg_j(k)^\intercal (\overline \vw(k) - \vw_j(k)) + \tilde \vg_j(k)^\intercal (\vw_j(k)-\overline \vw(k))  + F_j(\overline \vw(k)) - F_j(\vx)\label{e:subgrad2},
\end{align}
where \eqref{e:subgrad1} follows from $\vg_j(k)$ being a subgradient of $F_j$ in $\vw_j(k)$ and \eqref{e:subgrad2}  from $\tilde \vg_j(k)$ being a subgradient of $F_j$ in $\overline \vw(k)$.
Summing over $j$ the LHS and RHS of the above inequality, we obtain:
\begin{alignat}{2}
    \sum_{j=1}^M  \vg_j(k)^\intercal  (\overline \vw(k) - \vx) & \ge  \sum_{j=1}^M  \vg_j(k)^\intercal (\overline \vw(k) - \vw_j(k))  + \sum_{j=1}^M \tilde \vg_j(k)^\intercal (\vw_j(k)-\overline \vw(k)) + \sum_{j=1}^M F_j(\overline \vw(k)) -  \sum_{j=1}^M F_j(\vx) \nonumber \\
    & = \mathrlap{ - \langle \mG(k), \Delta\mW(k) \rangle_F + \langle \tilde \mG(k), \Delta\mW(k) \rangle_F + F(\overline \vw(k)) - F(\vx),} \nonumber 
\end{alignat}
where we have used the definition of the Frobenius inner product~\eqref{e:frobenius_prod}.
By computing the expected value we obtain
\begin{alignat}{2}
    \EXs{\bm\xi}{\sum_{j=1}^M \vg_j(k)^\intercal  (\overline \vw(k) - \vx)} &\ge \left( \langle \EXs{\bm\xi}{\mG(k)}, \Delta\mW(k) \rangle_F - \langle \EXs{\bm\xi}{\tilde \mG(k)}, \Delta\mW(k) \rangle_F\right) + \EXs{\bm\xi}{F(\overline \vw(k)) - F(\vx)}\nonumber\\
    & \mathrlap{\ge - \left( \norm*{\EXs{\bm\xi}{\mG(k)}}_F + \norm*{\EXs{\bm\xi}{\tilde\mG(k)}}_F \right) \norm*{\Delta\mW(k)}_F + \EXs{\bm\xi}{F(\overline \vw(k)) - F(\vx)}}\nonumber\\
    %& \mathrlap{\ge - \left( \EXs{\bm\xi}{\norm*{\mG(k)}_F} + \EXs{\bm\xi}{\norm*{\tilde\mG(k)}_F} \right) \norm*{\Delta\mW(k)}_F + \EXs{\bm\xi}{F(\overline \vw(k)) - F(\vx)}}\nonumber\\
    %& \mathrlap{\ge - \left( \sqrt{\EXs{\bm\xi}{\norm*{\mG(k)}_F^2}} + \sqrt{\EXs{\bm\xi}{\norm*{\tilde\mG(k)}_F^2}} \right) \norm*{\Delta\mW(k)}_F + \EXs{\bm\xi}{F(\overline \vw(k)) - F(\vx)}}\nonumber\\
    & \mathrlap{\ge - 2 H\norm*{\Delta\mW(k)}_F + \EXs{\bm\xi}{F(\overline \vw(k)) - F(\vx)}}.\label{e:exp_scalar}
\end{alignat}

From~\eqref{e:quadratic} and~\eqref{e:exp_scalar}, we obtain:
\begin{align}
    \EXs{\bm\xi}{\norm*{\overline \vw(k+1) - \vx}_2^2} \le &  \EXs{\bm\xi}{\norm*{\overline \vw(k) - \vx}_2^2} + \frac{\eta^2 E}{M}
         + \frac{4 \eta H}{M}  \norm*{\Delta\mW(k)}_F  - \frac{2 \eta}{M} \EXs{\bm\xi}{F(\overline \vw(k)) - F(\vx)}.
\end{align}
Then the thesis follows from considering $\vx$ a generic point in the optimal set $\sW^*$.
\end{proof}

\subsection{Proof of Proposition~\ref{p:new_bound}}
\label{a:proof_new_bound}
\begin{proof}
We start computing a bound for the time average of $\norm*{\Delta \mW(k)}_F$ using Corollary~\ref{c:max_dist}:
\begin{align}
\frac{1}{K}\sum_{k=0}^{K-1} \norm*{\Delta \mW(k)}_F \le &
    \frac{\sqrt{M}\sqrt{\Rsp}}{K} \frac{1 - |\lambda_2|^{K}}{1- |\lambda_2|}
    +  \eta \sqrt{\Esp{}} (1-\alpha) \frac{K-1}{K}\nonumber\\
    &+ \frac{ \eta \sqrt{\Esp{}}\alpha}{1- |\lambda_2|} \left( 1- \frac{1}{K} \frac{1 - |\lambda_2|^{K}}{1- |\lambda_2|} \right)\label{e:avg_frob_bound}
\end{align}

%From Lemma~\ref{l:distance}, the norm inequality $\lVert \vx \rVert_2 \le n^{1/2} \lVert \vx  \rVert_{\max} $ for any $\vx \in \mathbb R^n$, and \eqref{e:state_ineq_classic2} from Corollary~\ref{c:max_dist}, it follows:
%\begin{align}
%	F(\overline \vw(k)) & - F^*	 \le \frac{M}{2 \eta} \left({\dist{\overline \vw(k)} {\sW^*}}^2 - {\dist{\overline \vw(k+1)} {\sW^*}}^2 \right)   + \frac{\eta L^2}{2 }
%	+ 2 L \sum_{i=1}^M \lVert \overline \vw(k) - \vw_i(k) \rVert_2 \\
%					& \le \frac{M}{2 \eta} \left({\dist{\overline \vw(k)} {\sW^*}}^2 - {\dist{\overline \vw(k+1)} {\sW^*}}^2 \right) + \frac{\eta L^2}{2 }
%					+ 2 L M n^{1/2} \lVert \mW(k) - \overline \mW(k)  \rVert_{\max} \\
%					& \le 	\frac{M}{2 \eta}  \left( {\dist{\overline \vw(k)} {\sW^*}}^2 - {\dist{\overline \vw(k+1)} {\sW^*}}^2 \right) + \frac{\eta L^2}{2 }\\
%					& \phantom{\le} + 2 L M n^{1/2} \left(\Rsp |\lambda_2|^{k}  +  \eta \sqrt{\Esp{}} (1-\alpha(1))\mathbbm 1_{k \ge 1}+  \eta \sqrt{\Esp{}} \alpha(1) \frac{1-|\lambda_2|^{k}}{1-|\lambda_2|}\right).
%\end{align}
If we take into account convexity of $F$, Lemma~\ref{l:distance}, and \eqref{e:avg_frob_bound}, we obtain:
\begin{align}
\EXs{\bm\xi}{F\!\left(\hat{ \overline \vw}(k) \right)}- F^* 
        & = \EXs{\bm\xi}{F\!\left(\frac{1}{K} \sum_{k=0}^{K-1} \overline \vw(k) \right)}- F^* \nonumber\\
        & \underset{\textrm{convexity}}{\le} \frac{1}{K} \sum_{k=0}^{K-1} \left(\EXs{\bm\xi}{ F(\overline \vw(k))}  - F^* \right)\nonumber\\
    & \underset{\textrm{Lem~\ref{l:distance}}}{\le} \frac{1}{K} \sum_{k=0}^{K-1}\left( \frac{M}{2 \eta}\left({\dist{\overline \vw(k)} {\sW^*}}^2 -{\dist{\overline \vw(k+1)} {\sW^*}}^2 \right)  + \frac{\eta E}{2} + 2H \norm*{\Delta \mW(k)}_F\right) \nonumber\\
     & =\frac{M}{2 \eta K}  \left( {\dist{\overline \vw(0)} {\sW^*}}^2 - {\dist{\overline \vw(K)} {\sW^*}}^2 \right)  + \frac{\eta E}{2 } + 2H \frac{1}{K}\sum_{k=0}^{K-1} \norm*{\Delta \mW(k)}_F\nonumber\\
     & \le\frac{M}{2 \eta K}   {\dist{\overline \vw(0)} {\sW^*}}^2   + \frac{\eta E}{2 } + 2H \frac{1}{K}\sum_{k=0}^{K-1} \norm*{\Delta \mW(k)}_F\nonumber\\
     & \underset{\textrm{\eqref{e:avg_frob_bound}}}{\le} \frac{M}{2 \eta K} {\dist{\overline \vw(0)} {\sW^*}}^2 + \frac{\eta E}{2 } +2H \frac{\sqrt{M}\sqrt{\Rsp}}{K} \frac{1 - |\lambda_2|^{K}}{1- |\lambda_2|}  \nonumber\\
     & \phantom{\underset{\textrm{\eqref{e:avg_frob_bound}}}{\le}}+ 2H \eta \sqrt{\Esp{}}\left(     
       (1-\alpha) \frac{K-1}{K}
    + \frac{ \alpha}{1- |\lambda_2|} \left( 1- \frac{1}{K} \frac{1 - |\lambda_2|^{K}}{1- |\lambda_2|} \right)
     \right)\label{e:bound_double_avg}
\end{align}
%If we select the learning rate $\eta=\eta_0/\sqrt{K}$, we obtain:
%\begin{align}
%F&\!\left(\frac{1}{K} \sum_{k=0}^{K-1} \overline \vw(k) \right)- F^* 
%  \le \frac{M}{2 \eta_0 \sqrt K} {\dist{\overline \vw(0)} {\sW^*}}^2 
%		 + \frac{\eta_0 L^2}{2 \sqrt K} \\
%		&\phantom{\le} + 2 L M n^{1/2} \left(  \frac{\Rsp}{K} \frac{1 - |\lambda_2|^{K}}{1- |\lambda_2|}  +  \frac{\eta_0}{\sqrt K} \Lsp (1-\alpha(1)) \mathbf 1_{K>1} + \frac{\eta_0}{\sqrt K} \frac{  \Lsp \alpha(1)}{1- |\lambda_2|} \left( 1- \frac{1}{K} \frac{1 - |\lambda_2|^{K}}{1- |\lambda_2|} \right) \right).
%\end{align}
\end{proof}

\subsection{Proof of Corollary~\ref{p:classic_bound}}
\label{a:proof_classic_bound}

\begin{proof}
Because of the relations $\Rsp \le R$ and $\Esp{} \le E$, the only step to prove is that 
\begin{align}
\label{e:abc}
	(1-\alpha) \frac{K-1}{K} + \alpha \frac{1}{1- |\lambda_2|} \left( 1- \frac{1}{K} \frac{1 - |\lambda_2|^{K}}{1- |\lambda_2|} \right) \le \frac{1}{1- |\lambda_2|} \left( 1- \frac{1}{K} \frac{1 - |\lambda_2|^{K}}{1- |\lambda_2|} \right).
\end{align}
The two sides can be rewritten as the sums indicated below:
\begin{align}
	\frac{1}{K} \sum_{k=0}^{K-1} \left( (1-\alpha) \mathbbm{1}_{k \ge 1}  + \alpha  \frac{1 - |\lambda_2|^{k}}{1- |\lambda_2|} \right) \le   \frac{1}{K} \sum_{k=0}^{K-1}   \frac{1 - |\lambda_2|^{k}}{1- |\lambda_2|}.
\end{align}
%because the sums on the two sides are exactly equal to the corresponding sums in \eqref{e:abc}.

It is then sufficient to prove that for each $k\ge 0$
\begin{align}
	 (1-\alpha) \mathbbm{1}_{k \ge 1}  + \alpha  \frac{1 - |\lambda_2|^{k}}{1- |\lambda_2|}  \le     \frac{1 - |\lambda_2|^{k}}{1- |\lambda_2|}.
\end{align}
This relation is obviously satisfied for $k=0$ ($\alpha<1$). For any $k>0$, it follows from 
\begin{align}
\frac{1 - |\lambda_2|^{k}}{1- |\lambda_2|} = \sum_{h=0}^{k-1} |\lambda_2|^{h} \ge 1.
\end{align}
\end{proof}

\subsection{Convergence of local estimates}
\label{a:conv_local_estimate}
\begin{prop}
\label{p:new_bound_local}
Under assumptions A1-A5 and that~a constant learning rate $\eta(k)=\eta$ is used, an upper bound for the objective value at the end of the \mbox{($K-1$)th} iteration is given, for each $i$, by:
\begin{align}
\EX{F\!\left(\hat    \vw_i(K-1) \right)}- F^*  \le &   \frac{M}{2 \eta {K}} {\dist{\overline \vw(0)} {\sW^*}}^2 + \frac{\eta E}{2}  + H \frac{3 M \sqrt{\Rsp}}{K} \frac{1 - |\lambda_2|^{K}}{1- |\lambda_2|}  \nonumber\\
& + 3 \eta \sqrt{M}H \sqrt{\Esp{}} \left( (1-\alpha) \frac{K-1}{K} + \frac{ \alpha}{1- |\lambda_2|} \left( 1- \frac{1}{K} \frac{1 - |\lambda_2|^{K}}{1- |\lambda_2|} \right)\right).
\label{e:new_bound_local}	 
\end{align}
\end{prop}
\begin{proof}
We consider the local model at node $i$ ($\vw_i(k)$). Using convexity of the local functions and the definition of subgradients, we obtain:
\begin{align}
\EXs{\bm\xi}{F\!\left(\hat{  \vw}_i(k) \right)} &= \sum_{j=1}^M  \EXs{\bm\xi}{F_j\!\left(\hat{  \vw}_i(k)\right)} 
    \underset{\textrm{convexity}}{\le} \sum_{j=1}^M \left( \EXs{\bm\xi}{F_j\!\left(\hat{\overline  \vw}(k)\right)} + \EXs{\bm\xi}{\vg_j\!(\hat{  \vw}_i(k))}^\intercal (\hat{  \vw}_i(k) - \hat{\overline \vw}(k))\right)\nonumber\\
    &= \EXs{\bm\xi}{F\!\left(\hat{\overline \vw}(k) \right)} + \sum_{j=1}^M  \EXs{\bm\xi}{\vg_j\!(\hat{  \vw}_i(k))}^\intercal (\hat{  \vw}_i(k) - \hat{\overline \vw}(k)) \nonumber\\
	&\le  \EXs{\bm\xi}{F\!\left(\hat{\overline \vw}(k) \right)} + \sum_{j=1}^M  \norm*{\EXs{\bm\xi}{\vg_j\!(\hat{  \vw}_i(k))}}_2 \lVert \hat{  \vw}_i(k) - \hat{\overline \vw}(k)\rVert_2\nonumber\\
	&\le \EXs{\bm\xi}{F\!\left(\hat{\overline \vw}(k) \right)} + \sqrt{M} \norm*{\EXs{\bm\xi}{G(\hat{\vw}_i(k))}}_F \norm*{\hat{\vw}_i(k) - \hat{\overline \vw}(k)}_2 \nonumber\\
	&\le \EXs{\bm\xi}{F\!\left(\hat{\overline \vw}(k) \right)} + \sqrt{M} H \norm*{\hat{\vw}_i(k) - \hat{\overline \vw}(k)}_2\nonumber\\
	&\le  \EXs{\bm\xi}{F\!\left(\hat{\overline \vw}(k) \right)} + \sqrt{M} H \frac{1}{K} \sum_{k=0}^{K-1 }\lVert   \vw_i(k) - \overline \vw(k)\rVert_2 \nonumber\\
	&\le  \EXs{\bm\xi}{F\!\left(\hat{\overline \vw}(k) \right)} + \sqrt{M} H \frac{1}{K} \sum_{k=0}^{K-1 }\norm*{\mW(k) - \overline \mW(k)}_F \nonumber\\
	&\le  \EXs{\bm\xi}{F\!\left(\hat{\overline \vw}(k) \right)} + \sqrt{M} H \frac{\sqrt{M}\sqrt{\Rsp}}{K} \frac{1 - |\lambda_2|^{K}}{1- |\lambda_2|}  \nonumber\\
     & \phantom{\le}+ \sqrt{M} H \eta \sqrt{\Esp{}}\left(     
       (1-\alpha) \frac{K-1}{K}
    + \frac{ \alpha}{1- |\lambda_2|} \left( 1- \frac{1}{K} \frac{1 - |\lambda_2|^{K}}{1- |\lambda_2|} \right)
     \right).
	\label{e:bound_time_avg}
\end{align}

Putting together \eqref{e:bound_double_avg} and \eqref{e:bound_time_avg}, we obtain:
\begin{align}
\EXs{\bm\xi}{F\!\left(\hat{  \vw}_i(k) \right)} - F^*  
    & \le   \frac{M}{2 \eta K} {\dist{\overline \vw(0)} {\sW^*}}^2 + \frac{\eta E}{2 } +(2+\sqrt{M})H \frac{\sqrt{M}\sqrt{\Rsp}}{K} \frac{1 - |\lambda_2|^{K}}{1- |\lambda_2|}  \nonumber\\
     &\phantom{\le} + (2+\sqrt{M})H \eta \sqrt{\Esp{}}\left(     
       (1-\alpha) \frac{K-1}{K}
    + \frac{ \alpha}{1- |\lambda_2|} \left( 1- \frac{1}{K} \frac{1 - |\lambda_2|^{K}}{1- |\lambda_2|} \right)\right)\nonumber\\
    & \le   \frac{M}{2 \eta K} {\dist{\overline \vw(0)} {\sW^*}}^2 + \frac{\eta E}{2} + H \frac{3 M \sqrt{\Rsp}}{K} \frac{1 - |\lambda_2|^{K}}{1- |\lambda_2|} \nonumber\\
    & \phantom{\le}+ 3\sqrt{M}H \eta \sqrt{\Esp{}}\left(     
       (1-\alpha) \frac{K-1}{K}
    \phantom{\le}+ \frac{ \alpha}{1- |\lambda_2|} \left( 1- \frac{1}{K} \frac{1 - |\lambda_2|^{K}}{1- |\lambda_2|} \right)\right),
\end{align}
where the last inequality is simply to slightly compact the formula.
%Finally, replacing $\eta = \eta_0/\sqrt{K}$, we conclude the proof.
\end{proof}

\begin{cor}
\label{p:classic_bound_local}
Under assumptions A1-A5 and that constant learning rate $\eta(k)=\eta$ is used, an upper bound for the objective value at the end of the \mbox{($K-1$)th} iteration is given, for each $i$, by:
\begin{align}
\EX{F\!\left(\hat    \vw_i(K-1) \right)}  - F^*   \le  &  \frac{M}{2 \eta {K}} {\dist{\overline \vw(0)} {\sW^*}}^2 + \frac{\eta E}{2 } + \sqrt{E} \frac{3 M \sqrt{R}}{K} \frac{1 - |\lambda_2|^{K}}{1- |\lambda_2|}  + \frac{  3\eta \sqrt{M} E  }{1- |\lambda_2|} \left( 1-  \frac{1}{K}\frac{1 - |\lambda_2|^{K}}{1- |\lambda_2|} \right).
\label{e:classic_bound_local}	 
\end{align}
%\begin{align}
%\EX{F\!\left(\hat    \vw_i(K-1) \right)}- F^*  &\le   \frac{M}{2 \eta_0 \sqrt{K}} {\dist{\overline \vw(0)} {\sW^*}}^2 + \frac{\eta_0 E}{2 M \sqrt{K}} + \sqrt{E} \frac{3 M \sqrt{R}}{K} \frac{1 - |\lambda_2|^{K}}{1- |\lambda_2|} \nonumber\\
%    & \phantom{\le}+ 3\sqrt{M} E\frac{\eta_0}{\sqrt{K}}\frac{  1 }{1- |\lambda_2|} \left( 1-  \frac{1}{K}\frac{1 - |\lambda_2|^{K}}{1- |\lambda_2|} \right).
%\label{e:classic_bound}	 
%\end{align}
In particular, if workers compute full-batch subgradients and the 2-norm of subgradients of functions $F_i$ is bounded by a constant $L$, we obtain: 
\begin{align}
F\!\left(\hat    \vw_i(K-1) \right)- F^*  \le &  \frac{M}{2 \eta {K}} {\dist{\overline \vw(0)} {\sW^*}}^2 + \frac{\eta M L^ 2}{2 }  + L \frac{3 M^{3/2} \sqrt{R}}{K} \frac{1 - |\lambda_2|^{K}}{1- |\lambda_2|}  +\frac{  3 \eta M^{3/2} L^2}{1- |\lambda_2|} \left( 1-  \frac{1}{K}\frac{1 - |\lambda_2|^{K}}{1- |\lambda_2|} \right).
\label{e:classic_bound2_local}	 
\end{align}
\end{cor}
\begin{proof}
The result follows immediately from Proposition~\ref{p:new_bound_local} and \eqref{e:abc}.
\end{proof}

\section{Proof of Proposition~\ref{p:estimates}}
\label{a:estimates}
\begin{proof}
Our first remark is that selecting a uniform permutation and then a uniform random batch of size $B$ without resampling is equivalent to selecting uniformly $B$ elements from the dataset $\sS$ without resampling. We denote this sampling with the random variable $\xi_\sS$. This is true independently from the presence of data replication, as far as data point replicas are located at different nodes. 
%We denote this sample from the $\xi_B$.
From this observation and the formula for the variance of the average of a sample drawn  without resampling, it follows:
\begin{align}
	\EXs{\pi}{\EXs{\bm\xi}{\norm*{\mG}_F^2}} 
		& = \sum_{i,j}\EXs{\pi}{\EXs{\bm\xi}{\emG_{i,j}^2}} 
		   = \sum_{i,j} \EXs{\xi_{\sS}}{\emG_{i,j}^2} 
		   =  \sum_{i,j}\left(\left(\EXs{\xi_{\sS}}{\emG_{i,j}}\right)^2 + \Var_{\xi_{\sS}}\!\left[\emG_{i,j}\right]\right)\nonumber\\
		& =  \sum_{i,j} \left( (\partial F)_i^2  + \frac{S}{S-1}\frac{\sigma_i^2}{B} \left(1-\frac{B}{S}\right)\right)
		   =  \sum_{i,j} \left( (\partial F)_i^2  + \frac{S-B}{S-1}\frac{\sigma_i^2}{B} \right)\nonumber\\
		& = M \left(\norm*{\partial F}_2^2 + \frac{S-B}{S-1}\frac{\sigma^2}{B} \right). \label{e:expE}
\end{align}

From $\sum_{j=1}^M x_j^2 = \sum_{j=1}^M (x_j-\bar x)^2 + M \bar x^2$, where $\bar x$ denote the mean of the $M$ values $x_j$, it follows
\begin{align}
\norm*{\Delta\mG}_F^2 & = \norm*{\mG}_F^2  - \norm*{\mG \frac{\vone \vone^\intercal}{M}}_F^2.
\label{e:pythagora}
\end{align}
In order to compute $\EXs{\pi}{\EXs{\bm\xi}{\norm*{\Delta\mG}_F^2}} $, 
we will then compute $\EXs{\pi}{\EXs{\bm\xi}{\norm*{\mG \vone \vone^\intercal/M}_F^2}}$ 
and then use~\eqref{e:expE} and~\eqref{e:pythagora}. 
We denote the double expectation over $\bf\xi$ and $\pi$ simply as $\EX{}$.

\begin{align}
    \EX{ \left( \frac{1}{M} \sum_{j=1}^M G_{i,j} \right)^2 } 
            & = \frac{1}{M^2} \left(\sum_{j} \EX{G_{i,j}^2} +\sum_{j\neq j'}^M \EX{G_{i,j} G_{i,j'}} \right)\nonumber\\
            & = \frac{1}{M} \left( \EX{G_{i,j}^2} + (M-1) \EX{G_{i,j} G_{i,j'}} \right), \label{e:avg_grad}
\end{align}
where $j$ and $j'$ are arbitrary values in $\{1,\dots M\}$ with $j' \neq j$.
\begin{align}
    \EX{G_{i,j}^2}  & = \EX{G_{i,j}}^2 + \Var\left[G_{i,j}\right]
                        = (\partial F)_i^2 + \frac{S-B}{S-1} \frac{\sigma_i^2}{B}.
                        \label{e:single_square}
\end{align}
Let $\mathbbm 1_{s,s'}$ be the indicator function denoting if the datapoint $(x^{(s)},y^{(s)})$ has been selected in the minibatch at node $j$ and the datapoint $(x^{(s')},y^{(s')})$ has been selected in the minibatch at node $j'$. In order to keep the notation compact we also denote by $\partial f_{i,s}$ and $\partial f_{i,s'}$ respectively $(\partial f(\vw,(x^{(s)},y^{(s)})))_i$ and $(\partial f(\vw,(x^{(s')},y^{(s')})))_i$.
\begin{align}
    \label{e:exp_prod}
    \EX{G_{i,j} G_{i,j'}} & = \frac{1}{B^2} \sum_{s=1}^S \sum_{s'=1}^S \partial f_{i,s} \partial f_{i,s'} \EX{\mathbbm 1_{s,s'}}.
\end{align}
The probability that the point $s$ is in the local dataset at node $i$ and the point is selected in the minibatch is $\frac{CS/M}{S} \times \frac{B}{CS/M}=\frac{B}{S}$. Let us consider first the case $s'=s$. Given that $s$ is in $j$, the probability that one of the other $C-1$ copies is stored in $j'$ is $\frac{C-1}{M-1}$ and this copy has a probability $\frac{B}{CS/M}$ to be selected. Then the total probability of the event that a copy of $s$ is selected in the minibatch at $j$ and another copy is selected in the minibatch at $j'$ is 
\[\frac{B}{S} \times \frac{C-1}{M-1} \times \frac{B}{CS/M} = \frac{B^2}{S^2} \frac{C-1}{C} \frac{M}{M-1}.   \]
If $s'\neq s$, then $s'$ may be present in $j$ or not. The first event occurs with probability $\frac{\frac{CS}{M}-1}{S-1}$ (because $s$ has already been located in $j$) and in this case the probability that $s'$ is also located in $j'$ is $\frac{C-1}{M-1}$. $s'$ is not present in $j$ with probability $1-\frac{\frac{CS}{M}-1}{S-1}$ and in this case it is located in $j'$ with probability $\frac{C}{M-1}$.
Then the total probability of the event that a copy of $s'$ is located at  $j'$ given that a copy of $s$ is located at $j$ is 
\[\frac{\frac{CS}{M}-1}{S-1} \frac{C-1}{M-1}+ \left(1-\frac{\frac{CS}{M}-1}{S-1}\right)\frac{C}{M-1} = \frac{C}{M}+\frac{M-C}{M(M-1)(S-1)},\]
and the probability that a copy of $s$ is selected in the minibatch at  $j$ and a copy of $s'\neq s$  is located at $j'$ is 
\[\frac{B}{S} \times \frac{C}{M}+\frac{M-C}{M(M-1)(S-1)} \times \frac{B}{\frac{CS}{M}} = \frac{B^2}{S^2} \left( 1+ \frac{M-C}{C(M-1)(S-1)} \right). \]
In conclusion, we have just proved that:
\begin{align}
    \EX{\mathbbm 1_{s,s'}} = 
            \begin{cases}
                \frac{B^2}{S^2} \frac{C-1}{C} \frac{M}{M-1}, &\textrm{ for }s'= s,\\
                \frac{B^2}{S^2}\left(1+\frac{M-C}{C(M-1)(S-1)}\right), &\textrm{ for }s'\neq s.
            \end{cases}
    \label{e:indicator}
\end{align}
We observe that 
\begin{align}
    & \frac{1}{S} \sum_s (\partial f_{i,s})^2 = (\partial F)_i^2 +\sigma_i^2,\\
    & (\partial F)_i^2 = \left(\frac{1}{S} \sum_s \partial f_{i,s} \right)^2
        =\frac{1}{S^2} \left( \sum_s (\partial f_{i,s})^2 + \sum_{s \neq s'} \partial f_{i,s} \partial f_{i,s'} \right).
\end{align}
Using these two relations and~\eqref{e:indicator} in~\eqref{e:exp_prod} we obtain
\begin{align}
    \EX{G_{i,j} G_{i,j'}} & = (\partial F)_i^2 + \sigma_i^2 \frac{C-M}{C(M-1)(S-1)}.
\end{align}
This equality together with~\eqref{e:avg_grad} and~\eqref{e:single_square} leads to: 
\begin{align}
    \EX{\norm*{\mG \vone \vone^\intercal/M}_F^2} 
        & = \sum_i \EX{G_{i,j}^2} + (M-1) \EX{G_{i,j} G_{i,j'}}\nonumber\\
        & = \norm*{\partial F}_2^2 + \frac{S-B}{S-1} \frac{\sigma^2}{B} + (M-1) \left(\norm*{\partial F}_2^2 + \sigma^2 \frac{C-M}{C(M-1)(S-1)} \right)\nonumber\\
        & = M \norm*{\partial F}_2^2 + \sigma_2 \frac{CS-MB}{CB(S-1)}.
\end{align}
Finally,
\begin{align} 
\EX{\norm*{\Delta \mG}_F^2} & 
    = \EX{\norm*{\mG}_F^2} - \EX{\norm*{\mG \vone \vone^\intercal/M}_F^2}\\
    & = M \sigma^2 \left( \frac{S-B}{B(S-1)} - \frac{CS-MB}{MCB(S-1)}\right).
\end{align}

We now move to bound $\EXs{\pi}{\norm*{\EXs{\bm\xi}{\mG}}_F}$. The lower bound follows immediately from the fact that any norm is a convex function, so that
\[\EXs{\pi}{\norm*{\EXs{\bm\xi}{\mG}}_F} \ge \norm*{\EXs{\pi}{\EXs{\bm\xi}{\mG}}}_F = \sqrt{M} \norm*{\partial F}_2.\]
For the upper bound:
\begin{align}
    \EXs{\pi}{\norm*{\EXs{\bm\xi}{\mG}}_F}
        & = \EXs{\pi}{ \sqrt{ \sum_{i,j} \left(\EXs{\bm\xi}{G_{i,j}}\right)^2} }\nonumber\\
        & \le \sqrt{ \sum_{i,j} \EXs{\pi}{  \left(\EXs{\bm\xi}{G_{i,j}}\right)^2} }\\
        & = \sqrt{\sum_{i,j} \left(\EXs{\pi}{\EXs{\bm\xi}{G_{i,j}}}\right)^2 + \Var_{\pi}\left[\EXs{\bm\xi}{G_{i,j}}\right]}\\
        & = \sqrt{\sum_{i,j} (\partial F)_i^2 + \sigma_i^2 \frac{M-C}{C(S-1)}}\\
        & = \sqrt{M \left( \norm*{\partial F}_2^2 + \frac{M-C}{C(S-1)} \sigma^2\right)},
\end{align}
where the first inequality follows from the concavity of square root function. Observe that  $\EXs{\bm\xi}{G_{:,j}}$ is equal to the full-batch gradient computed on the local dataset at node $j$. As such, its expected value over the permutations $\pi$ is equal to $\partial F$, and its variance corresponds to the variance of the sample average when we draw $CS/M$ samples from a dataset of $S$ elements without resampling.
\end{proof}

%\section{Effect of the replication factor $C$ on $\beta$}

\section{Toy example}
\label{a:toy_example}
In this section we consider a toy example to illustrate the implications of the findings in the previous section. 
We want to solve the following problem:
\begin{alignat}{3}
& \minim_w &&  F(w)=\frac{1}{M}\sum_{l=1}^M 1 - y^{(l)}  x^{(l)} w \label{e:toy_problem}\\
& \textrm{subject to }\quad &&  w \in \sW\nonumber
\end{alignat}
%This function is often used for Support Vector Machine (SVM) classification 
where $x^{(l)}$ is a data feature and $y^{(l)} \in \{-1,1\}$ identifies the class label. Note that there is only one parameter $w \in \mathbb R$ ($n=1$) and, for the moment, the number of nodes equals the size of the dataset ($M=m$). 
%We compare these results with our theoretical findings in the previous sections. 
$\Delta \mG(k) = \mG(k) - \mG(k) \frac{\vone \vone^\intercal}{M}$ is in this case a constant row vector.

\begin{figure*}[h]
%\vspace{.3in}
        \centering
        \begin{subfigure}{0.3275\textwidth}
            \centering
            \includegraphics[width=\textwidth]{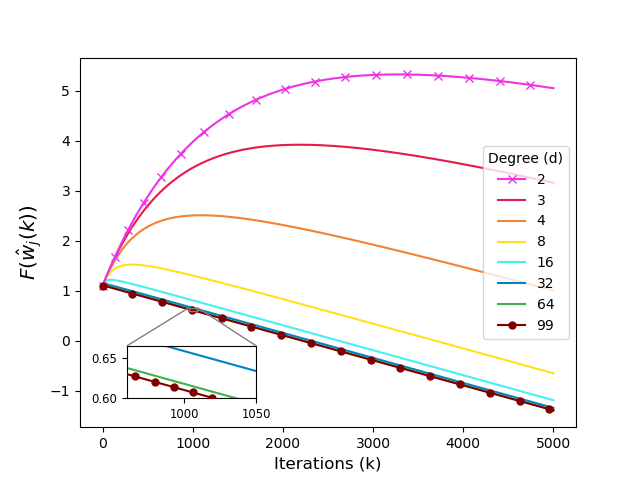}
            \caption%
            {{\small Rings with aligned gradient}}    
            \label{f:rings}
        \end{subfigure}
        \hfill
        \begin{subfigure}{0.3275\textwidth}  
            \centering 
            \includegraphics[width=\textwidth]{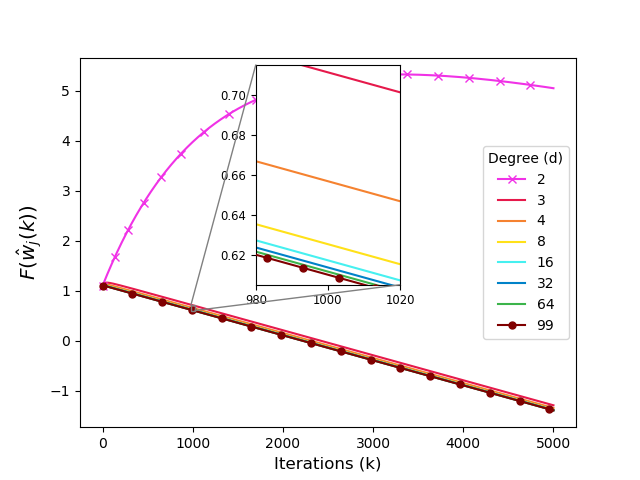}
            \caption%
            {{\small Expanders with aligned gradient}}    
            \label{f:expanders}
        \end{subfigure}
        %\vskip\baselineskip
        \hfill
        \begin{subfigure}{0.3275\textwidth}   
            \centering 
            \includegraphics[width=\textwidth]{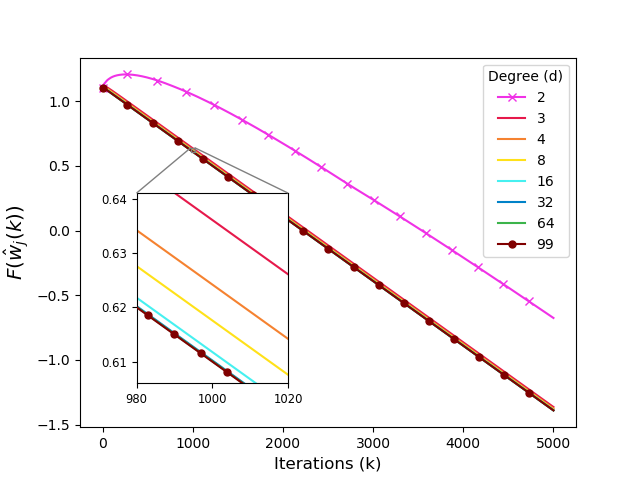}
            \caption%
            {{\small Expanders with generic gradient}}    
            \label{f:no_alignment}
        \end{subfigure}
%        \quad
%        \begin{subfigure}[b]{0.475\textwidth}   
%            \centering 
%            \includegraphics[width=\textwidth]{figures/simulations_wide_layout/test4_eigvecsvm_100n_err_vs_iter.png}
%            \caption[]%
%            {{\small Network 4}}    
%            \label{fig:mean and std of net44}
%        \end{subfigure}
	%\vspace{.3in}
        \caption[ ]
        {\small Objective function~(\ref{e:toy_problem}) evaluated for the worst model $\vw_j(k)$ versus number of iterations: effect of the topology and its relation with the gradients.} 
        \label{f:gradient_and_topology}
    \end{figure*}

    \begin{figure*}[h]
    	%\vspace{.3in}
        \centering
        \begin{subfigure}[b]{0.3275\textwidth}
            \centering
            \includegraphics[width=\textwidth]{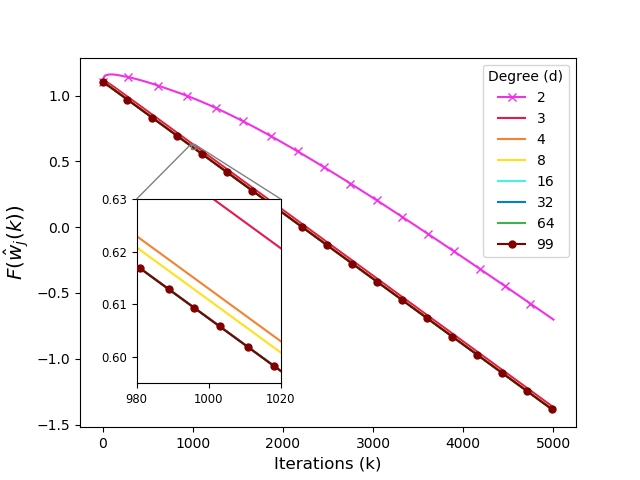}
            \caption[]%
            {{\small 2X dataset}}    
            \label{f:similar_gradients2}
        \end{subfigure}
        \hfill
        \begin{subfigure}[b]{0.3275\textwidth}  
            \centering 
            \includegraphics[width=\textwidth]{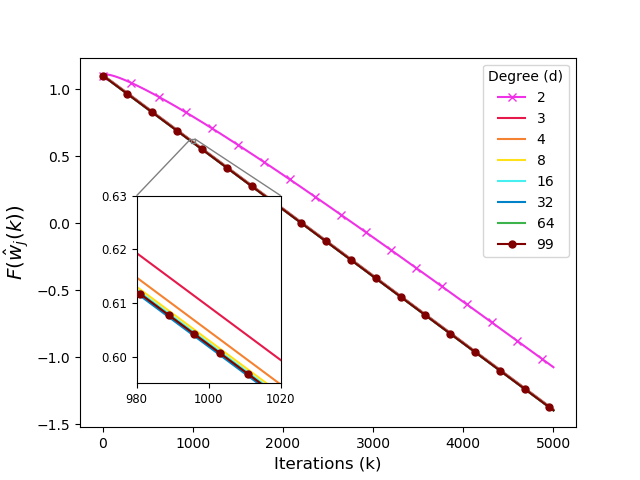}
            \caption[]%
            {{\small 20X dataset}}    
            \label{f:similar_gradients20}
        \end{subfigure}
        %\vskip\baselineskip
        \hfill
        \begin{subfigure}[b]{0.3275\textwidth}   
            \centering 
            \includegraphics[width=\textwidth]{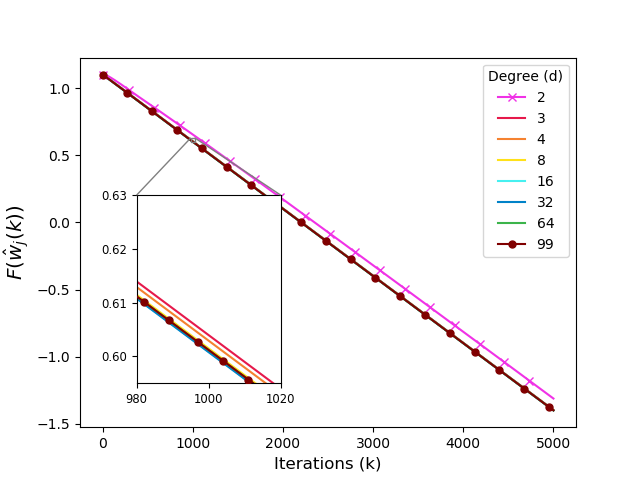}
            \caption[]%
            {{\small 100X dataset}}    
            \label{f:similar_gradients100}
        \end{subfigure}
%        \quad
%        \begin{subfigure}[b]{0.475\textwidth}   
%            \centering 
%            \includegraphics[width=\textwidth]{figures/simulations_wide_layout/test4_eigvecsvm_100n_err_vs_iter.png}
%            \caption[]%
%            {{\small Network 4}}    
%            \label{fig:mean and std of net44}
%        \end{subfigure}
	%\vspace{.3in}
        \caption[]
        {\small Objective function~(\ref{e:toy_problem}) evaluated for the worst model $\vw_j(k)$ versus number of iterations: effect of dataset size. Expanders with generic gradient.} 
        \label{f:similar_gradients}
    \end{figure*}

%The datasets are generated as follows.
%We start from a vector $\vu$ of values in $[-1,1]$ which sum up to zero. % ($\vu^\intercal \vone=0$). 
%We describe later how the vector $\vu$ is selected in the different experiments. Features $x^{(l)}$ are defined as $x^{(l)} = |u_l+\zeta|$, where $\zeta$ is a small positive constant. The labels are defined as $y^{(l)}=-\sign(u_l+\zeta)$. We select $\sW=[-30,1]$. Observe that $F(w) = 1 - \sum_l y^{(l)} x^{(l)}/M = 1 + \zeta w$,
%%$F'(w)= -M^{-1} \sum_l y^{(l)} x^{(l)} =  \zeta+ M^{-1}\sum_l u_l  = \zeta>0$, 
%and consequently the minimizer of problem~(\ref{e:toy_problem}) is $w^*=-30$. We select $\zeta=1/10$, $\eta(k)= 1/10$ and all nodes start with the same initial estimate $w_i(0)=1$.

We consider an undirected $d$-regular graph with $M=100$ nodes and homogeneous  non-negative weights, i.e.,~$A_{i,j}=1/(d+1)$ for $(i,j) \in \Ed$ and $A_{i,j}=0$ otherwise. Matrix $\mA$ is then symmetric. 

%Our numerical results have been obtained through a discrete event simulator developed in Python. 

%Fig.~XXX~a) shows a  typical dataset used for the next experiment. 
%is positive if and only if positive labelled samples have a larger mean than negative labeled ones. In the datasets $x^*<0$, because $\sum_l$.
%Let $J^+$ (/$J^-$) be the samples belonging to class $+1$ (/$-1$). The two classes have the same number of elements. The minimizer $x^*$ of \eqref{e:simple_svm}  is positive if and only if positive labelled samples have a larger mean than $ \sum_{l\in J^+}  \vx^{(l)} >  \sum_{l\in J^-} \vx^{(l)}$. the average feature value for samples of nodes with a positive label have  of  simple case the optimal value $x^*$ will be positive 
%where $\chi$ is the feature value and $y$

%\paragraph{Bound XXX is tight.}
Figure~\ref{f:rings} presents results for undirected $d$-regular rings, where each node $i$ %$ \in \{0, \dots M-1\}$ 
is connected to nodes $i-1$ and $i+1$ (hence forming a cycle) as well as the $d-2$ closest nodes  on the cycle.
%: for $d$ even, to nodes  $\{i-d/2, i-d/2+1, \dots, i-2, i+2, \dots, i+d/2-1, i+d/2\}$, for $d=3$,  to $i+2$ or $i-2$.
%For $d$ even, $i$ is connected to nodes  $\{i-d/2, i-d/2+1, \dots, i-2, i+2, \dots, i+d/2-1, i+d/2\}$. For $d=3$, node $i$ is connected to $i+2$ or to $i-2$.
%i.e., in the set $\{i-d, i-d+1, \dots, i-2, i+2, \dots, i+d-1, i+d\}$ (sum and differences are module $M$).
Among graphs with degree~$d$, these rings are poorly connected, with a diameter that is order of $M/d$ and low spectral gap $1-|\lambda_2|$. 
The curves show the evolution of the objective function evaluated at the worst estimate, i.e.,~$\max_i F(\hat w_i(k))$ %, across different iterations 
for different values of the degree $d$.
For each degree, the dataset is chosen selecting a vector $\vu$ orthogonal to $\vone$ and perturbing each of its components by a small amount $\zeta$, so that $\mG = \vu^\intercal + \zeta \vone^\intercal$ and $\Delta \mG(k) = \vu^\intercal$. Because $\vone$ and $\vu$ are orthogonal, it follows that 
\[E = \norm*{\vu^\intercal + \zeta \vone^\intercal}_F^2= \norm*{\vu}_2^2 + \zeta^2 \norm*{\vone}_2^2 = \norm*{\vu}_2^2 + M \zeta^2 = \Esp{} + M \zeta^2.\] 
Then, by selecting $\zeta$ small enough we can make $\Esp{}$ and $E$ arbitrarily close. This also a full-batch setting, so that $H=\sqrt{E}$.
We can select $\vu$ equal to a left eigenvector of $\mA$ relative to $\lambda_2$, because $\mA$ has orthogonal eigenvectors among which there is $\vone$. In this case, we say that gradients are \emph{aligned with the topology},\footnote{
	It corresponds to the slowest convergence scenario for distributed dual averaging, see~{\cite[Prop.~1]{duchi12}}.
} If $\vu$ is a left eigenvector of $\mA$, all the energy of $\Delta\mG(k)$ is in the subspace defined by $\mP_2$, then $\alpha=\sqrt{e_2}=1$. Details about how the dataset is built are in the sections below,  together with calculations showing that  the value of the objective function~is 
\begin{align}
 \max_i & (F(\hat w_i(k-1)))  =  1 +  \zeta + \frac{  \eta \zeta}{1-\lambda_2} \left( 1- \frac{1- \lambda_2^{k}}{k(1-\lambda_2)}\right) -  \eta  \zeta^2 \frac{k}{2}. 
\label{e:toy_worst}
\end{align}
Equation (\ref{e:toy_worst}) exactly matches the plots in Fig.~\ref{f:rings}. 
Comparing~(\ref{e:toy_worst}) with~(\ref{e:classic_bound}) and~(\ref{e:new_bound}), we recognize the same dependence on the second largest eigenvalue.\footnote{
	Equation~(\ref{e:toy_worst}) indicates a much faster convergence than those bounds, because, for simplicity, we considered a differentiable linear objective function~\eqref{e:toy_problem}.
} 
Because $\Esp{} \approx E$, $E=\sqrt{E}$, and $\alpha=1$,  the two bounds~(\ref{e:classic_bound}) and (\ref{e:new_bound}) are almost equivalent, but for the fact that~(\ref{e:new_bound}) correctly takes into account that there is no dependence on the initial estimates ($\Rsp=0$). % because all the nodes start from the same local model $\Rsp$.

Figure~\ref{f:rings} shows clearly a high variability of the performance of the optimization algorithm across different topologies. The number of iterations required to achieve a given approximation of the optimum is orders of magnitude larger for the cycle ($d=2$) than for the clique ($d=99$).

In Fig.~\ref{f:expanders} the same experiments are executed on $d$-regular connected random graphs. These graphs are known to be Ramanujan graphs with high-probability, i.e.,~they have the % largest possible spectral gap  
smallest $|\lambda_2|$ among all graphs with the same degree~\citep{mckay81}. The curves still follow~(\ref{e:toy_worst}), but because $\lambda_2$ is  smaller, objective function $F(\cdot)$ takes on smaller values. Note that the curves for $d=2$ and $d=M-1$ are unchanged, because the corresponding graphs are always the same (the cycle and the clique, respectively). We also observe that the relative performance differences  for graphs with $d\ge 3$ are much smaller:  the marginal benefit to increase connectivity is much less for random graphs.

Until now, we have assumed  that row matrix $\Delta\mG$ is aligned with a left eigenvector of $\mA$, but there is no particular reason to think this should be the case. Figure~\ref{f:no_alignment} shows numerical results for the case when the same datasets as in Fig.~\ref{f:expanders} are used, but the graph is a new  independently generated $d$-regular random graph.\footnote{
	For the case $d=2$, the original dataset is randomly distributed across the nodes.
} %are generated, independently from those used to generate the datasets.
In the figure, we see that connectivity is even less important and that the curve for $d=2$ starts approaching the others. This experiment shows the effect of $\alpha$. $\vu$ is the most difficult configuration to average for the consensus matrix $\mA$. Now vector $\Delta\mG$ and the vector $\vu$ %of $\mA$ %corresponding to $\lambda_2$ 
are arbitrarily oriented, so that on average only $(1/M)$th of the energy of $\Delta\mG$ falls in the direction of~$\vu$.

%Another way to quantify this effect is to look at the norm of the perturbation vector $ \vu^{\intercal} \frac{1- \lambda_2^k}{1-\lambda_2}$

%We look now at the effect of $\Lsp$. 
We now look at how gradient variability affects convergence.
%In the experiments  the local gradient  $\vg_i$ at node~$i$ is equal to $u_i+\zeta$, where $u_i \in [-1,1]$ so nodes have very different local gradients. 
%and the differences $\vg_i - \bar \vg$ are equal to $u_i$ so they are spread over $[-1,1]$. 
We increase the dataset size and have each node compute its local function on the basis of more data samples. The additional data is built as above from the second eigenvectors of new  independently generated $d$-regular random graphs. The new data have then the same statistical properties. Figure~\ref{f:similar_gradients} shows what happens %for dataset sizes 2, 20, and 100 times larger, and 
when each node stores respectively 2, 20, and 100 data points.  Because the local datasets become more and more similar as dataset size increases, $\Esp{}$ reduces and the curves get closer and closer. For a $100$x dataset,  the convergence rate of a cycle ($d=2$) differs little from that of a fully connected graph.

\subsection{Datasets' generation}
The datasets are generated as follows.
We start from a vector $\vu$ of values in ${[-1,1]}$ which sum up to zero. % ($\vu^\intercal \vone=0$). 
We describe later how the vector $\vu$ is selected in the different experiments. Features $x^{(l)}$ are defined as $x^{(l)} = |u_l+\zeta|$, where $\zeta$ is a small positive constant. The labels are defined as $y^{(l)}=-\sign(u_l+\zeta)$. We select $\sW={[-30,1]}$. Observe that $F(w) = 1 - \sum_l y^{(l)} x^{(l)}/M = 1 + \zeta w$,
%$F'(w)= -M^{-1} \sum_l y^{(l)} x^{(l)} =  \zeta+ M^{-1}\sum_l u_l  = \zeta>0$, 
and consequently the minimizer of problem~(\ref{e:toy_problem}) is $w^*=-30$. We select $\zeta=1/10$, $\eta(k)= 1/10$ and all nodes start with the same initial estimate $w_i(0)=1$.

We say that gradients are \emph{aligned with the topology}, when vector $\vu$ is a left eigenvector of $\mA$ relative to 
%the second largest eigenvalue 
$\lambda_2$,  normalized so that $\lVert \vu \rVert_\infty=1$ and $\min_i{u_i}=-1$. In this case, $\Delta \mG(k) = \mG(k) - \mG(k) \frac{\vone \vone^\intercal}{M} = \vu^\intercal$.

\subsection{Proof of~\eqref{e:toy_worst}}
In the toy example, the estimates matrix (a $1\times M$ vector in this case) evolves as
\begin{align*}
\mW(k) 	& = \mW(0) \mA^k - \sum_{h=0}^{k-1}  \eta \mG \mA^{k-1-h} \\
		&= \vone^{\intercal} - \eta \sum_{h=0}^{k-1} \left( \vu^{\intercal} \mA^{k-1-h} + \zeta \vone^{\intercal} \mA^{k-1-h} \right)
		 = \vone^{\intercal} - \eta  \sum_{h=0}^{k-1}\left(\lambda_2^{k-1-h} \vu^\intercal  + \zeta \vone^{\intercal} \right) \\
		&= \vone^{\intercal} - \eta \vu^{\intercal} \frac{1- \lambda_2^k}{1-\lambda_2} - \eta \zeta k \vone^{\intercal}.
\end{align*}

If $\lambda_2>0$, and $\zeta $ is small enough, the worst local model is the one of the node, call it $j$, that stores the data pair $x^{(j)}= |-1+ \zeta|$  for which $u_j=-1$. 
Its local model evolves as:
\begin{align*}
w_j(k) & = 1 + \eta \frac{1- \lambda_2^k}{1-\lambda_2} - \eta \zeta k,
\end{align*}

For node $j$ the time-average model is
\begin{align*}
\hat w_j(k-1) & = 1 +  \frac{\eta}{1-\lambda_2} \left( 1- \frac{ 1- \lambda_2^{k}}{k (1-\lambda_2)}\right) - \eta \zeta \frac{k}{2}. \\
\end{align*}
Finally, the objective function is 
\begin{align}
 \max_i (F(\hat w_i(k-1))) & = F(\hat w_j(k-1)) \nonumber\\
& = 1 +  \zeta + \frac{  \eta \zeta}{1-\lambda_2} \left( 1- \frac{1- \lambda_2^{k}}{k(1-\lambda_2)}\right) -  \eta  \zeta^2 \frac{k}{2}.\tag{\ref{e:toy_worst}}
\end{align}
Equation (\ref{e:toy_worst}) exactly matches the plots in Fig.~\ref{f:rings}. This implicitly confirms that $\lambda_2>0$.

Because $\Delta\mG(k)=\vu^\intercal $  is a left eigenvector of $\mA$ corresponding to $\lambda_2$, all the energy of $\Delta\mG(k)$ is in the subspace defined by $\mP_2$, then $\alpha=\sqrt{e_2}=1$.
%Moreover, the tightest bound for the 2-norm of the local gradients is $L=\max_l | u_l + \zeta |$. Similarly,  $\Lsp = \sqrt{\sum_i u_i^2} \ge \max_i |u_i | = |-1|=1$. We note that by tuning the value of $\zeta$ we can make $L$ arbitrarily close to $1$. It follows that, in this situation, the two bounds~(\ref{e:classic_bound}) and (\ref{e:new_bound}) are almost equivalent, but for the fact that~(\ref{e:new_bound}) correctly takes into account that there is no dependence on the initial estimates ($\Rsp=0$). % because all the nodes start from the same local model $\Rsp$.

%\twocolumn[

\clearpage
\newpage

\section{Experiments}
\label{a:experiments}
We consider two families of regular graphs that we call directed ring lattices, and undirected expanders. Let $\sV=\{0,1,2, \dots M\}$ be the set of nodes.  In a directed regular ring lattice with degree $d$, each node $i$ is connected to nodes $(i+1)\!\!\mod M$, $(i+2)\!\!\mod M$, \dots, $(i+d)\!\!\mod M$. An undirected regular expander is obtained by generating a large number (200) of regular random graphs (using the \textsc{NetworkX} implementation of the algorithm in~\citep{mckay90}), and selecting the one with the largest spectral gap. All the experiments presented in the main paper are performed on undirected regular expanders.

For each experiment, the learning rate has been set using the configuration rule in \citep{smith17}. We increase geometrically the learning rate and evaluate the training loss after one iteration. We then determine two ``knees'' in the loss versus learning rates plots: the  learning rate value for which the loss starts decreasing significantly and the value for which it starts increasing again.  We set the learning rate to the geometric average of these two values. As Fig.~\ref{f:lr_config} shows, for a given experiment, this procedure leads to select the same learning rate independently of the degree of the topology, respectively equal to $\eta = 3 \times 10^{-4}$ for the linear regression (Fig.~\ref{f:lr_config_regr}), $\eta = 6 \times 10^{-4}$ for ResNet18 on MNIST dataset (Fig.~\ref{f:lr_config_res_cifar}), $\eta = 0.1$ for 2-conv layers on MNIST dataset (Fig.~\ref{f:lr_config_nn_mnist}) and $\eta = 0.05$ for ResNet18 on CIFAR-10 dataset (Fig.~\ref{f:lr_config_nn_cifar}). 

\begin{figure*}
%\vspace{.3in}
        \centering
        \begin{subfigure}[b]{0.45\textwidth}
            \centering
            \includegraphics[width=\textwidth]{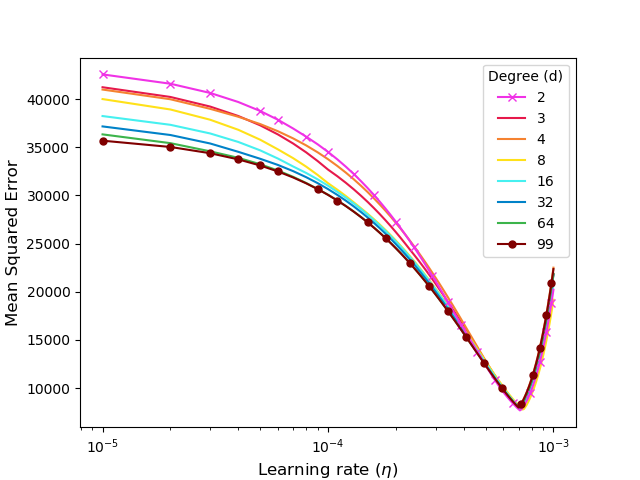}
            \caption[]%
            {{\small Linear regression on CT dataset, M=100}}    
            \label{f:lr_config_regr}
        \end{subfigure}
        \hfill
        \begin{subfigure}[b]{0.45\textwidth}
            \centering
            \includegraphics[width=\textwidth]{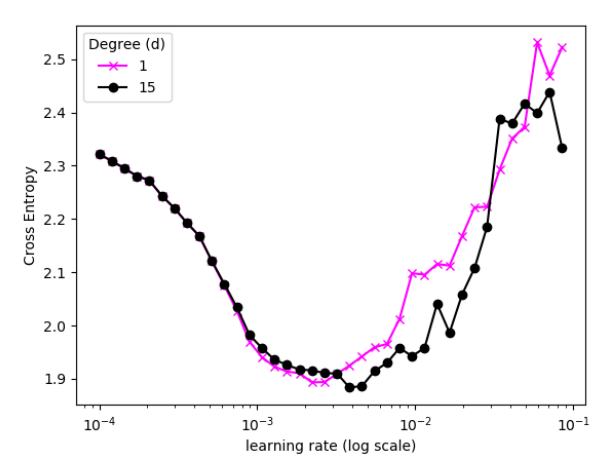}
            \caption[]%
            {{\small ResNet18 on MNIST dataset, M=16}}    
            \label{f:lr_config_res_cifar}
        \end{subfigure}

        \begin{subfigure}[b]{0.45\textwidth}  
            \centering 
            \includegraphics[width=\textwidth]{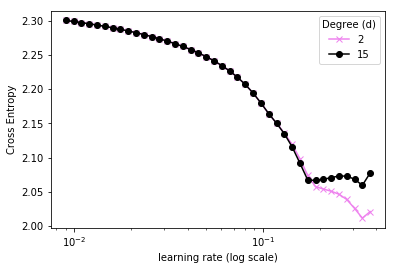}
            \caption[]%
            {{\small 2-conv layers on MNIST, M=16}}    
            \label{f:lr_config_nn_mnist}
        \end{subfigure}
        %\vskip\baselineskip
        \hfill
        \begin{subfigure}[b]{0.45\textwidth}  
            \centering 
            \includegraphics[width=\textwidth]{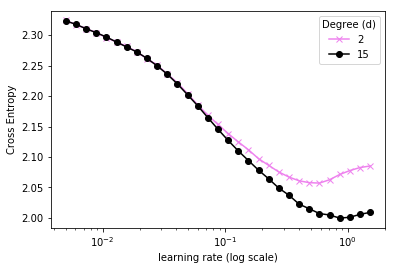}
            \caption[]%
            {{\small ResNet18 on CIFAR-10, M=16}}    
            \label{f:lr_config_nn_cifar}
        \end{subfigure}
         \caption[]
        {\small Error after one iteration vs learning rate. } 
        \label{f:lr_config}
    \end{figure*}

In comparison to Table~\ref{t:parameters}, Table~\ref{t:parameters2} also shows the predictions for an ``intermediate'' bound, obtained by replacing in \eqref{e:classic_bound} $R$ with $\Rsp{}$. These predictions are denoted by $k''_o$. Comparing $k''_o$ with $k'_o$ and $k'_n$ reveals that the difference between $k'_o$ and $k'_n$ is mainly due to the effect of $\Esp{}$, $H$, and $\alpha$, while $\Rsp{}$ plays a less important role.

\begin{table*}
\caption{Empirical estimation of $E$, $E_{\textrm{sp}}$, $H$, $\alpha$ on different ML problems and comparison of their joint effect ($\beta$) with the value $\widehat \beta$ predicted through~\eqref{e:perm_approx}. Number of iterations by which training losses for the ring and the clique differ by $4\%$, $10\%$, as predicted by the old bound~\eqref{e:classic_bound}, $k'_o$, by the old bound~\eqref{e:classic_bound} where $R$ is replaced by $R_{sp}$, $k^{''}_o$, by the new one~\eqref{e:new_bound}, $k'_n$, and as measured in the experiment, $k'$. When their values exceeds the total number of iterations we ran (respectively 1200 for CT, 1190 for MNIST, and 1040 for CIFAR-10), we simply indicate it as $\infty$.}
\label{t:parameters2}
\setlength\tabcolsep{2.5pt} 
\centering
\begin{tabular}{ccccccccccccccccccc}
\hline
\multirow{2}{*}{Dataset} & \multirow{2}{*}{Model}& \multirow{2}{*}{M}& \multirow{2}{*}{B} & \multirow{2}{*}{$\sqrt{E/E_{sp}}$}& \multirow{2}{*}{$\sqrt{E}/H$} & \multirow{2}{*}{$\frac{1}{\alpha}$} & \multirow{2}{*}{$\beta$} & \multirow{2}{*}{$\widehat \beta$}& \multicolumn{4}{c}{$@ 4\%$}  & \multicolumn{4}{c}{$@ 10\%$}\\
&&&&&&&&&$k'_o$&$k^{''}_o$& $k'_n$& $k'$ &$k'_o$& $k^{''}_o$& $k'_n$ & $k'$\\
\hline 
\multirow{4}{*}{ \begin{tabular}{c}CT\\ (S=52000)\end{tabular}}& 
\multirow{4}{*}{ \begin{tabular}{c}Linear regr.\\ n=384\end{tabular}}& 
\multirow{2}{*}{16}        &128  & 7.92  &1.01 & 1.53 & 12.23 & 12.31& 1 &2&  $\infty$ & $\infty$ & 1&5& $\infty$ & $\infty$ \\
            & &            & 3250 & 38.45  &    1.00&1.64 & 62.86 & 60.97&1 & 2&$\infty$ & $\infty$ & 1&5& $\infty$ & $\infty$ \\ 
&&\multirow{2}{*}{100}     & 128 &7.75  &1.01 &1.54 &12.05 & 11.56&1 &2& 10 & $\infty$ & 1&4& $\infty$ & $\infty$\\
& &                        & 520 &15.58 &1.00 &1.51 &23.60 & 22.96&1 &2 &17 & $\infty$ & 1& 4& $\infty$ & $\infty$\\
\hline
\multirow{3}{*}{ \begin{tabular}{c}MNIST\\ (S=60000)\end{tabular}}& 
\multirow{4}{*}{ \begin{tabular}{c}2-conv layers\\ n=431080\end{tabular}}
& \multirow{2}{*}{16} &128 & 1.45 &1.42 &1.49 & 3.07 &2.92& 1 &5 &16 & $\infty$ & 1&11& 72 & $\infty$\\
     & &                   &500 & 2.15 & 1.14 &1.53 & 3.75 &3.71 & 1 &6& 22& 40&1 &14& 260 &  $\infty$ \\
%    & &                   &1000 & 2.64 & 1.08 &1.48 & 4.22 & \\
& &                     64    & 128 & 1.41 & 1.42 & 1.51  &3.02 &3.03& 1&5 &10 &$\infty$ & 1& 11& 24& $\infty$\\
\cline{1-1}\cline{3-17}
split by digit&&           10   &500 &1.01 & 1.00 &1.42 & 1.42 & 3.62& 1 &2& 3 & 60 & 1&3& 7& 100\\
\hline
\multirow{2}{*}{ \begin{tabular}{c}CIFAR-10\\ (S= 50000)\end{tabular}}& 
\multirow{2}{*}{ \begin{tabular}{c}ResNet18\\ n=11173962\end{tabular}}& 
\multirow{2}{*}{16}        &128  &1.07&3.35 &1.49 & 5.34&5.62 &1 &2& 10& 30 & 1&4&20&$    \infty$\\
            & &          & 500  &1.18&1.91 & 1.50&  3.40&3.52& 1  & 10& 21& $\infty$& 1&20& 250 & $\infty$\\ 
\hline

\end{tabular}
\end{table*}

Figures~\ref{f:cdf_spark} and \ref{f:cdf_asciq}  show the Cumulative Distribution Functions of computing times for the Spark cluster and for ASCI Q super-computer, respectively. 

The following experiments have been carried out:
\begin{enumerate}
	\item Linear regression on CT dataset on undirected expanders with computation time distribution from the Spark cluster (Fig.~\ref{f:regression_spark}) and from ASCI Q super-computer (Fig.~\ref{f:regression_asciq}).
	%\item SVM on SUSY dataset with computation time distribution from the Spark cluster~Fig.~\ref{f:classification_spark}.
	\item ResNet18 on MNIST dataset on directed ring lattices with computation time distribution from the Spark cluster (Fig.~\ref{f:nn_spark}).
	%\item SVM on SUSY dataset with computation time distribution from ASCI Q super-computer~Fig.~\ref{f:classification_asciq}.
	\item 2-conv layers on MNIST dataset on undirected expanders with computation time distribution from ASCI Q super-computer (Fig.~\ref{f:nn_asciq})
	\item ResNet18 on CIFAR-10 dataset on undirected expanders with computation time distribution from the Spark cluster (Fig.~\ref{f:cifar_spark}) and from ASCI Q super-computer (Fig.~\ref{f:cifar_asciq}).
\end{enumerate}

\begin{figure*}
%\vspace{.3in}
        \centering
        \begin{subfigure}[b]{0.4\textwidth}
            \centering
            \includegraphics[width=\textwidth]{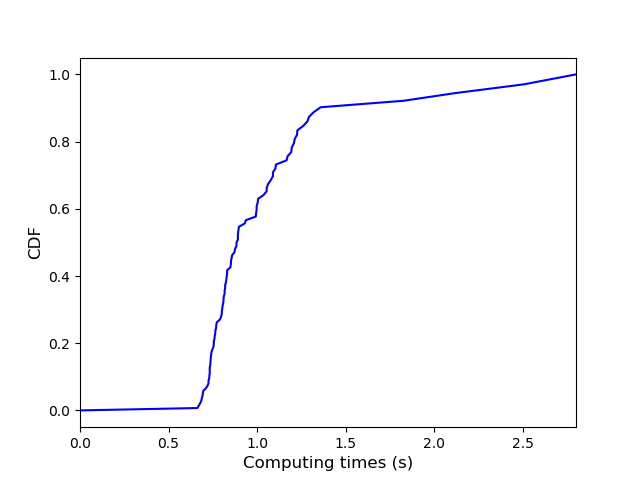}
            \caption[]%
            {{\small Spark cluster}}    
            \label{f:cdf_spark}
        \end{subfigure}
        \hfill
        \begin{subfigure}[b]{0.4\textwidth}  
            \centering 
            \includegraphics[width=\textwidth]{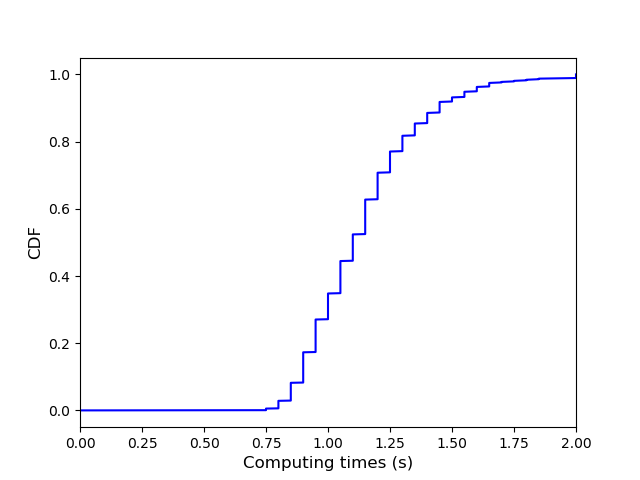}
            \caption[]%
            {{\small ASCI Q}}    
            \label{f:cdf_asciq}
        \end{subfigure}
        %\vskip\baselineskip
         \caption[ The average and standard deviation of critical parameters ]
        {\small Empirical distribution of the computation times.} 
        \label{f:cdf}
    \end{figure*}

    \begin{figure*}
%\vspace{.3in}
        \centering
        \begin{subfigure}[b]{0.3\textwidth}
            \centering
            \includegraphics[width=\textwidth]{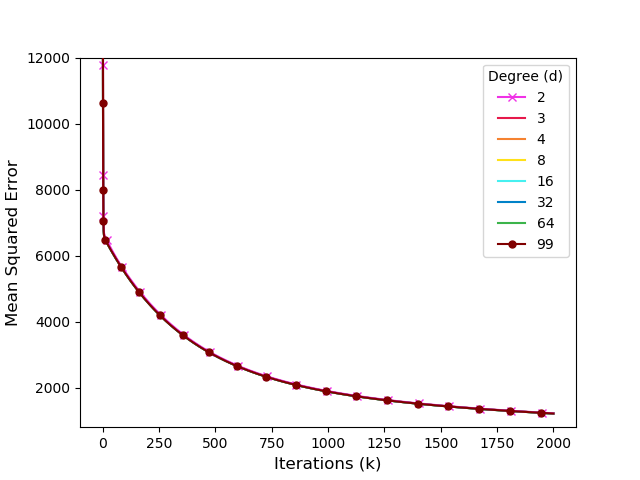}
            \caption[]%
            {{\small Error vs iterations}}    
            %\label{f:regr_err_vs_iter}
        \end{subfigure}
        \hfill
        \begin{subfigure}[b]{0.3\textwidth}  
            \centering 
            \includegraphics[width=\textwidth]{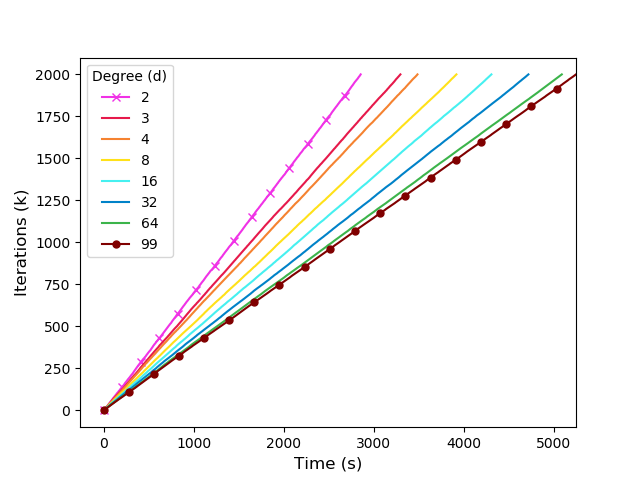}
            \caption[]%
            {{\small Throughput}}    
            %\label{f:regr_iter_vs_time}
        \end{subfigure}
        %\vskip\baselineskip
        \hfill
        \begin{subfigure}[b]{0.3\textwidth}   
            \centering 
            \includegraphics[width=\textwidth]{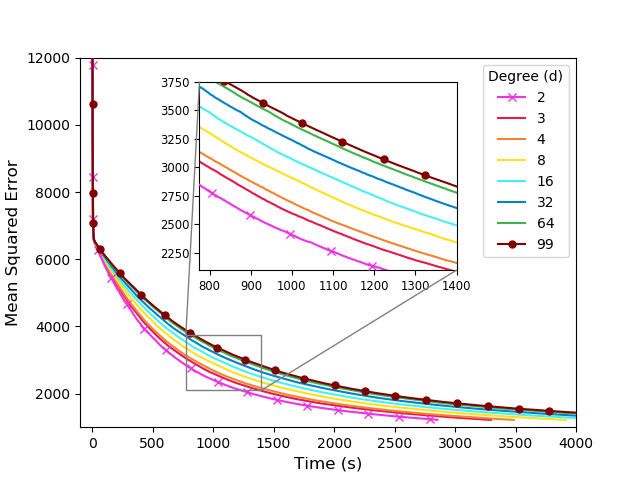}
            \caption[]%
            {{\small Error vs time}}    
            %\label{f:regr_error_vs_time}
        \end{subfigure}
         \caption[ The average and standard deviation of critical parameters ]
        {\small Effect of network connectivity  (degree $d$) on the convergence for linear regression on  dataset CT with computation times from a Spark cluster. M = 100, B = 128. } 
        \label{f:regression_spark}
    \end{figure*}
    \begin{figure*}
%\vspace{.3in}
        \centering
        \begin{subfigure}[b]{0.3\textwidth}
            \centering
            \includegraphics[width=\textwidth]{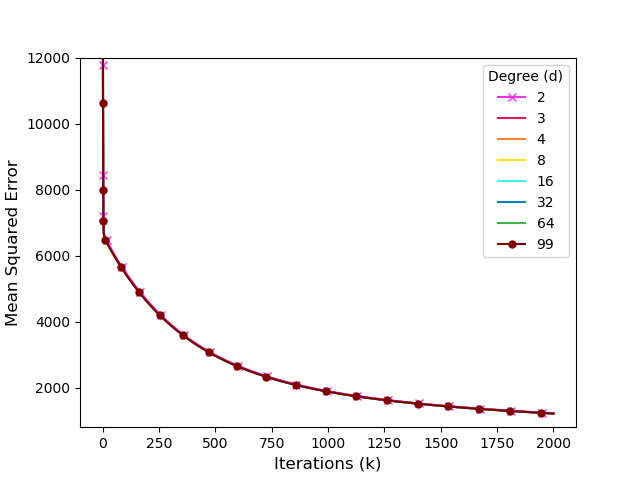}
            \caption[]%
            {{\small Error vs iterations}}    
            %\label{f:regr_err_vs_iter}
        \end{subfigure}
        \hfill
        \begin{subfigure}[b]{0.3\textwidth}  
            \centering 
            \includegraphics[width=\textwidth]{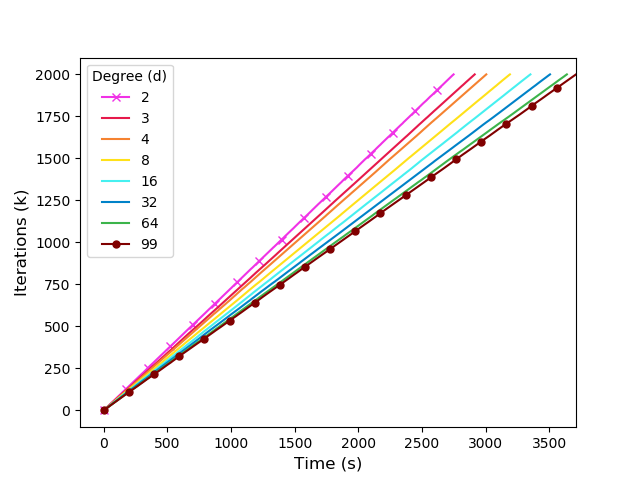}
            \caption[]%
            {{\small Throughput}}    
            %\label{f:regr_iter_vs_time}
        \end{subfigure}
        %\vskip\baselineskip
        \hfill
        \begin{subfigure}[b]{0.3\textwidth}   
            \centering 
            \includegraphics[width=\textwidth]{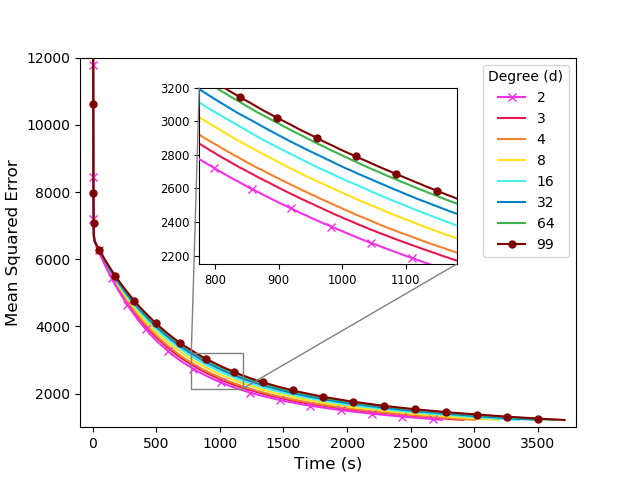}
            \caption[]%
            {{\small Error vs time}}    
            %\label{f:regr_error_vs_time}
        \end{subfigure}
         \caption[ The average and standard deviation of critical parameters ]
        {\small Effect of network connectivity  (degree $d$) on the convergence for linear regression on  dataset CT with computation times from ASCI-Q super-computer. M = 100, B = 128. } 
        \label{f:regression_asciq}
    \end{figure*}

%    \begin{figure*}[h]
%%\vspace{.3in}
%        \centering
%        \begin{subfigure}[b]{0.3\textwidth}
%            \centering
%            \includegraphics[width=\textwidth]{figures/simulations_wide_layout/test6_svm_real_svm_100n_err_vs_iter.png}
%            \caption[]%
%            {{\small Error vs iterations}}    
            %\label{f:regr_err_vs_iter}
%        \end{subfigure}
%        \hfill
%        \begin{subfigure}[b]{0.3\textwidth}  
%            \centering 
%            \includegraphics[width=\textwidth]{figures/simulations_wide_layout/test6_svm_real_svm_100n_spark_iter_vs_time.png}
%            \caption[]%
%            {{\small Throughput}}    
            %\label{f:regr_iter_vs_time}
%        \end{subfigure}
        %\vskip\baselineskip
%        \hfill
%        \begin{subfigure}[b]{0.3\textwidth}   
%            \centering 
%            \includegraphics[width=\textwidth]{figures/simulations_wide_layout/test6_svm_real_svm_100n_spark_err_vs_time.png}
%            \caption[]%
%            {{\small Error vs time}}    
            %\label{f:regr_error_vs_time}
%        \end{subfigure}
%         \caption[ The average and standard deviation of critical parameters ]
%        {\small Effect of network connectivity  (degree $d$) on the convergence for SVM on  dataset SUSY with computation times from a Spark cluster.} 
%        \label{f:classification_spark}
%    \end{figure*}

    \begin{figure*}
%\vspace{.3in}
        \centering
        \begin{subfigure}[b]{0.3275\textwidth}
            \centering
            \includegraphics[width=\textwidth]{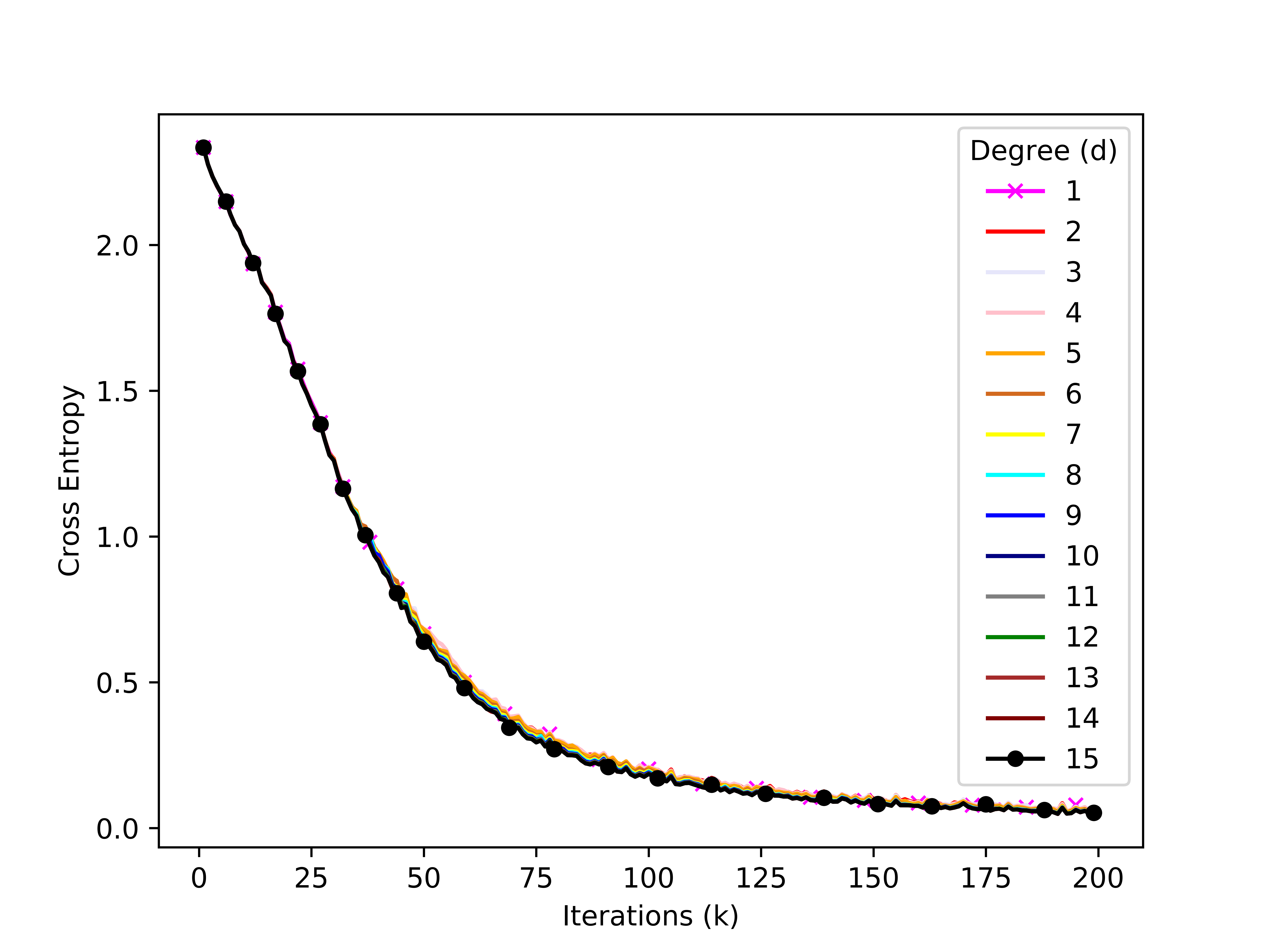}
            \caption[]%
            {{\small Error vs iterations}}    
            %\label{f:nn_err_vs_iter}
        \end{subfigure}
        \hfill
        \begin{subfigure}[b]{0.3275\textwidth}  
            \centering 
            \includegraphics[width=\textwidth]{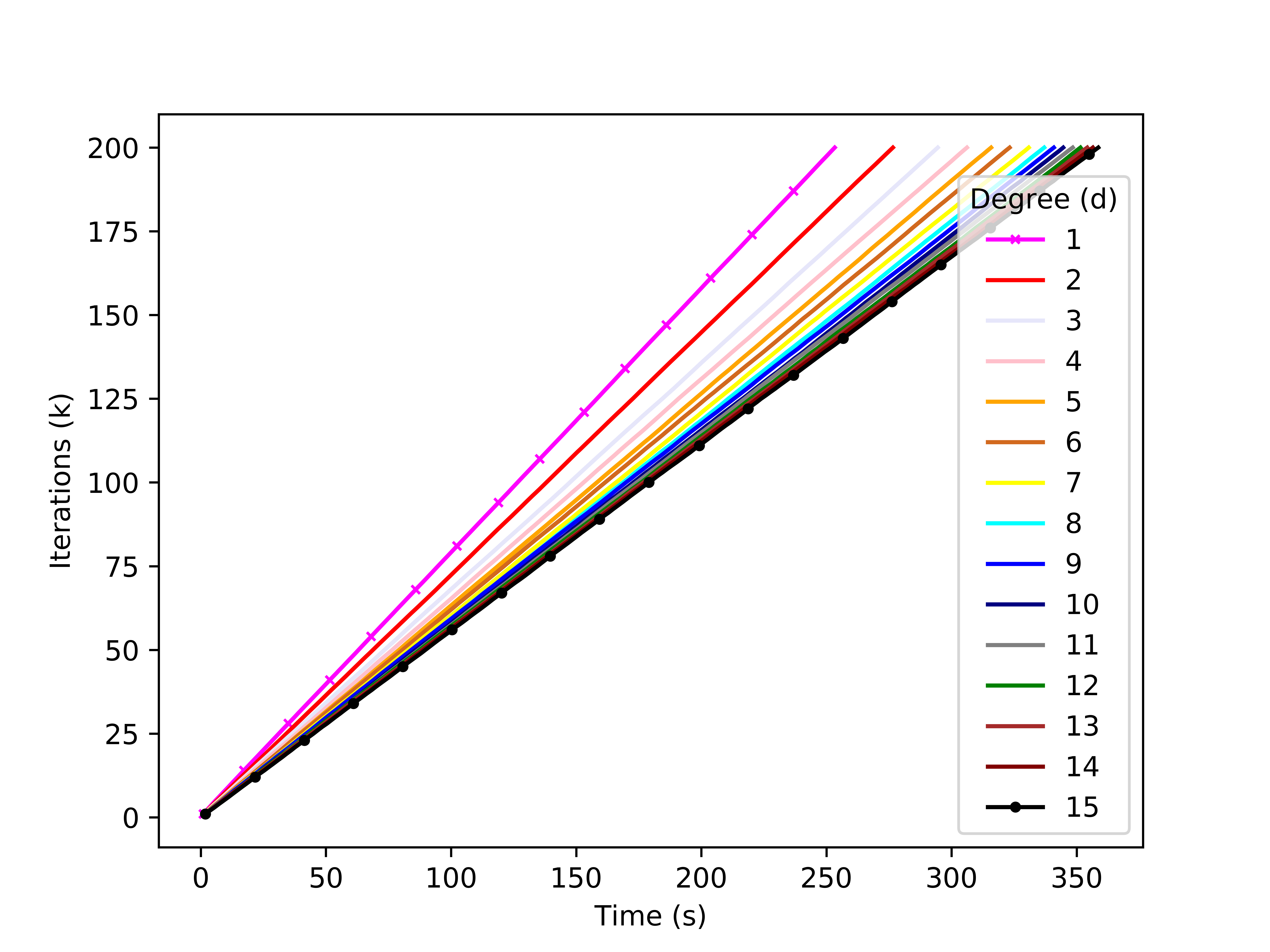}
            \caption[]%
            {{\small Throughput}}    
            %\label{f:nn_iter_vs_time}
        \end{subfigure}
        %\vskip\baselineskip
        \hfill
        \begin{subfigure}[b]{0.3275\textwidth}   
            \centering 
            \includegraphics[width=\textwidth]{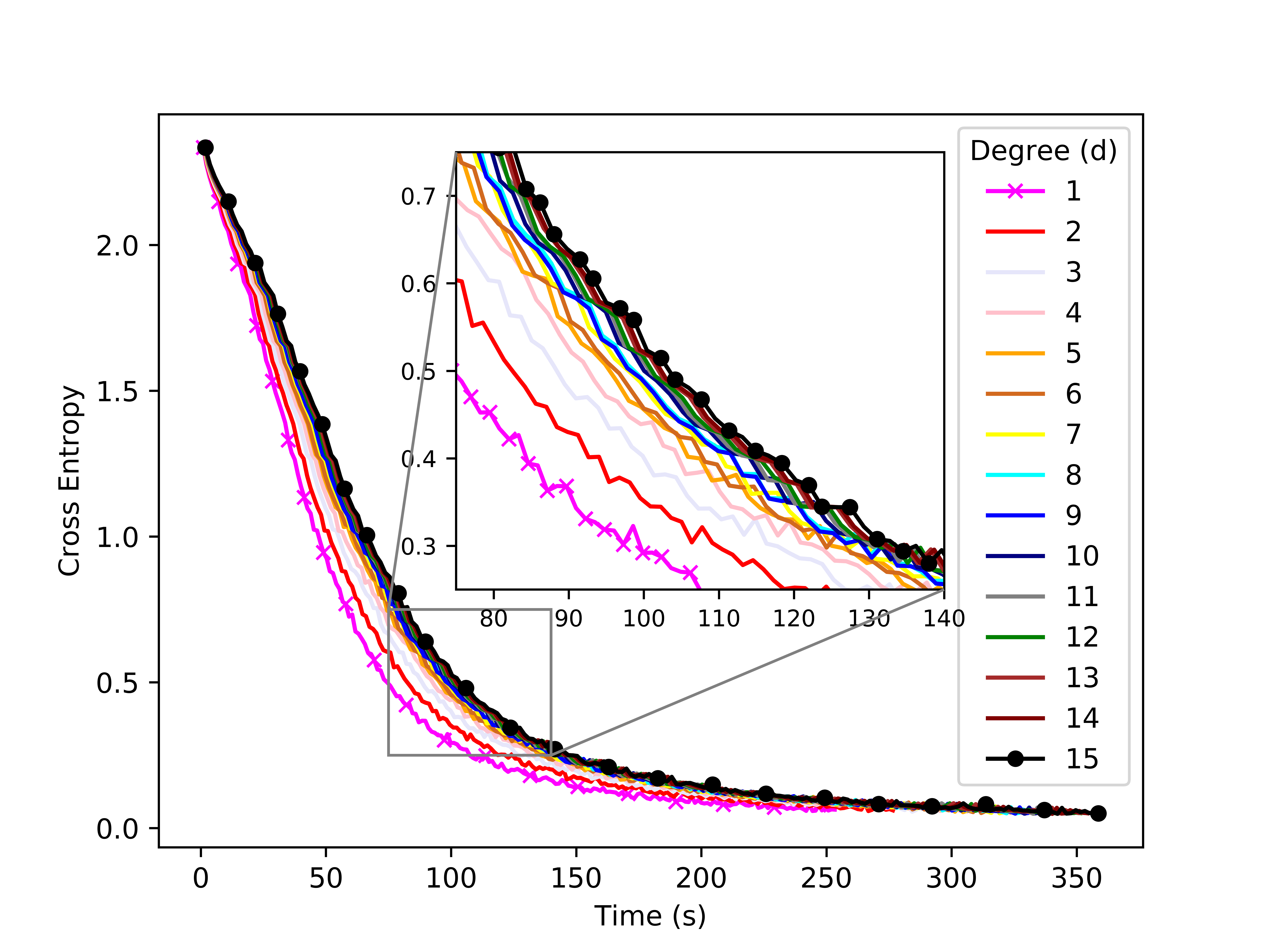}
            \caption[]%
            {{\small Error vs time}}    
            %\label{f:regr_error_vs_time}
        \end{subfigure}
         \caption[ The average and standard deviation of critical parameters ]
        {\small Effect of network connectivity  (degree $d$) on the convergence for ResNet18 on dataset MNIST with computation times from a Spark cluster. M = 16, B = 500.} 
        \label{f:nn_spark}
    \end{figure*}
   \begin{figure*}
%\vspace{.3in}
        \centering
        \begin{subfigure}[b]{0.3275\textwidth}
            \centering
            \includegraphics[width=\textwidth]{figures/simulations_wide_layout/MNIST_LOSSVSITER.png}

            \caption[]%
            {{\small Error vs iterations}}    
            %\label{f:nn_err_vs_iter}
        \end{subfigure}
        \hfill
        \begin{subfigure}[b]{0.3275\textwidth}  
            \centering 
            \includegraphics[width=\textwidth]{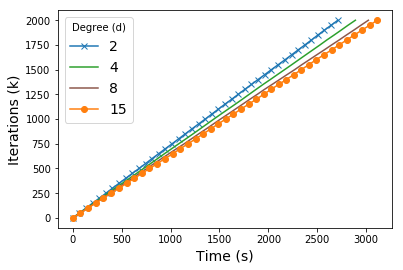}
            \caption[]%
            {{\small Throughput}}    
            %\label{f:nn_iter_vs_time}
        \end{subfigure}
        %\vskip\baselineskip
        \hfill
        \begin{subfigure}[b]{0.3275\textwidth}   
            \centering 
            \includegraphics[width=\textwidth]{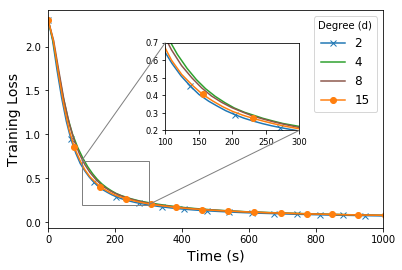}
            \caption[]%
            {{\small Error vs time}}    
            %\label{f:regr_error_vs_time}
        \end{subfigure}
         \caption[ The average and standard deviation of critical parameters ]
        {\small Effect of network connectivity  (in-degree $d$) on the convergence for 2-conv layers on dataset MNIST with computation times from ASCI-Q super-computer. M = 16, B = 500. } 
        \label{f:nn_asciq}
    \end{figure*}

        \begin{figure*}
%\vspace{.3in}
        \centering
        \begin{subfigure}[b]{0.3\textwidth}
            \centering
            \includegraphics[width=\textwidth]{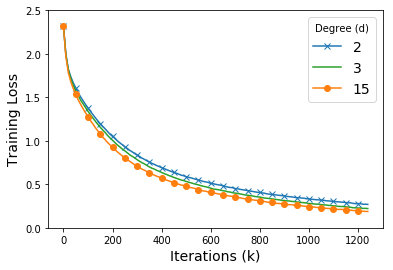}
            \caption[]%
            {{\small Error vs iterations}}    
            %\label{f:regr_err_vs_iter}
        \end{subfigure}
        \hfill
        \begin{subfigure}[b]{0.3\textwidth}  
            \centering 
            \includegraphics[width=\textwidth]{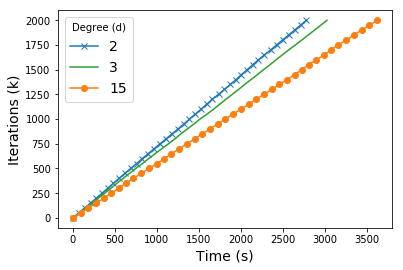}
            \caption[]%
            {{\small Throughput}}    
            %\label{f:regr_iter_vs_time}
        \end{subfigure}
        %\vskip\baselineskip
        \hfill
        \begin{subfigure}[b]{0.3\textwidth}   
            \centering 
            \includegraphics[width=\textwidth]{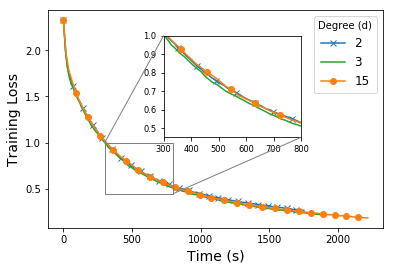}
            \caption[]%
            {{\small Error vs time}}    
            %\label{f:regr_error_vs_time}
        \end{subfigure}
         \caption[ The average and standard deviation of critical parameters ]
        {\small Effect of network connectivity  (degree $d$) on the convergence for CIFAR-10 with computation times from a spark cluster. M = 16, B = 128. } 
        \label{f:cifar_spark}
    \end{figure*}
    
        \begin{figure*}
%\vspace{.3in}
        \centering
        \begin{subfigure}[b]{0.3\textwidth}
            \centering
            \includegraphics[width=\textwidth]{figures/simulations_wide_layout/CIFAR10_LOSSVSITER.png}
            \caption[]%
            {{\small Error vs iterations}}    
            %\label{f:regr_err_vs_iter}
        \end{subfigure}
        \hfill
        \begin{subfigure}[b]{0.3\textwidth}  
            \centering 
            \includegraphics[width=\textwidth]{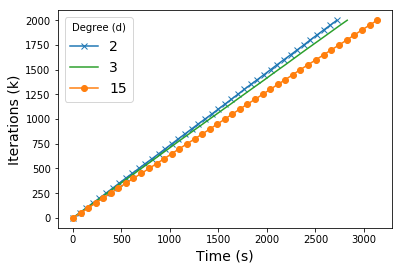}
            \caption[]%
            {{\small Throughput}}    
            %\label{f:regr_iter_vs_time}
        \end{subfigure}
        %\vskip\baselineskip
        \hfill
        \begin{subfigure}[b]{0.3\textwidth}   
            \centering 
            \includegraphics[width=\textwidth]{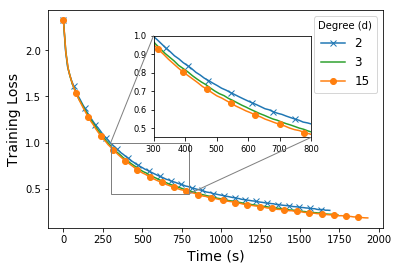}
            \caption[]%
            {{\small Error vs time}}    
            %\label{f:regr_error_vs_time}
        \end{subfigure}
         \caption[ The average and standard deviation of critical parameters ]
        {\small Effect of network connectivity  (degree $d$) on the convergence for CIFAR-10 with computation times from ASCI-Q super computer. M = 16, B = 128.} 
        \label{f:cifar_asciq}
    \end{figure*}
 %   \begin{figure*}
%\vspace{.3in}
 %       \centering
 %       \begin{subfigure}[b]{0.3\textwidth}
 %           \centering
 %           \includegraphics[width=\textwidth]{figures/simulations_wide_layout/test7_svm_real_svm_100n_err_vs_iter.png}
 %           \caption[]%
 %           {{\small Error vs iterations}}    
            %\label{f:regr_err_vs_iter}
 %       \end{subfigure}
 %       \hfill
 %       \begin{subfigure}[b]{0.3\textwidth}  
 %           \centering 
 %           \includegraphics[width=\textwidth]{figures/simulations_wide_layout/test7_svm_real_svm_100n_custom_iter_vs_time.png}
 %           \caption[]%
 %           {{\small Throughput}}    
            %\label{f:regr_iter_vs_time}
 %       \end{subfigure}
        %\vskip\baselineskip
 %       \hfill
 %       \begin{subfigure}[b]{0.3\textwidth}   
 %           \centering 
 %           \includegraphics[width=\textwidth]{figures/simulations_wide_layout/test7_svm_real_svm_100n_custom_err_vs_time.png}
 %           \caption[]%
 %           {{\small Error vs time}}    
            %\label{f:regr_error_vs_time}
 %       \end{subfigure}
 %        \caption[ The average and standard deviation of critical parameters ]
%        {\small Effect of network connectivity  (degree $d$) on the convergence for SVM on  dataset SUSY with computation times from ASCI Q super-computer.} 
%        \label{f:classification_asciq}
%    \end{figure*}

%\begin{figure}
%    \centering
%    \includegraphics[scale = 0.5]{figures/CT100.png}
%    \caption{Effect of network connectivity of CT data set}
%    \label{fig:my_label}
%\end{figure}

\end{document}